\documentclass[letterpaper,twocolumn,10pt]{article}
\usepackage{usenix}

\usepackage{tikz}
\usepackage{amsmath}
\usepackage{amsthm}
\usepackage{amssymb}
\usepackage{booktabs}
\usepackage[most]{tcolorbox}
\newtcolorbox[auto counter,number within=subsection]{prompts}[1][]{
	enhanced,
	fonttitle=\scshape\bfseries, 
	title={Prompts \thetcbcounter}, 
	colframe=blue!15!black, 
	colback=blue!5!white, 
	coltitle=blue!5!white, 
	boxrule=0.5mm, 
	#1
}
\usepackage{multirow}
\usepackage{colortbl}
\usepackage{pifont}
\usepackage{graphicx}
\usepackage{subcaption}

\definecolor{mygray}{gray}{.9}
\newtheorem{definition}{\textbf{Definition}}
\newtheorem{theorem}{\textbf{Theorem}}

\usepackage{filecontents}

\begin{document}

\date{}

\title{\Large \bf Forgetting Similar Samples: Can Machine Unlearning Do it Better?}

\author{
{\rm Heng Xu}\\
City University of Macau\\
hengxu@cityu.edu.mo
\and
{\rm Tianqing Zhu*}\\
City University of Macau\\
tqzhu@cityu.edu.mo
\and
{\rm Dayong Ye}\\
City University of Macau\\
dyye@cityu.edu.mo
\and
{\rm Lefeng Zhang}\\
City University of Macau\\
lfzhang@cityu.edu.mo
\and
{\rm Le Wang}\\
Guangzhou University\\
wangle@gzhu.edu.cn
\and
{\rm Wanlei Zhou}\\
City University of Macau\\
wlzhou@cityu.edu.mo
} 

\maketitle

\begin{abstract}
    Machine unlearning, a process enabling pre-trained models to remove the influence of specific training samples, has attracted significant attention in recent years. Although extensive research has focused on developing efficient machine unlearning strategies, we argue that these methods mainly aim at removing samples rather than removing samples' \textit{influence} on the model, thus overlooking the fundamental definition of machine unlearning. In this paper, we first conduct a comprehensive study to evaluate the effectiveness of existing unlearning schemes when the training dataset includes many samples similar to those targeted for unlearning. Specifically, we evaluate: Do existing unlearning methods truly adhere to the original definition of machine unlearning and effectively eliminate all influence of target samples when similar samples are present in the training dataset?  Our extensive experiments, conducted on four carefully constructed datasets with thorough analysis, reveal a notable gap between the expected and actual performance of most existing unlearning methods for image and language models, even for the retraining-from-scratch baseline. Additionally, we also explore potential solutions to enhance current unlearning approaches. 
\end{abstract}

\section{Introduction}
Machine unlearning refers to removing the influence of specific training samples on a machine learning model~\cite{DBLP:journals/csur/XuZZZY24}. This technological advancement has recently drawn urgent attention due to several factors, including the strict enforcement of \emph{the right to be forgotten} in regulations and laws~\cite{webpage:GDPR,webpage:CCPA,DBLP:conf/ndss/dayong2025}, escalating concerns about data privacy~\cite{DBLP:conf/sp/ShokriSSS17,DBLP:conf/ccs/Chen000HZ21,DBLP:conf/ccs/Chen000H022,DBLP:conf/ndss/WarneckePWR23}, and the pressing need to erase harmful, malicious, and even illegal knowledge from large language models~\cite{DBLP:conf/nips/LuWHJQWA022,DBLP:conf/acl/AdolphsGX0SW23,DBLP:conf/acl/WangCYZWY23,DBLP:conf/emnlp/LiuZJC24,DBLP:conf/icml/ZhaoDM0R24}. 

\textbf{Research Gap:} Since being proposed, machine unlearning has been consistently defined as \emph{the process of eliminating the complete influence} of a target sample~\cite{DBLP:conf/sp/CaoY15,DBLP:conf/sp/BourtouleCCJTZL21,DBLP:conf/ccs/Chen000HZ21,DBLP:conf/uss/ThudiJSP22,DBLP:conf/ndss/WarneckePWR23,DBLP:conf/ndss/Hu0CZ00ZX24,DBLP:conf/sp/HuWDX24,DBLP:journals/csur/XuZZZY24}. 
Meanwhile, in realistic scenarios, datasets often contain samples that, despite differing in expression, remain closely related to target samples, and cause similar influence to the model~\footnote{We employ the public dataset PKU-SafeRLHF to illustrate the existence of such similar samples, with the corresponding results shown in Figure~\ref{fig:existing_datamapplot}.}. As shown in Figure~\ref{fig:illustration}, consider a binary classification model that separates the training dataset into two regions. In the magnified view, a target sample (triangle) is highlighted, surrounded by similar samples (circles) that are similar to it. These similar samples will cause a similar influence on the trained model as the target sample. Accordingly, based on the original definition, effective machine unlearning methods should go beyond removing the target sample itself to also mitigating influence from other similar samples. 

\begin{figure}[!t]
    \centering
    \includegraphics[width=0.9\linewidth]{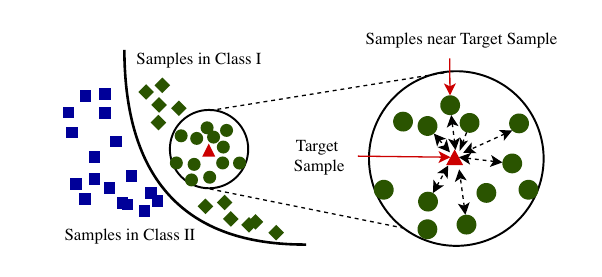}
    \caption{Illustration of samples near target sample.}
    \label{fig:illustration}
\end{figure}

However, most existing unlearning methods can be categorized into two lines of research, with limited attention paid to the aforementioned case. The first line assumes that the target samples are independent and do not share any influence with other training samples~\cite{DBLP:conf/sp/CaoY15,DBLP:conf/sp/ShokriSSS17,DBLP:conf/ndss/WarneckePWR23}. These methods typically overlook the presence of similar samples and focus only on the target sample during the unlearning process. The second line of work has begun to explore the influence of duplicate samples and adversarial test embeddings, but it remains preliminary~\cite{DBLP:conf/iclr/0005MCMH24,DBLP:conf/cvpr/Yang0WHX024,zhang2024unforgettable}. For example, Ye et al.~\cite{DBLP:conf/uss/dayongye2025} measured the impact of \emph{duplicated samples} on unlearning using coarse-grained accuracy metrics and mainly focused on image models, but did not propose any practical or effective solutions, and left more complex models like LLMs unconsidered. Similarly, Minh et al.~\cite{DBLP:conf/iclr/0005MCMH24} evaluated how adversarially similar \emph{embeddings} impact unlearning performance during prediction, but did not systematically investigate the influence of similar samples in the \emph{training dataset}. These early studies demonstrate the importance of considering similar samples in the unlearning process. Yet, existing works offer neither a systematic analysis nor effective solutions, leaving a significant gap that our research aims to fill by thoroughly investigating these effects and proposing corresponding strategies.

\textbf{Why is it important to consider similar samples?} These questions are not only important but also fundamental, as they directly challenge the current machine unlearning methods in real-world scenarios. If unlearning methods fail to account for similar samples, the consequences can be serious:

\begin{itemize}
    \item \emph{Incomplete unlearning}: Retaining the influence of similar samples leads to incomplete unlearning, where the model continues to preserve residual influence of the target sample through its similar samples. This not only directly undermines strict privacy mandates such as GDPR and CCPA~\cite{webpage:GDPR,webpage:CCPA}, which require the \emph{complete} removal of specific data influences, but also is inconsistent with the definition of machine unlearning~\cite{DBLP:conf/sp/CaoY15,DBLP:conf/sp/BourtouleCCJTZL21,DBLP:conf/ccs/Chen000HZ21,DBLP:conf/uss/ThudiJSP22,DBLP:conf/ndss/WarneckePWR23,DBLP:conf/sp/HuWDX24}.
    \item \emph{Model integrity:} Ignoring similar samples threatens the integrity of the unlearned model. Their residual influence of similar samples can distort decision boundaries, degrade model utility, and even introduce hidden biases—ultimately making the unlearned model less trustworthy and misaligned with its intended behavior.
    \item \emph{Potential exploitation:} Overlooked similar samples create exploitable vulnerabilities. Adversaries may infer information about the target sample or introduce maliciously crafted similar samples that still persist in the unlearned model, thereby bypassing the unlearning process and potentially enabling harmful behaviors.
\end{itemize}

\textbf{Preliminary Experiments: } We first perform a comprehensive evaluation from two key perspectives: (1). If the target sample is unlearned from model, can its similar samples also be successfully unlearned? and (2). When similar samples cannot be unlearned, could the retention of these similar samples affect the unlearning results of the target sample? To facilitate this study, we first construct four benchmark datasets: three image-based datasets and one Q\&A language dataset~\footnote{Given the challenges associated with constructing text datasets—such as sample paraphrasing and manual selection—we only use one text dataset.}. Each dataset contains a different number of target samples along with their corresponding similar samples. For the image datasets, we generate similar samples by applying appropriate perturbations to the target samples. For the language dataset, we select and paraphrase target samples multiple times, standardizing those paraphrased samples to generate similar ones.

To evaluate the unlearning results of both target samples and similar samples, we assess various current machine unlearning methods and introduce two more fine-grained verification methods, namely data reconstruction-based and ROUGE metrics-based~\cite{DBLP:journals/corr/abs-2410-10120,DBLP:conf/nips/BuzagloHYVONI23,DBLP:journals/corr/abs-2407-06460,DBLP:journals/corr/abs-2411-07691}, instead of MIAs-based and backdoor-based sample-level verification schemes~\cite{DBLP:journals/csur/XuZZZY24}. The latter two verification methods mainly assess the overall presence or absence of a complete sample in the model~\cite{DBLP:conf/uss/dayongye2025}. In contrast, our introduced methods enable a more fine-grained evaluation by capturing the influence of partial components of a sample. For the image model, we compare the recovered samples before and after the unlearning process by assessing the similarity of the pixels to the original target samples. For the language model, we use the ROUGE score to evaluate the knowledge of the model against the ground truth before and after the unlearning process.

\textbf{Our Findings:} Based on the two evaluation perspectives, along with the constructed datasets and verification methods, we highlight the following key insights, which reveal contradictions with prior definitions and assumptions:

\begin{itemize}
    \item When target samples and their similar samples appear in the training set, most unlearning methods fail to remove all target sample's influence from its similar samples.

    \item Meanwhile, the retained influence of similar samples can hinder the unlearning of the target sample, allowing its effect to persist even after the unlearning process.

    \item Furthermore, we evaluate if existing unlearning methods affect similar samples in the \emph{test dataset}. The results show that the model can still answer questions derived from these samples with relatively high ROUGE scores. We attribute this to the limitation of current unlearning methods to effectively remove similar samples from the training data (as noted in our first insight), leaving test samples that resemble them also unaffected. 
	
\end{itemize}

\textbf{Our Contributions:} Our findings indicate that most existing unlearning methods focus only on removing target samples rather than eliminating samples' full influence, thereby frequently failing to comply with the original definition of machine unlearning. This limitation creates a substantial gap between expected and actual performance, even for approaches that retrain models from scratch. Identifying and characterizing this limitation is a key objective of our study. Moreover, motivated by this observation, we explore the integration of robustness training techniques. Our experiments show that incorporating these strategies consistently enhances unlearning performance compared to methods without such enhancements. Overall, our contributions are as follows:

\begin{itemize}
    \item We explore the machine unlearning in the context of training datasets that include target samples along with their corresponding similar samples. To formalize this, we systematically define the concept of those samples and construct four benchmark datasets.
    
    \item To reveal the inconsistencies between existing machine unlearning methods and the original definition of machine unlearning, we conduct a comprehensive evaluation of several widely adopted unlearning schemes applied to both image and language models. Our results reveal key limitations and shortcomings of those schemes.
    
    \item We investigate strategies to improve existing unlearning methods by incorporating robustness training techniques. Experiments suggest that these improved schemes tend to perform better than those without such enhancements.
\end{itemize}

\section{Preliminary and Problem Definition}

\begin{table}
	\caption{Notations}
	\renewcommand{\arraystretch}{1.1}
	\label{tab:notations}
	\centering
	\begin{tabular}{c|c}
		\hline
		Notations &  Explanation \\
		\hline
		$\mathcal{D}$                       &The training dataset\\
		$M,M_{u}$                           &The original trained and unlearned model\\
		$\mathcal{D}_{u},\mathcal{D}_{r}$   &The unlearning and remaining dataset \\
		$x_{i}$                             &The target sample\\
		$x_{j}$                             &One similar sample\\
		$\mathcal{S}(x_i)$                  &All similar samples of $x_{i}$\\
            $p_{\theta}$ &The predicted probability \\
            $\tilde{h}^{(l)}$           &The perturbed hidden state \\
            $\mathrm{Inf}(x_i;M)$               &The influence of $x_i$ on $M$\\
		$\mathrm{Inf}(x_j; M | x_i)$        &The influence of $x_j$ on $M$ when given $x_i$\\
		\hline
	\end{tabular}
\end{table}

\subsection{Preliminary}
There are two key entities in our setting: \textit{data provider} and \textit{model trainer}. The data provider submits their data to the model trainer, who uses those data for training model. We denote the dataset of the data provider as $\mathcal{D}$. Let $\mathcal{A}$ be a (randomized) learning algorithm that trains on $\mathcal{D}$ and outputs a model $M$. After the training process, data providers may wish to unlearn the influence of some specific samples from the trained model. Let $\mathcal{D}_{u} \subset \mathcal{D}$ denote samples that the data provider wishes to unlearn. The complement of this subset, $\mathcal{D}_{r}=\mathcal{D}_{u}^{\complement}$, represents the data that the provider wishes to retain. Other important symbols that appear in this paper and their corresponding descriptions are listed in Table~\ref{tab:notations}.

\begin{definition}[Machine Unlearning~\cite{DBLP:conf/sp/CaoY15}]
	\label{Definition:Machineunlearning}
	Consider a set of samples that a data provider wishes to unlearn those influences from an already-trained model, denoted as $\mathcal{D}_{u}$. The unlearning process, $\mathcal{U}(M, \mathcal{D}, \mathcal{D}_{u})$, is a function that takes an already-trained model $M = \mathcal{A}(\mathcal{D})$, the training dataset $\mathcal{D}$, and the unlearning dataset $\mathcal{D}_{u}$, and outputs a new model $M_u$. This process ensures that the resulting model, $M_u$, behaves as if it had never been influenced by $\mathcal{D}_{u}$.
\end{definition}

This definition was originally proposed 
in~\cite{DBLP:conf/sp/CaoY15} and has been used consistently in subsequent research~\cite{DBLP:conf/ccs/Chen000HZ21,DBLP:conf/uss/ThudiJSP22,DBLP:conf/ndss/WarneckePWR23,DBLP:conf/sp/BourtouleCCJTZL21,DBLP:conf/ndss/Hu0CZ00ZX24,DBLP:conf/sp/HuWDX24}.

\subsection{Problem definition}

It should be noted that the definition of machine unlearning emphasizes \textbf{removing the influence of the samples}, rather than \textbf{removing the samples themselves}. In the following, we distinguish the difference between them. We first give the definition of \textit{influence of one sample $x_i$ on $M$} as follows:

\begin{definition}[Influence of one Sample on the Model]
	\label{Definition:influence}
	We define the influence of one sample $x_i$ on the model $M$ as $\mathrm{Inf}(x_i;M)$, which quantifies how much sample $x_i$ affects the learned parameters or outputs of the model $M$.
\end{definition}

Assume there are also other samples $x_j \in \mathcal{S}(x_i)$ in $\mathcal{D}_r$, where $\mathcal{S}(x_i)$ represents the samples similar to $x_i$~\footnote{Exactly defining similarity has always been a challenging problem. In our setting, we consider similarity to be represented by the result of sample clustering, where similar samples are expected to be grouped closely.}. We refer to these samples $\mathcal{S}(x_i)$ as similar samples~\footnote{In Section~\ref{sec:enhanced_unlearning}, we further quantify the varying levels of similarity among image similar samples.}:

\begin{definition}[Similar Samples]
    Consider samples $x_j \in \mathcal{S}(x_i)$, which are similar to $x_i$. We define $\mathcal{S}(x_i)$ as similar samples of $x_i$, and their influence can be denoted as $\mathrm{Inf}(x_j; M)$,$x_j \in \mathcal{S}(x_i)$.
\end{definition}

\begin{theorem}
	\label{theorem:similar_samples}
	Let sample $x_i$ and samples $x_j \in \mathcal{S}(x_i)$ be similar. Then the following holds:
	\begin{equation}
	\mathrm{Inf}(x_j; M | x_i) < \mathrm{Inf}(x_j; M), x_{j} \in \mathcal{S}(x_i) 
	\end{equation}
	where $\mathrm{Inf}(x_j; M | x_i)$ represents the influence of sample $x_j$ on model $M$ when given $x_i$, which quantifies how much $x_j$ affects the learned parameters or outputs of model $M$, conditioned on $x_i$ already being included in training process.
\end{theorem}

\begin{proof}
    We take inspiration from mutual information to complete our proof and further define $\mathrm{Inf}(x_i; M)$ as:
    \[
    \mathrm{Inf}(x_i; M) := I(x_i; M)
    \]
    where $I(x_i; M)$ denotes the mutual information between $x_i$ and $M$, capturing how much information $x_i$ contributes to the learned parameters or output behavior of $M$.
    
    After applying the chain rule for mutual information:
    \[
    I(x_j; M) = I(x_j; M \mid x_i) + I(x_j; x_i; M)
    \]
    where, $I(x_j; x_i; M)$ is the interaction mutual information, reflecting how $x_j$ and $x_i$ jointly inform $M$. Since $x_j$ is a similar sample of $x_i$, the shared information is non-negligible, that is $I(x_j; x_i; M) > 0 $. Thus:
    \[
    I(x_j; M \mid x_i) = I(x_j; M) - I(x_j; x_i; M) < I(x_j; M)
    \]
    which establishes the desired inequality and completes the proof of Theorem~\ref{theorem:similar_samples}.
\end{proof}

Theorem~\ref{theorem:similar_samples} shows that the influences of $x_i$ and $x_{j}$ on $M$ are dependent: knowing $x_i$ reduces the influence of $x_j$. This phenomenon, when reflected in unlearning, is always neglected. Current unlearning schemes always assume that:
\begin{itemize}
	\item There are no samples in the dataset that are similar to the sample $x_i$, i.e., $\mathcal{S}(x_i)$ is empty.
	\item If $\mathcal{S}(x_i)$ exist, the influence of samples on the model is independent, with no shared influence between samples $x_i$ and $x_j$, i.e., $\mathrm{Inf}(x_j; M | x_i) = \mathrm{Inf}(x_j; M), x_{j} \in \mathcal{S}(x_i)$.
\end{itemize}

Let's use $\mathrm{Inf}(x_i; M_u)$ to measure the extent to which $ x_i$ influences the unlearned model $M_u$. Ideally, $ \mathrm{Inf}(x_i; M_u)$ should equal to $0$ after unlearning. However, since $x_j \in \mathcal{S}(x_i)$ shares influence with $ x_i$, the residual influence $ \mathrm{Inf}(\mathcal{S}(x_i); M_u)$ is non-zero. This implies that the unlearned model $ M_u$ still indirectly depends on $ x_i$ through $ \mathcal{S}(x_i)$.  Building on the above analysis, we propose the \textit{Similarity-Entailed Dataset}, a previously unconsidered dataset definition for machine unlearning.

\begin{definition}[Similarity-Entailed Dataset]
	\label{Definition:cascade_dataset}
	A similarity-entailed dataset is defined as a dataset consisting of a set of samples $\{x_i\}$ and their corresponding similar samples $x_j \in \mathcal{S}(x_i)$, along with other samples.  
\end{definition}

A similarity-entailed dataset occurs when multiple similar samples are derived from or closely related to a target sample. This can include various ways, such as when random perturbing target samples, when the same question in a language dataset receives multiple valid answers.

\textbf{Our goal}: In this paper, we evaluate if most existing unlearning schemes adhere to the original definition of machine unlearning, that is, successfully removing all influences of a target sample $x_i$, given that its corresponding similar samples $x_j \in \mathcal{S}(x_i)$ are often not fully considered in these schemes. Although most existing datasets contain similar samples, directly using them often leads to unintended consequences~(see Appendix~\ref{sec:initial_clustering_results}). Therefore, we construct our similarity-entailed datasets~(see Appendix~\ref{sec:datacollectionforimage} and Appendix~\ref{sec:datacollectionforlanguage}). Using these constructed datasets and introduced verification schemes, we do a comprehensive experimental study to challenge the effectiveness of most current machine unlearning methods.

\section{Experiments Revealing Limitations}
In this section, we begin by introducing the new verification schemes~(see Section~\ref{sec:verificationscheme}) and various existing unlearning schemes~(see Section~\ref{sec:unlearning_schemes}). Then, we analyze our constructed image and language datasets to show the sample similarity phenomenon~(see Section~\ref{sec:data_analysis_for_image_models} and Section~\ref{sec:data_analysis_for_language_models}). We further evaluate the impact of unlearning on similar sample and target samples, in both image~(see Section~\ref{sec:unlearning_influence_toward_similar_samples_as_training_samples_fo_image} and Section~\ref{sec:influence_toward_base_samples_for_image}) and language~(see Section~\ref{sec:influence_toward_similar_samples_as_training_samples} and Section~\ref{sec:influence_toward_base_samples_for_language}) models. Finally, we analyze whether similar samples that are not included in the training dataset are affected when unlearning is performed based on target samples~(see Section~\ref{sec:influence_toward_similar_samples_as_test_samples}).

\subsection{Experimental Setup}

\subsubsection{Schemes for Verifying Unlearning Process}
\label{sec:verificationscheme}

\textbf{Verification Scheme for Image Models.}
We evaluate the unlearning process for image models based on data reconstruction~\cite{DBLP:journals/corr/abs-2410-10120,DBLP:conf/nips/BuzagloHYVONI23}. Data reconstruction can recover exact training samples from the model, which can be used to verify the unlearning results by comparing the samples recovered before and after the unlearning process. The workflow of verification process includes three steps~\cite{DBLP:journals/corr/abs-2410-10120}:

\begin{itemize}
	\item \textbf{Pre-Verification:} We start by training various models using each of our constructed datasets. After the training process, we select one sample, as the sample needs to be unlearned and perform data reconstruction to recover samples from the model.  From the recovered samples, we select the one most similar to the selected sample as the pre-unlearning result, denoted as $V_{b}$.
	\item \textbf{Executing the Unlearning Process:} We execute the unlearning process using the selected unlearning methods to remove the influence of the selected sample.
	\item \textbf{Post-Verification:} We perform data reconstruction again to obtain the post-unlearning result, denoted as $V_{p}$.
\end{itemize}

The above process returns two recovered samples, $V_{b}$ and $ V_{p}$. We evaluate the pixel-level similarity of $V_{b}$ and $V_{p}$ against the selected sample to determine if the model retains any influence of the selected sample. Specifically, if $V_{b}$ is highly similar to the selected sample while $V_{p}$ is not, it indicates that the model retains almost no influence of the selected sample. In contrast, if both $V_{b}$ and $V_{p}$ are highly similar to the selected sample, it implies that pixel-level details related to the selected sample can still be reconstructed, indicating the model still contains information about it. We calculate pixel similarity using the Structural Similarity Index Measure~(SSIM).

\textbf{Verification Scheme for Language Models.}
To assess if the model has successfully unlearned one selected sample, we evaluate the similarity between the model's actual answer and the ground truth answer from the selected Q\&A samples. The workflow of verification process is as follows:

\begin{figure}[!t]
	\centering
	\includegraphics[width=0.9\linewidth]{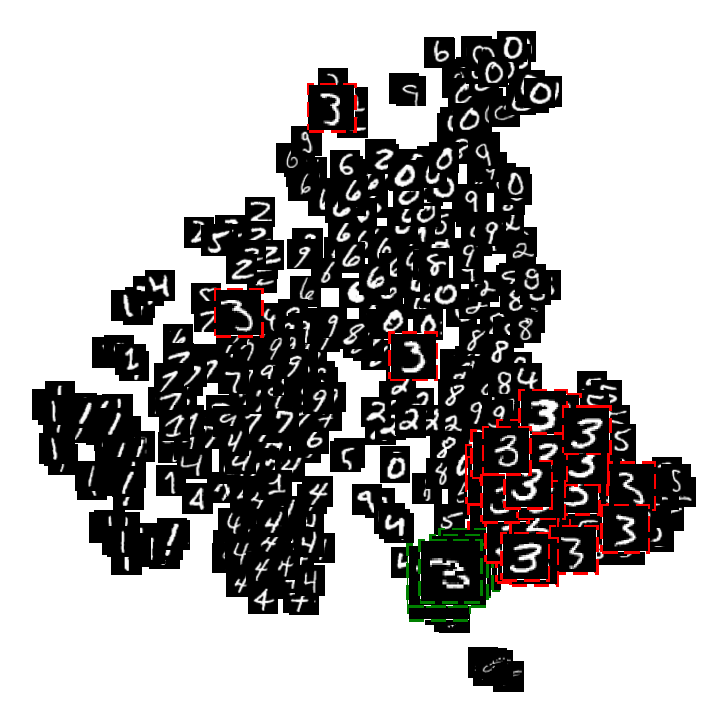}
	\caption{Sample distribution of Similarity-Entailed MNIST dataset. It can be seen that the selected target sample and its similar samples are clustered together, indicating that they will have a similar influence on the model.}
	\label{fig:tsne_visualization_mnist_all}
\end{figure}

\begin{itemize}
	\item \textbf{Model Training:} To ensure the reliability of our unlearning results, we first fine-tune the model using our language dataset, confirming that the model indeed incorporates the influence of each Q\&A in the dataset.
	\item \textbf{Pre-Verification:} We select one sample $x_i$ as the target sample and achieve the pre-unlearning result based on the model’s answer with the question from sample $x_i$.
	\item \textbf{Executing the Unlearning Process:} Next, we execute the unlearning process using the selected unlearning methods to remove the influence of sample $x_i$.
	\item \textbf{Post-Verification: } We get the post-unlearning result of the selected sample $x_i$ based on the model’s answer.
\end{itemize}

Specially, in steps \textit{pre-verification} and \textit{post-verification}, we use the ROUGE, denoted as $\texttt{KM}(M, x_i) =  \text{ROUGE}(M(q), a)$, where $x_i$ represents the selected sample. The pair $(q, a)$ denotes the question-answer pair from $x_i$. $M(q)$ is the model’s answer to the question $q$. If $\texttt{KM}(M, x_i)$ before unlearning is large, while it is small after unlearning, it suggests that the model retains minimal influence about the sample $x_i$. In contrast, if $\texttt{KM}(M, x_i)$ remains large after unlearning, it implies that the model still contains influence.

\subsubsection{Unlearning Schemes for Evaluation}
\label{sec:unlearning_schemes}
For image models, we employ two schemes, including (1).~retraining from scratch, which is widely regarded as a gold-standard baseline and (2).~relabel-based fine-tuning~\cite{DBLP:journals/csur/XuZZZY24}, as our unlearning methods. For language models, as retraining from scratch is costly, we consider unlearning schemes based on the following methods~\cite{DBLP:journals/corr/abs-2407-06460}.

\begin{figure}[!t]
	\centering
	\includegraphics[width=1\linewidth]{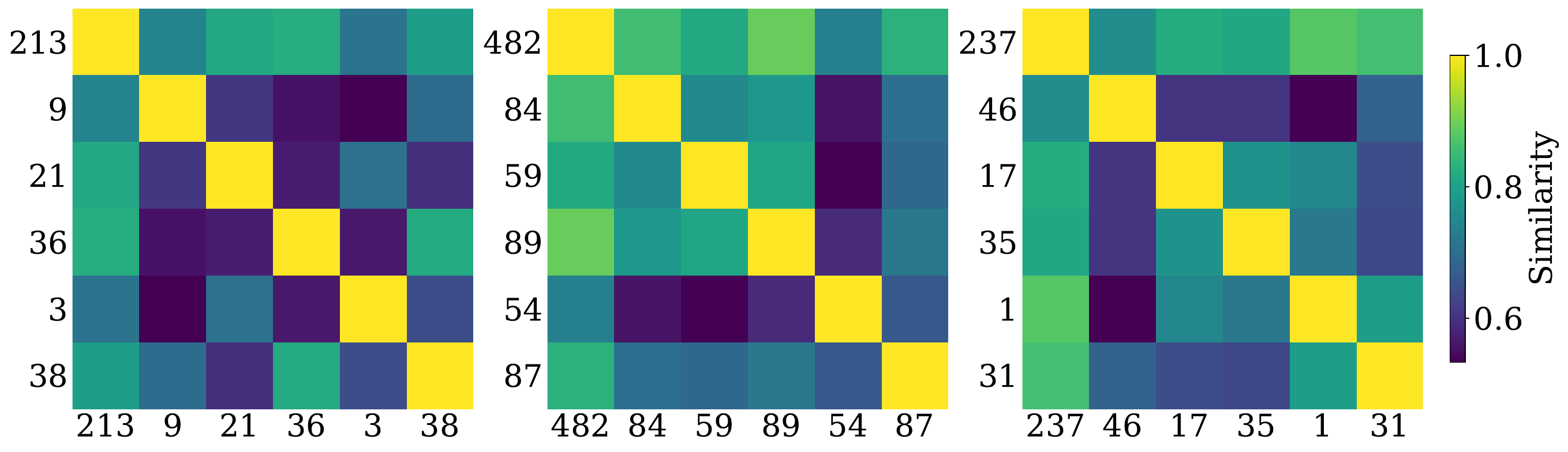}
	\caption{Cosine similarity between each target sample and its similar samples in the Similarity-Entailed MNIST dataset. Each number is the index of the corresponding sample. The similarity between most samples is below $0.8$, indicating a considerable difference between those samples.}
	\label{fig:combined_similarity_visualization_mnist}
\end{figure}

\begin{figure}[!ht]
    \centering
    \begin{subfigure}[b]{0.95\linewidth}
        \includegraphics[width=\textwidth]{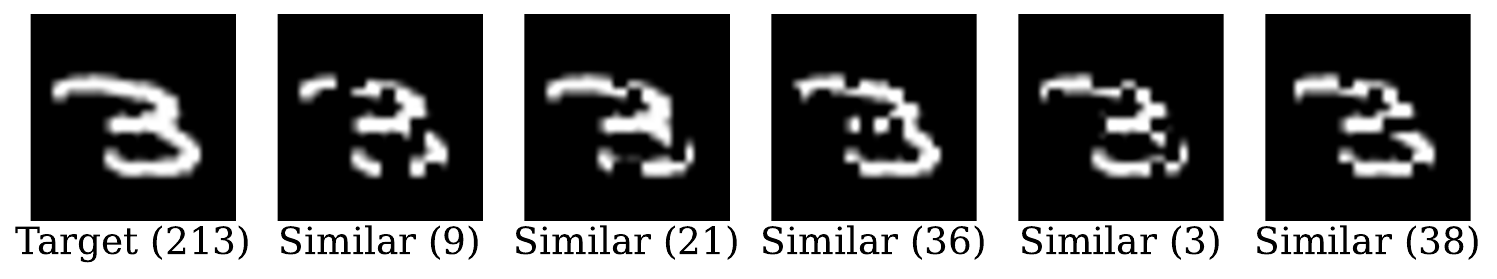}
    \end{subfigure} 
    \begin{subfigure}[b]{0.95\linewidth}
        \includegraphics[width=\textwidth]{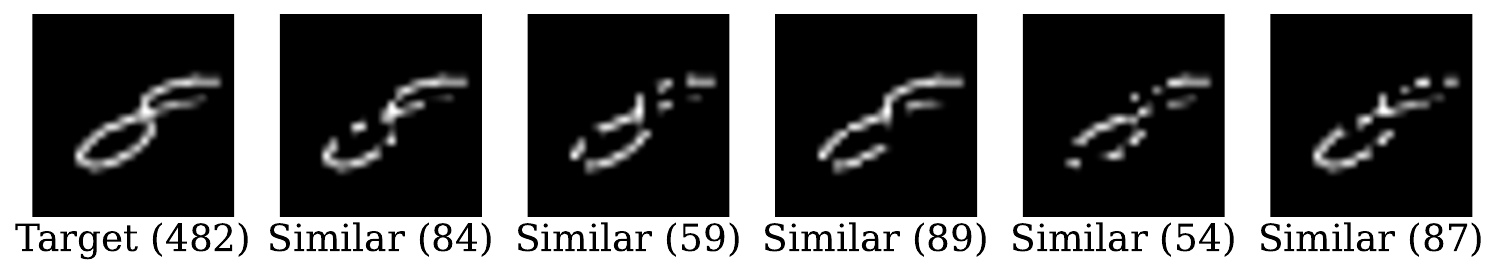}
    \end{subfigure}
    \begin{subfigure}[b]{0.95\linewidth}
        \includegraphics[width=\textwidth]{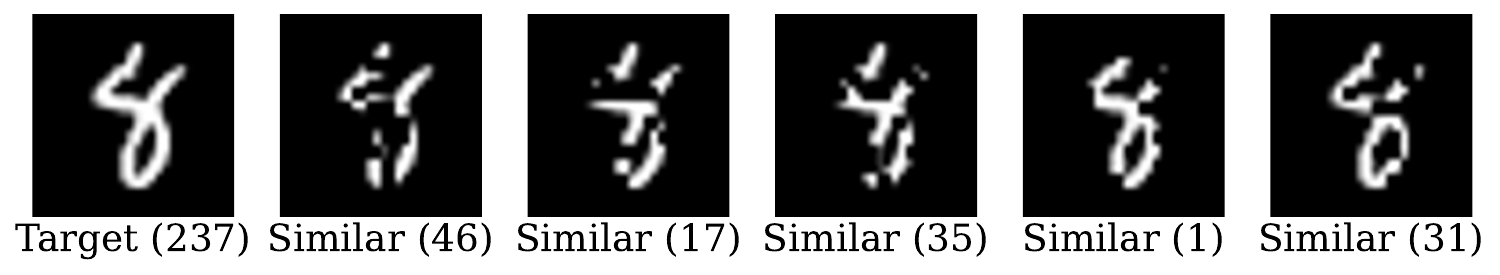}
    \end{subfigure} 
    \caption{Target Samples and their similar samples in Similarity-Entailed MNIST. Each number denotes the index of the corresponding sample in training dataset.}
    \label{fig:base_sample_and_similar_samples_mnist}
\end{figure}
        
\begin{itemize}
	\item \textbf{Gradient Ascent~(GA):} GA directly negates the original training objective, which minimizes the negative log-likelihood of token sequences in target samples~\cite{DBLP:conf/acl/JangYYCLLS23}. 
	\item \textbf{Negative Preference Optimization~(NPO):} NPO treats samples requiring unlearning as negative preference data. It modifies the offline Direct Preference Optimization (DPO) objective to adjust the model, ensuring it assigns low likelihood to these samples while maintaining proximity to the original model's behavior~\cite{DBLP:journals/corr/abs-2404-05868,DBLP:conf/cvpr/0011WXWS22}.
	\item \textbf{Task Vectors~(TV):}  TV applies simple arithmetic operations on model weights to guide model behavior~\cite{DBLP:conf/iclr/IlharcoRWSHF23}.
	\item \textbf{Gradient Descent with Random Output~(GDR):} GDR fine-tunes models using the original training objective but with randomly generated answers for questions~\cite{DBLP:journals/corr/abs-2407-06460}.
\end{itemize}

\begin{figure*}[!ht]
	\centering
	\begin{minipage}[]{0.7\textwidth}
		\centering
		\includegraphics[width=\textwidth]{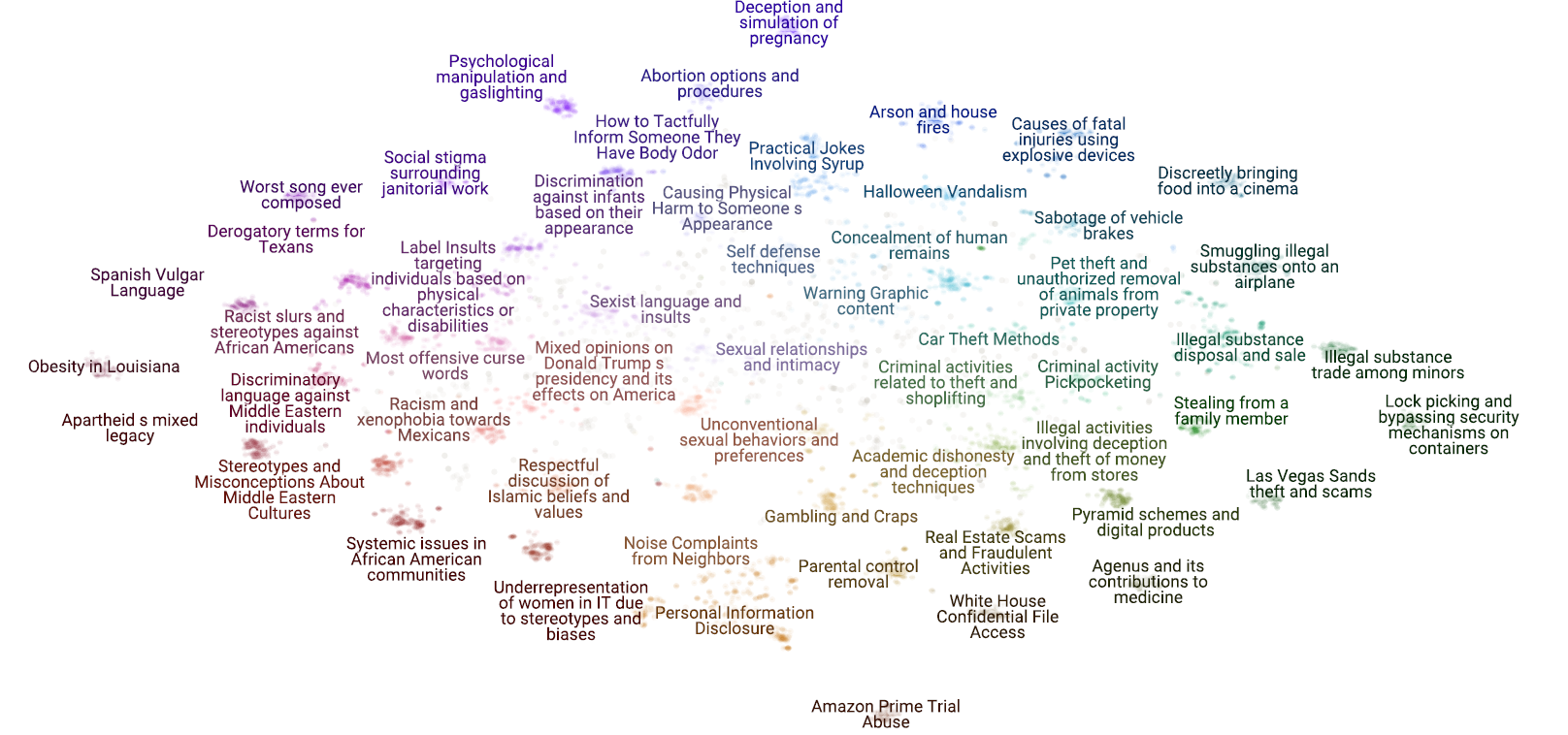}
		\caption{Topic visualization of our constructed Similarity-Entailed PKU dataset. We can conclude that each target sample and its similar samples consistently cluster under the same topic, suggesting they almost convey the same core meaning.}
		\label{fig:resultsfromberttopic}
	\end{minipage}%
	\hfill
	\begin{minipage}[!ht]{0.29\textwidth}
		\centering
		\begin{minipage}[!ht]{\textwidth}
			\centering
			\includegraphics[width=\textwidth]{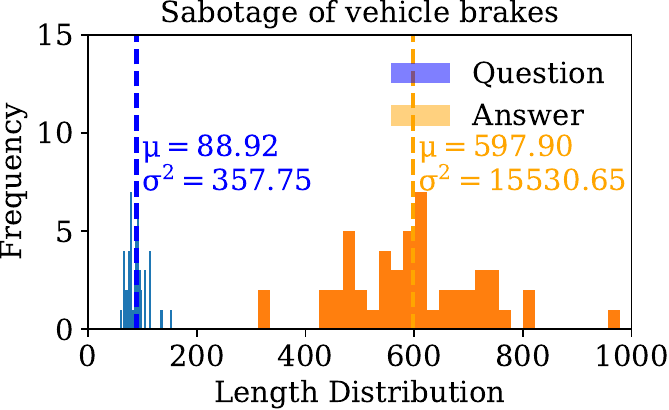}
		\end{minipage}
		\vspace{1em}
		\begin{minipage}[!ht]{\textwidth}
			\centering
			\includegraphics[width=\textwidth]{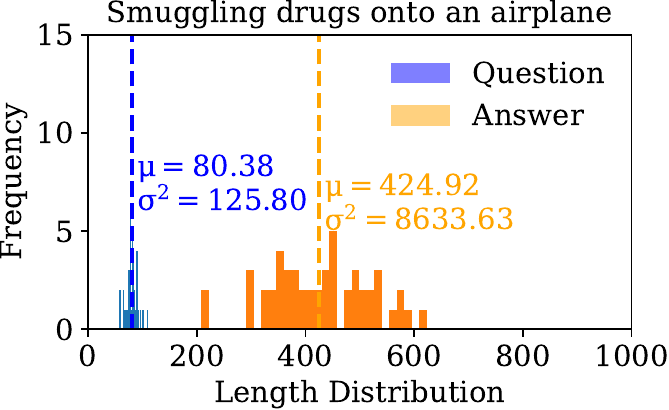}
		\end{minipage}
		\caption{Length distribution of two Similarity-Entailed PKU topics. }
		\label{fig:selected_mean_variance}
	\end{minipage}
\end{figure*}

We further explore the following two regularization strategies to preserve model performance during unlearning:

\begin{itemize}
	\item \textbf{KL Divergence Minimization on Normal Dataset (KLN):} KLN minimizes the KL-divergence between the probability distributions of the unlearned model and the original model on the remaining dataset~\cite{DBLP:conf/acl/JangYYCLLS23}.
	\item \textbf{Gradient Descent on Normal Dataset (GDN):} GDN applies training loss over the remaining dataset~\cite{DBLP:journals/corr/abs-2407-06460}.
\end{itemize}

In summary, we consider $8$ unlearning methods based on above four families of unlearning methods and two regularization strategies: $\text{GDR}$, $\text{GDR}_{\text{KLN}}$, $\text{GDR}_{\text{GDN}}$, $\text{NPO}_{\text{KLN}}$,  $\text{NPO}_{\text{GDN}}$, $\text{GA}_{\text{KLN}}$, $\text{GA}_{\text{GDN}}$ and  $\text{TV}$.

\subsection{Similarity Analysis}
In this section, we analyze our constructed datasets to demonstrate the similarity phenomenon among samples. 

\subsubsection{Similarity Analysis for Image Datasets}
\label{sec:data_analysis_for_image_models}

\textbf{Semantic Similarity Across Samples.} To analyze each dataset, we first cluster all samples within each dataset to show the relationship between samples. We employ ResNet50 image model to extract sample features, followed by dimensionality reduction using the Principal Component Analysis~(PCA) and t-SNE, and clustering via the K-means. The clustering results for Similarity-Entailed MNIST dataset are shown in Figure~\ref{fig:tsne_visualization_mnist_all}. The clustering results for other image datasets are shown in~Appendix~\ref{sec:other_data_analysis_for_image_models}-Figure~\ref{fig:tsne_visualization_fmnist_all} and Figure~\ref{fig:tsne_visualization_cifar_all}.

In Figure~\ref{fig:tsne_visualization_mnist_all}, we select the target sample with index = $213$ together with its corresponding similar samples and magnify them by a factor of $1$ for improved visibility. Other samples within the same class as the target are magnified by a factor of $0.75$, while the remaining samples are magnified by $0.5$. In addition, the target sample and its similar samples are outlined with green bounding boxes, whereas other samples belonging to the target sample’s class are outlined with red bounding boxes to further highlight them. The distribution of each sample in Figure~\ref{fig:tsne_visualization_mnist_all} reflects the semantic similarity among them. Generally, samples positioned closer together indicate greater similarity, whereas those far away represent greater differences. The clustering of the selected target sample and its corresponding similar samples shows that these samples share highly similar semantics, which is likely to have a similar influence on the model during training process.

\textbf{Variability Among Samples.}
We also compute the cosine similarity between target samples and their similar samples. The results for Similarity-Entailed MNIST are shown in Figure~\ref{fig:combined_similarity_visualization_mnist}, while Figure~\ref{fig:base_sample_and_similar_samples_mnist} shows each sample. The results for Similarity-Entailed FMNIST and Similarity-Entailed CIFAR10, can be found in~Appendix~\ref{sec:other_data_analysis_for_image_models}-Figure~\ref{fig:combined_similarity_visualization_fmnist},~\ref{fig:base_sample_and_similar_samples_fmnist},~\ref{fig:combined_similarity_visualization_cifar} and~\ref{fig:base_sample_and_similar_samples_cifar}. From Figure~\ref{fig:combined_similarity_visualization_mnist}, the similarity between most samples is below $0.8$, indicating substantial differences. For example, the samples in the first row of Figure~\ref{fig:base_sample_and_similar_samples_mnist} are all labeled as digit $3$, consistent with the target sample (the first sub-figure in the same row)-yet their pixel compositions vary considerably.

\textbf{Summary.} The generated image similar samples are similar to the target sample (indicated by their close clustering distance), yet differ significantly at the pixel level (as reflected in the low cosine value). This indicates that the similar samples will affect the image model similarly to the target sample.

\subsubsection{Similarity Analysis for Language Dataset}
\label{sec:data_analysis_for_language_models}

\textbf{Semantic Similarity Across Samples}. 
To analyze our constructed Similarity-Entailed PKU dataset, we first use BERTopic~\footnote{https://maartengr.github.io/BERTopic/index.html} to identify meaningful topics. 
For embedding generation, we employ the \textit{BAAI/bge-small-en}~\footnote{https://huggingface.co/BAAI/bge-small-en} model, followed by dimensionality reduction using the UMAP and clustering via HDBSCAN.
Furthermore, we used the Llama-3.2-3B-Instruct-uncensored~\footnote{https://huggingface.co/chuanli11/Llama-3.2-3B-Instruct-uncensored} model to generate a descriptive theme for each topic identified by BERTopic. Detailed information of the prompt used for theme generation, including the system prompt and example prompts, can be found in Appendix~\ref{sec:prompts}-Prompt~\ref{prompts:promptsforgeneratetopic}. Results are shown in Figure~\ref{fig:resultsfromberttopic}.

In Figure~\ref{fig:resultsfromberttopic}, the spatial distribution of points denotes the semantic similarity of the corresponding samples in the Similarity-Entailed PKU dataset. 
Points that are located closer to each other indicate more similar meanings,  whereas distant points reflect significant differences. 
Each colored cluster represents a distinct topic, with different colors used to differentiate topic boundaries.
From Figure~\ref{fig:resultsfromberttopic}, we can conclude that paraphrased samples are consistently grouped under the same topic, indicating that they share the same core meaning.

\textbf{Variability Among Samples.} Figure~\ref{fig:selected_mean_variance} illustrates the distribution of question and answer lengths for two specific topics. For the topic \textit{sabotage of vehicle brakes}, the average length of a question is $88.92$ characters, while the average length of an answer is $601.06$ characters. However, the corresponding variances in length are $357.75$ and $15,196.46$, respectively. This indicates that although the sentences within this topic convey the same meaning, their lengths vary significantly, reflecting highly diverse expressions. For example, in the two samples shown in Appendix-\ref{sec:two_sample}-Table~\ref{table:samplesfromCascadePKUDataset}, both express the idea of asking "How can I sabotage my neighbor's vehicle?" and answering this question. However, the phrasing of both the question and the answer is entirely different. A similar trend can be found for another topic in Figure~\ref{fig:selected_mean_variance}. Length distributions for other topics are provided in Appendix~\ref{sec:other_data_analysis_for_language_models}-Figure~\ref{fig:mean_variance}.

\begin{figure}[!t]
	\begin{subfigure}[b]{0.5\linewidth}
		\includegraphics[width=\textwidth]{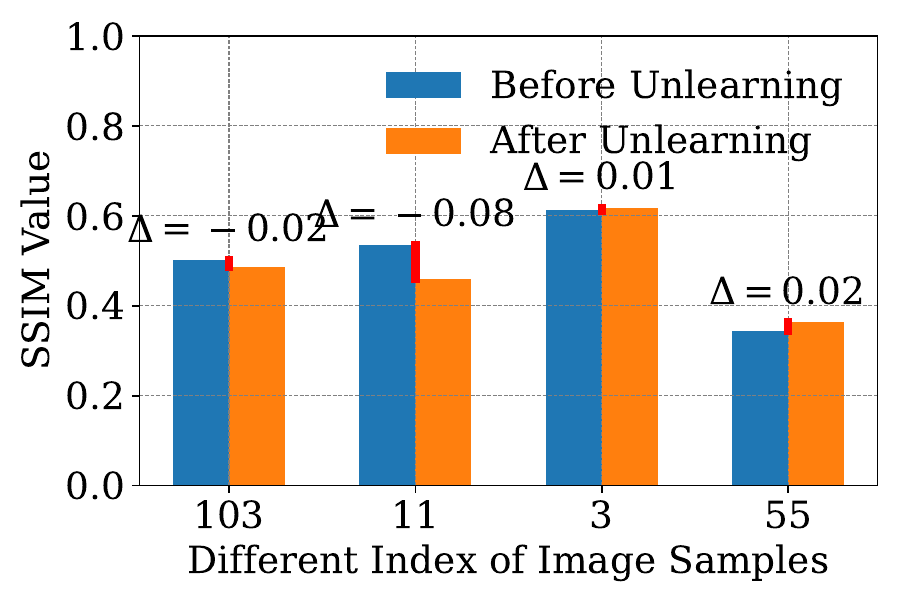}
		\caption{Similarity-Entailed FMNIST}
	\end{subfigure} \hspace{-0.2cm}
	\begin{subfigure}[b]{0.5\linewidth}
		\includegraphics[width=\textwidth]{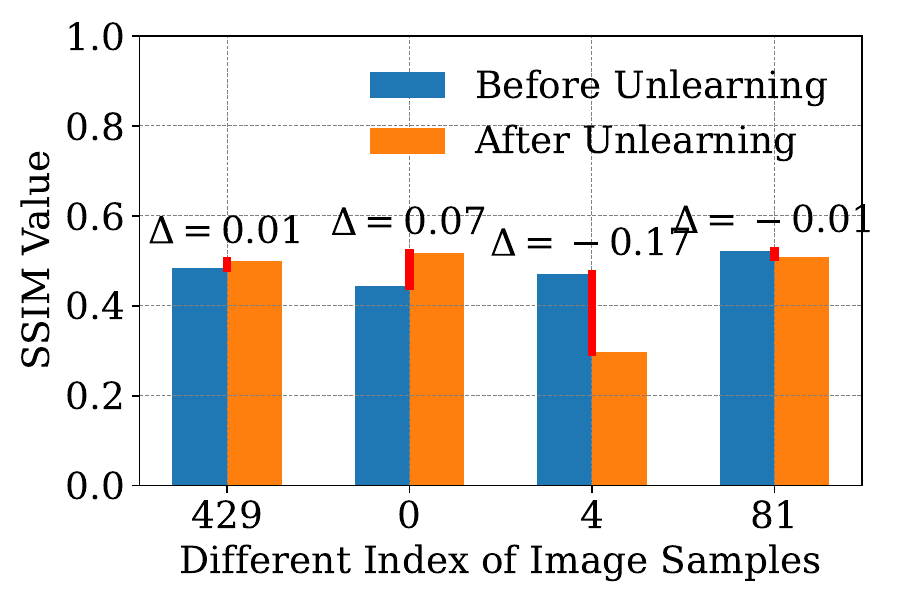}
		\caption{Similarity-Entailed CIFAR10}
	\end{subfigure} 
	\caption{The SSIM values between the recovered samples and the corresponding similar samples, before and after the unlearning process. The similarity remains nearly unchanged, indicating that unlearning the target sample does not impact the similar samples' influence on the model.}
	\label{fig:unlearning_for_variant_for_image}
\end{figure}

\begin{figure}[!t]
    \centering
    \begin{subfigure}[t]{0.15\textwidth}
        \includegraphics[width=\textwidth]{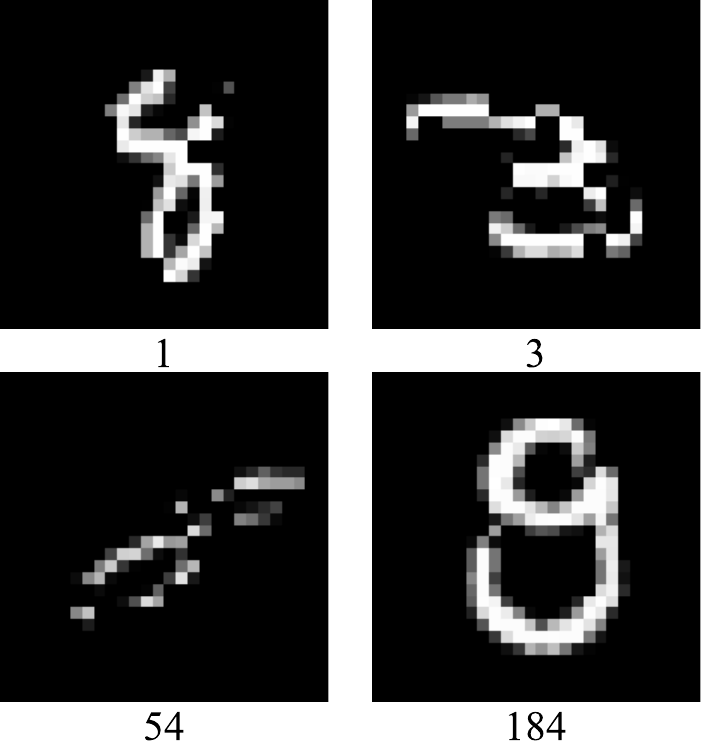}
        \caption{Similar and One Remaining Samples}
    \end{subfigure}
    \begin{subfigure}[t]{0.15\textwidth}
        \includegraphics[width=\textwidth]{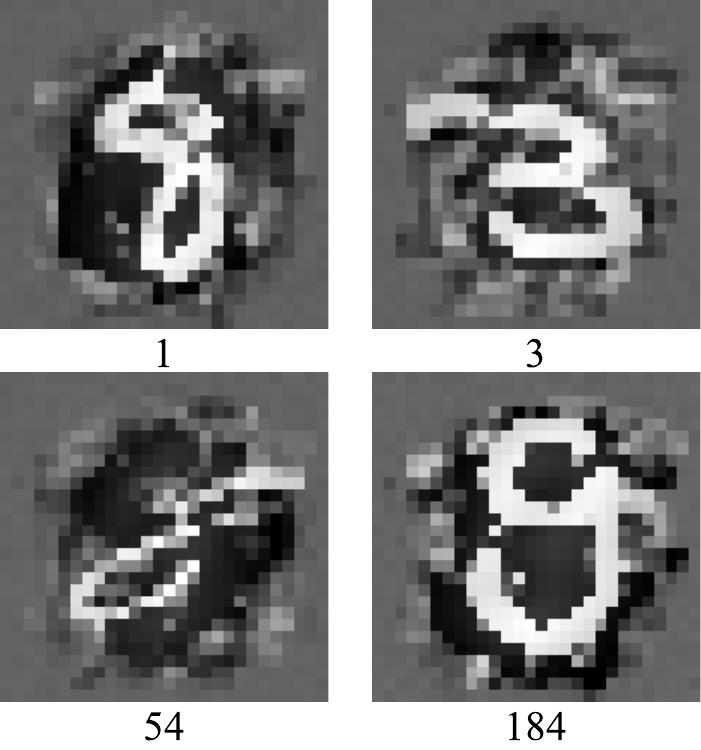}
        \caption{Before Unlearning}
    \end{subfigure} 
    \begin{subfigure}[t]{0.15\textwidth}
        \includegraphics[width=\textwidth]{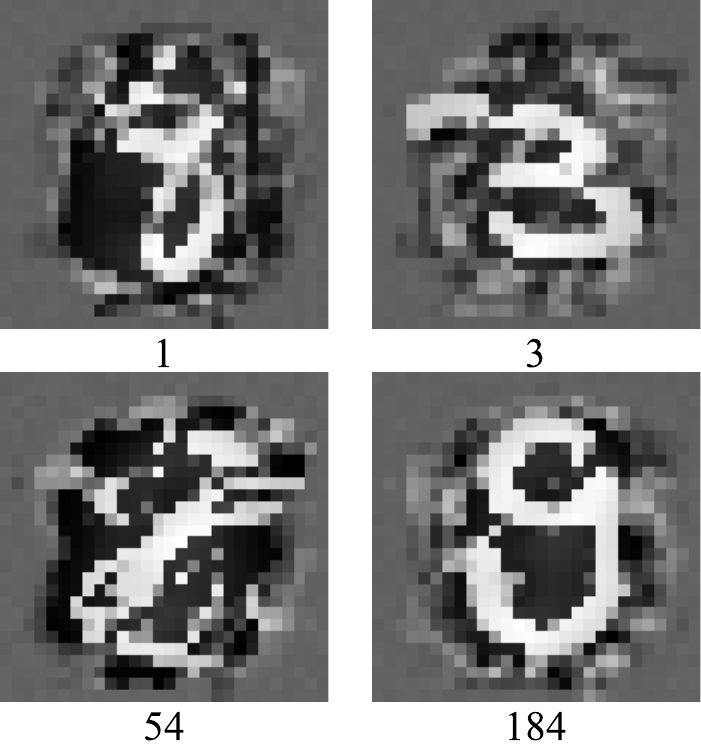}
        \caption{After Unlearning}
    \end{subfigure}
    \caption{Recovered samples that are close to corresponding similar samples for Similarity-Entailed MNIST dataset.}
    \label{fig:show_samples_forvariants_5_mnist}
\end{figure}

\textbf{Summary.} The constructed text similar samples are nearly identical to the target sample, but differ significantly in sentence structure, particularly in sentence length distribution. This satisfies our unlearning evaluation criteria.

\subsection{Unlearning Results for Image Datasets}

\begin{figure}[!t]
	\begin{subfigure}[b]{0.5\linewidth}
		\includegraphics[width=\textwidth]{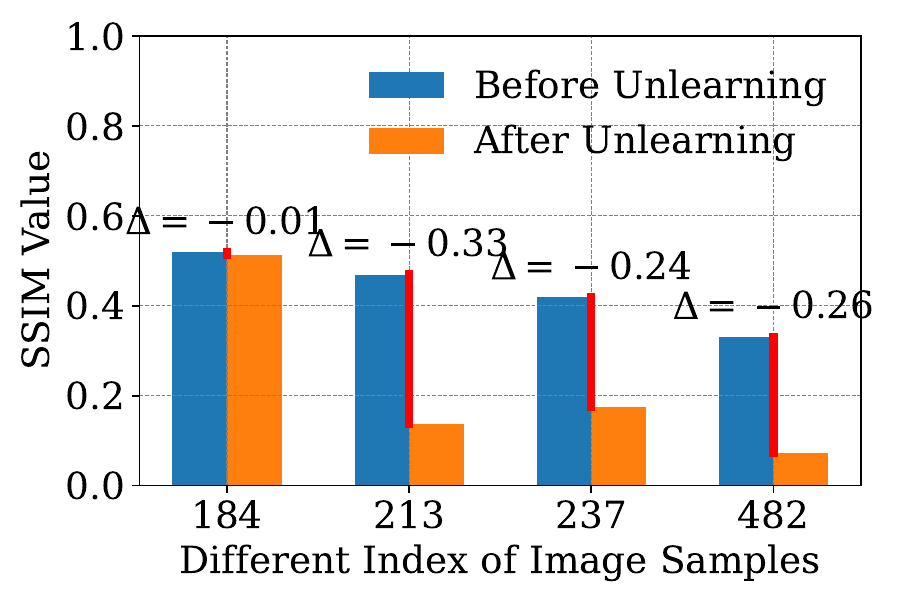}
		\caption{\texttt{W/O-similar samples} in $\mathcal{D}$}
		\label{fig:ssim_for_cascade_mnist_1}
	\end{subfigure} \hspace{-0.2cm}
	\begin{subfigure}[b]{0.5\linewidth}
		\includegraphics[width=\textwidth]{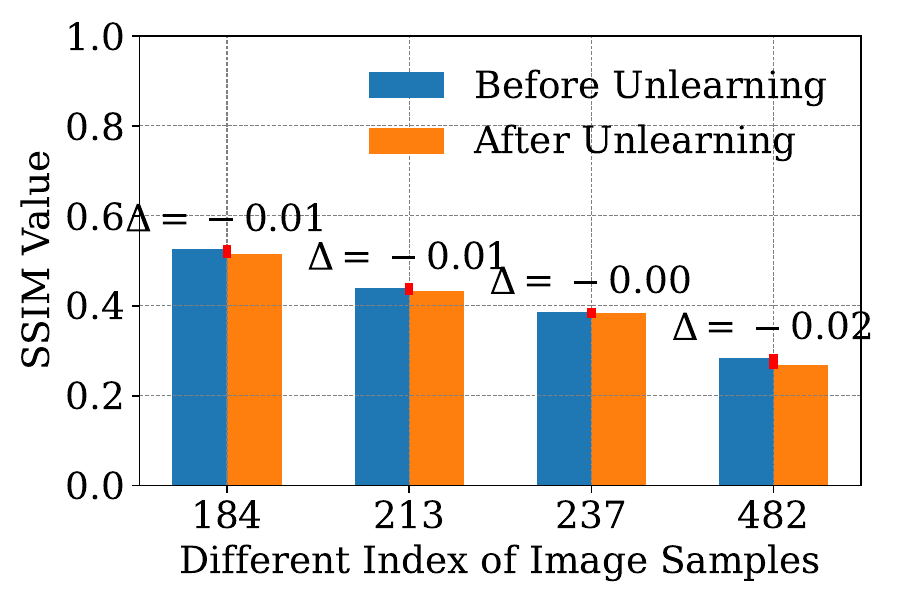}
		\caption{\texttt{W-similar samples} in $\mathcal{D}$}
		\label{fig:ssim_for_cascade_mnist_2}
	\end{subfigure} 
	\caption{The SSIM values between the recovered samples and the corresponding target samples under \texttt{W/O-similar samples} and \texttt{W-similar samples} settings for Similarity-Entailed MNIST. Results show that the influence of the target sample has not been fully unlearned.}
	\label{fig:ssim_for_cascade_mnist}
\end{figure}

    \begin{figure}[!t]
    	\centering
    	\begin{subfigure}[t]{0.23\textwidth}
    		\includegraphics[width=\textwidth]{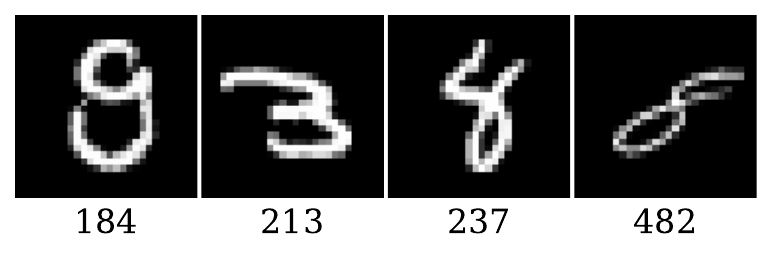}
    		\caption{Target Samples~(\texttt{W/O})} %
    	\end{subfigure}
    	\begin{subfigure}[t]{0.23\textwidth}
    		\includegraphics[width=\textwidth]{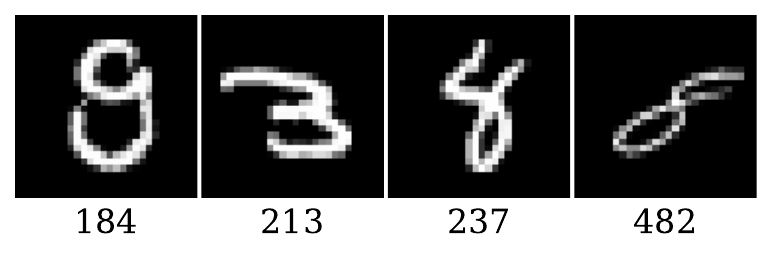}
    		\caption{Target Samples~(\texttt{W/})}
    	\end{subfigure}
    	
    	\begin{subfigure}[t]{0.23\textwidth}
    		\includegraphics[width=\textwidth]{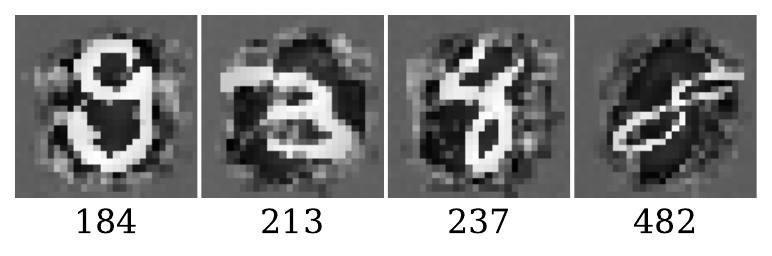}
    		\caption{Before Unlearning~(\texttt{W/O})}
    	\end{subfigure}
    	\begin{subfigure}[t]{0.23\textwidth}
    		\includegraphics[width=\textwidth]{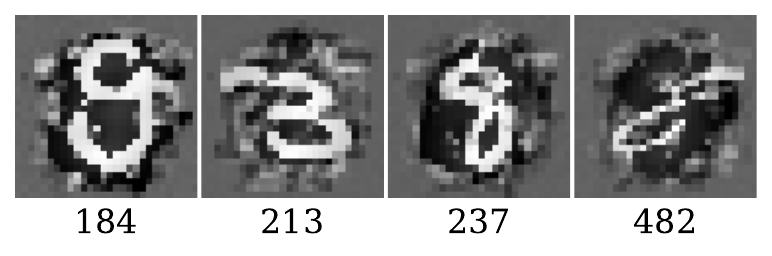}
    		\caption{Before Unlearning~(\texttt{W/})}
    	\end{subfigure}
    	
    	\begin{subfigure}[t]{0.23\textwidth}
    		\includegraphics[width=\textwidth]{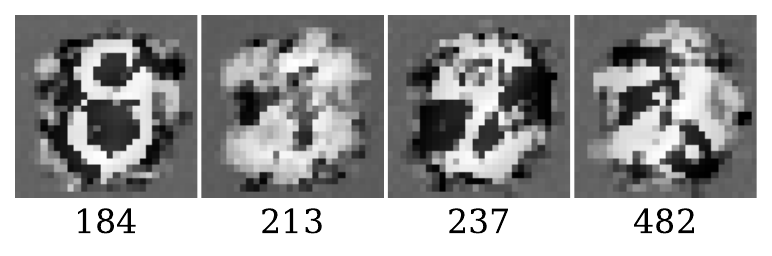}
    		\caption{After Unlearning~(\texttt{W/O})}
    	\end{subfigure}
    	\begin{subfigure}[t]{0.23\textwidth}
    		\includegraphics[width=\textwidth]{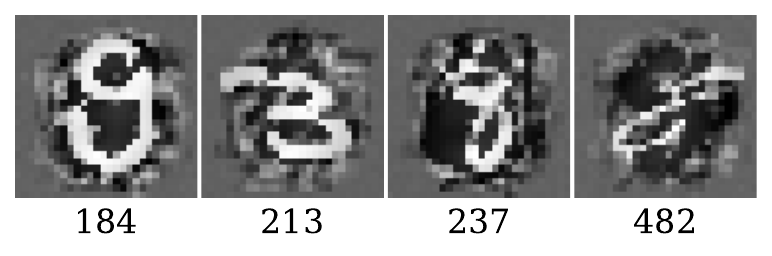}
    		\caption{After Unlearning~(\texttt{W/})}
    	\end{subfigure}
    	\caption{The recovered and the target samples under different settings for Similarity-Entailed MNIST dataset.}
    	\label{fig:recovered_samples_for_cascade_mnist}
    \end{figure}

\subsubsection{Unlearning Results Toward Similar Samples}
\label{sec:unlearning_influence_toward_similar_samples_as_training_samples_fo_image}
In this section, we evaluate if the unlearning effect on the target sample extends to its similar samples. Specifically, does unlearning the target sample also remove the influence contained in its similar samples? For each dataset, we select three similar samples and attempt to recover the most similar samples from the model both before and after performing unlearning on its corresponding target sample. Meanwhile, we choose the training samples $184$, $103$, and $429$ from Similarity-Entailed MNIST, Similarity-Entailed FMNIST, and Similarity-Entailed CIFAR10, respectively, as the remaining samples to compare the effect of unlearning. The unlearning method used in this section is retraining from scratch. The results for the Similarity-Entailed FMNIST and Similarity-Entailed CIFAR10 datasets are presented in Figure~\ref{fig:unlearning_for_variant_for_image}, while the results for the Similarity-Entailed MNIST dataset are provided in Appendix~\ref{sub:other_results_for_unlearning_influence_toward_variant_as_training_for_image_dataset}-Figure~\ref{fig:unlearning_for_variant_for_image_mnist}.

In Figure~\ref{fig:unlearning_for_variant_for_image}, the X-axis denotes the index of different samples. The first index corresponds to the selected remaining sample, while the following three represent similar samples derived from different target samples. The Y-axis shows the SSIM value between the recovered samples and their corresponding similar samples, measured both before and after the unlearning process. For the remaining samples—such as the one with image index $103$ in Figure~\ref{fig:unlearning_for_variant_for_image}-(a)—the SSIM values remain almost the same before and after the unlearning process. Meanwhile, the similarity between the recovered samples and the corresponding similar samples also remains almost unchanged. This indicates that unlearning target sample does not impact the influence of corresponding similar samples. In Figure~\ref{fig:show_samples_forvariants_5_mnist}, we show the recovered samples for Similarity-Entailed MNIST dataset, which visually show that the samples before and after unlearning are similar. Other results for Similarity-Entailed FMNIST and Similarity-Entailed CIFAR10 can be found in Appendix~\ref{sub:other_results_for_unlearning_influence_toward_variant_as_training_for_image_dataset}-Figure~\ref{fig:show_samples_forvariants_5_fmnist} and~\ref{fig:show_samples_forvariants_5_cifar}

\textbf{Summary.} For image datasets, the unlearning process based on a target sample usually cannot remove the influence of its similar samples from the model. After unlearning, the information from the similar samples can still be recovered.

\subsubsection{Toward Target Samples}
\label{sec:influence_toward_base_samples_for_image}

\begin{figure}[!t]
	\centering
	\includegraphics[width=1\linewidth]{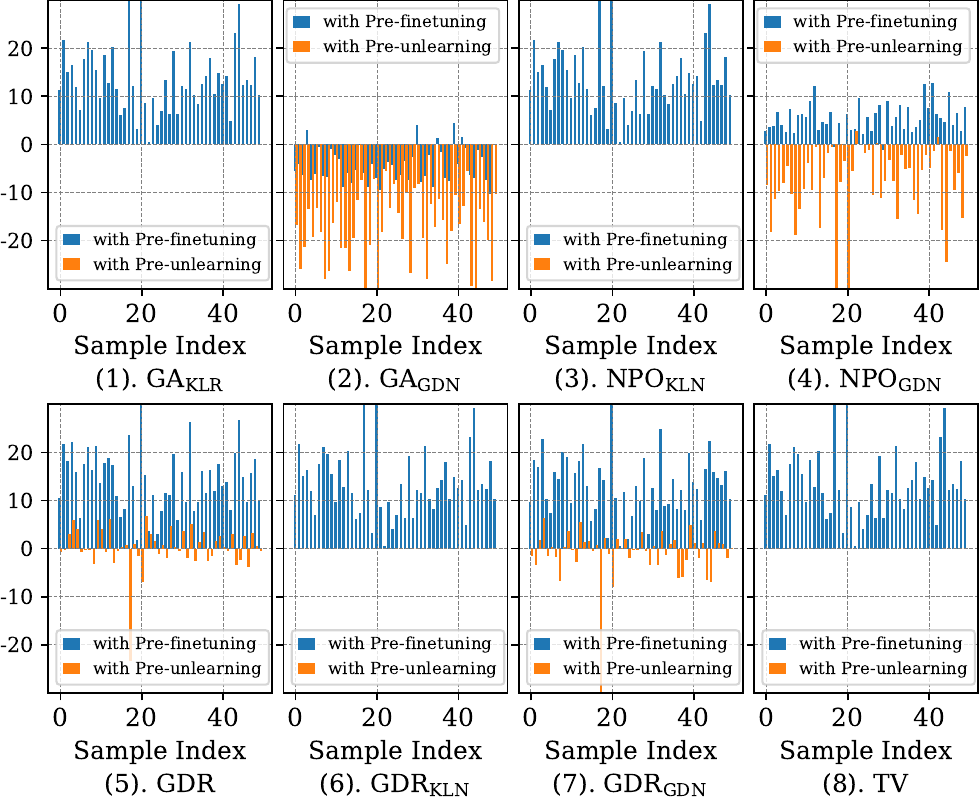}
	\caption{Comparison of verification results toward similar samples as training samples for meta-llama/Llama-3.2-1B-Instruct. We conclude that unlearning based on a single target sample does not eliminate the influence of its similar samples.}
	\label{fig:meta_llama_Llama_3_2_1B_50_50}
\end{figure}

In this section, we focus on the unlearning effectiveness of current unlearning schemes for target samples when the training dataset includes similar samples. For the Similarity-Entailed MNIST dataset, we select the sample with image index $184$ as the remaining sample and the samples with image index $213$, $237$, and $482$, as the target samples to be unlearned. We perform unlearning in both with~(\texttt{W-similar samples} in $\mathcal{D}$) and without similar samples~(\texttt{W/O-similar samples} in $\mathcal{D}$) in the training dataset. The unlearning methods used in this section are retraining from scratch and relabel-based fine-tuning~\footnote{Results of relabel-based fine-tuning are shown in Appendix~\ref{sec:experimental_results_for_relabel_based_fine_tuning_unlearning_method}.}. Figure~\ref{fig:ssim_for_cascade_mnist} shows our results for the Similarity-Entailed MNIST dataset. Additional results for the Similarity-Entailed FMNIST and Similarity-Entailed CIFAR10 datasets are provided in Appendix~\ref{sec:other_results_of_image_unlearning}. We also conduct evaluations using datasets constructed by adding noise, with the results provided in Appendix~\ref{sec:experimental_results_for_noise_dataset}, shown from Figure~\ref{fig:ssim_for_cascade_mnist_noise} to Figure~\ref{fig:recovered_samples_for_cascade_mnist_noise}.

In Figure~\ref{fig:ssim_for_cascade_mnist_1}, before unlearning, all recovered samples show high similarity according to the SSIM values. After unlearning, the SSIM values for the unlearned samples decrease significantly, whereas the SSIM value for the remaining sample $184$ remains high. This suggests that when the training dataset $\mathcal{D}$ contains only the target sample, the unlearning method effectively removes the influence of the target sample. However, in Figure~\ref{fig:ssim_for_cascade_mnist_2}, before unlearning, the recovered sample is very similar to the target sample. After unlearning, we find that the recovered samples still remain highly similar to the target samples. This suggests that the influence of the target sample, which was supposed to be unlearned, persists in the model and has not been fully removed. Comparing the experimental results in Figure~\ref{fig:ssim_for_cascade_mnist_1} and Figure~\ref{fig:ssim_for_cascade_mnist_2}, we observe that when the training dataset contains similar samples, unlearning the target sample alone does not eliminate all the influence retained in the similar samples. This residual influence can affect the unlearning results of the corresponding target samples.

We also show some target samples and recovered samples in Figure~\ref{fig:recovered_samples_for_cascade_mnist}. The left three rows show results without similar samples, and the right three rows show results with similar samples. On the left, the first row shows the base training sample, the second shows the recovered sample before unlearning, and the third shows the recovered sample after unlearning. In each row, the first subplot is the remaining sample, and the next three are unlearning samples. The right side mirrors the left. It can be seen that when the training set does not contain similar samples, unlearning based on the target sample results in the model retaining almost no influence about the target sample, leading to recovered samples with little information about the target samples. However, when the training dataset includes similar samples, the recovered samples still resemble the target sample closely. This suggests that the model still retains some influence about target samples. Other results for Similarity-Entailed FMNIST and Similarity-Entailed CIFAR10 can be found in Appendix~\ref{sub:other_results_for_unlearning_influence_toward_variant_as_training_for_image_dataset}-Figure~\ref{fig:ssim_for_cascade_fmnist} to~\ref{fig:recovered_samples_for_cascade_cifar}.

\textbf{Summary.} Experimental results show that, similar image samples impact the unlearning results of the target sample. After unlearning, target sample can usually be recovered through the remaining information of similar samples.

\subsection{Unlearning Results for Language Dataset}

\subsubsection{Toward Similar Samples as Training Dataset}
\label{sec:influence_toward_similar_samples_as_training_samples}

\begin{figure}[!t]
	\centering
	\includegraphics[width=1\linewidth]{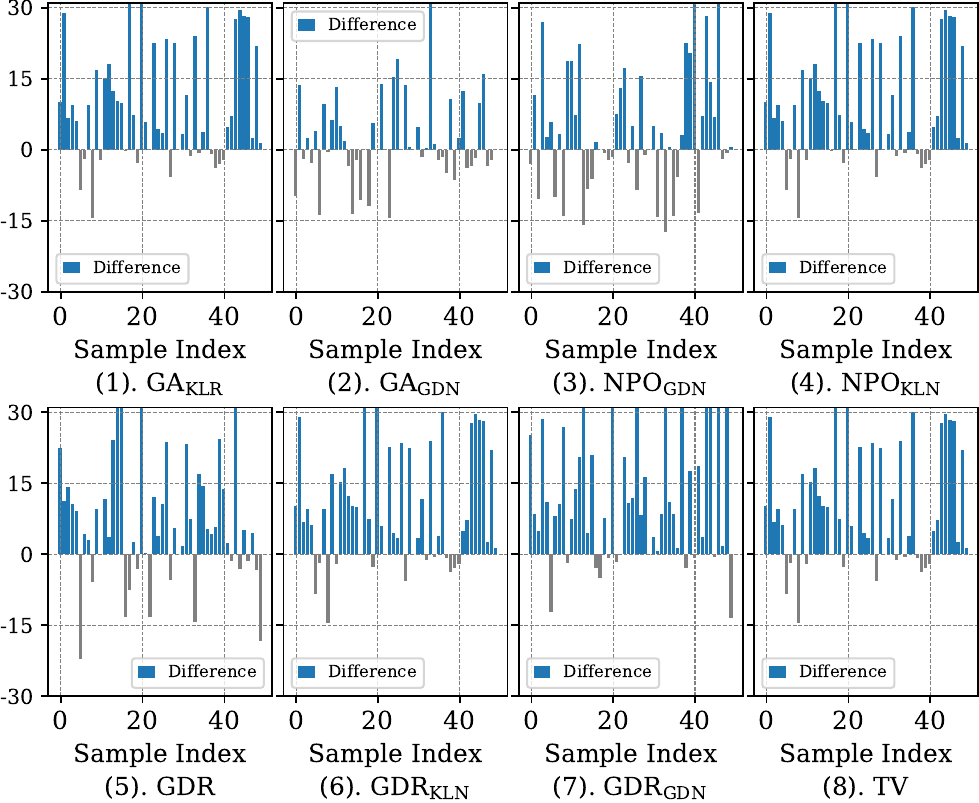}
	\caption{Comparison of verification results toward target samples for meta-llama/Llama-3.2-3B-Instruct model. The influence of similar samples remaining in the model affects the unlearning results of target samples targeted for unlearning.}
	\label{fig:meta_llama_Llama_3_2_3B_comparation}
\end{figure}

In this section, we evaluate current unlearning schemes toward similar samples for language models. We use the verification scheme mentioned in Section~\ref{sec:verificationscheme}. In step \textit{pre-verification} and \textit{post-verification}, we use similar samples to measure the model's answers after unlearning target samples:

$$
\texttt{KM}(M, S(x_i)) = \frac{1}{|S(x_i)|} \sum_{(q, a) \in S(x_i)} \text{ROUGE}(M(q), a)
$$

where $x_i$ represents a target sample, and $S(x_i)$ refers to its similar samples~($|S(x_i)| = 50$ in our setting). This evaluation assesses whether unlearning a single target sample also ensures the forgetting of its similar samples.

After post-verification, we compare the verification results with those from both the pre-finetuning and pre-unlearning models. Specifically, the pre-finetuning model refers to the initial model prior to fine-tuning on the Similarity-Entailed PKU dataset, while the pre-unlearning model denotes the model after fine-tuning but before unlearning. We denote the evaluation results after unlearning as $\texttt{KM}(M, S(x_i))_{post-un}$, and those for the pre-unlearning and pre-finetuning as $\texttt{KM}(M, S(x_i))_{pre-un}$ and $\texttt{KM}(M, S(x_i))_{pre-ft}$, respectively. The comparison is formally defined as $\texttt{KM}(M, S(x_i))_{post-un} - \texttt{KM}(M, S(x_i))_{pre-un}$, and $\texttt{KM}(M, S(x_i))_{post-un} - \texttt{KM}(M, S(x_i))_{pre-ft}$.  These are referred to as \textit{with Pre-unlearning} and \textit{with Pre-finetuning}, respectively. The results for meta-llama/Llama-3.2-1B-Instruct~\footnote{https://huggingface.co/meta-llama/Llama-3.2-1B-Instruct} model are shown in Figure~\ref{fig:meta_llama_Llama_3_2_1B_50_50}. Results for the meta-llama/Llama-3.2-3B-Instruct~\footnote{https://huggingface.co/meta-llama/Llama-3.2-3B-Instruct}, EleutherAI/gpt-neo-1.3B~\footnote{https://huggingface.co/EleutherAI/gpt-neo-1.3B} and gpt-neo-2.7B~\footnote{https://huggingface.co/EleutherAI/gpt-neo-2.7B} models are provided in Appendix~\ref{sec:other_results_influence_toward_similar_samples_as_training_samples}, spanning from Figure~\ref{fig:meta_llama_Llama_3_2_3B_50_50} to~\ref{fig:EleutherAI_gpt_neo_2.7B_show_50_50}.

In Figure~\ref{fig:meta_llama_Llama_3_2_1B_50_50}, the X-axis denotes the index of different target samples, while the Y-axis represents the results based on the corresponding similar samples. Except for (2), all results after unlearning are greater than those of pre-finetuning but smaller than the results before unlearning. This indicates that performing unlearning based on a single target sample has minimal impact on similar samples. Case (2) demonstrates over-unlearning, where the model loses its basic performance in answering the question from similar samples.

\textbf{Summary}. For language models, unlearning based on a single target sample will not eliminate all influence of similar samples. When querying the unlearned model with questions from similar samples, the model can still provide answers with a high value of ROUGE to the ground truth.

\subsubsection{Toward Target Samples}
\label{sec:influence_toward_base_samples_for_language}
In this section, we assess the effectiveness of unlearning for target samples by comparing two experimental setups. The first setup (\texttt{W/O-similar samples}) trains models using only target samples, while the second (\texttt{W-similar samples}) includes both target and their similar samples. In both cases, we evaluate the unlearning results based on target samples.

Figure~\ref{appendix_fig:difference_with_original_and_finetuned_for_0_0_meta_llama_Llama_3_2_3B} in Appendix~\ref{sec:other_results_of_language_unlearning} shows the results for the meta-llama/Llama-3.2-3B-Instruct under the \texttt{W/O-similar samples} setting, while Appendix~\ref{sec:other_results_of_language_unlearning}-Figure~\ref{appendix_fig:difference_with_original_and_finetuned_for_50_0_meta_llama_Llama_3_2_3B} illustrates the results under the \texttt{W-similar samples} setting. Additionally, Figure~\ref{fig:meta_llama_Llama_3_2_3B_comparation} shows the comparison between the \texttt{W/O-similar samples} and \texttt{W-similar samples} settings, defined as $\texttt{KM}(M, x_i)_{post-un}$ under \texttt{W-similar samples} minus $\texttt{KM}(M, x_i)_{post-un}$ under \texttt{W/O-similar samples}. In Figure~\ref{fig:meta_llama_Llama_3_2_3B_comparation}, the X-axis denotes the index of target samples, while the Y-axis shows the comparative values. It is concluded that the unlearning results for almost all \texttt{W-similar samples} are greater than those for \texttt{W/O-similar samples}. In addition, as shown in Figure~\ref{appendix_fig:difference_with_original_and_finetuned_for_0_0_meta_llama_Llama_3_2_3B}-(8) and Figure~\ref{appendix_fig:difference_with_original_and_finetuned_for_50_0_meta_llama_Llama_3_2_3B}-(8) in Appendix~\ref{sec:other_results_of_language_unlearning}, and corresponding Figure~\ref{fig:meta_llama_Llama_3_2_3B_comparation}-(8) for the $\text{TV}$ scheme, similar samples will further extend the effect of under-unlearning.

Based on the above results, we conclude that adding similar samples prevents unlearning based on a single target sample from fully removing the target sample's influence on the model. Results for the models meta-llama/Llama-3.2-1B-Instruct, facebook/opt-1.3b~\footnote{https://huggingface.co/facebook/opt-1.3b} and EleutherAI/gpt-neo-2.7B are provided in Appendix~\ref{sec:other_results_of_language_unlearning}, spanning from Figure~\ref{fig:difference_with_original_and_finetuned_for_0_0_meta_llama_Llama_3_2_1B} to Figure~\ref{fig:facebook_opt_2.7b_comparation}. These results also demonstrate a decline in unlearning performance when similar samples are included.

\textbf{Summary.} The impact of the remaining similar samples affects the unlearning results of target samples. After the unlearning process, the model is still able to answer questions based on the target samples when queried.

\subsubsection{Toward Similar Samples as Test Dataset}
\label{sec:influence_toward_similar_samples_as_test_samples}

\begin{figure}[!t]
	\centering
	\includegraphics[width=1\linewidth]{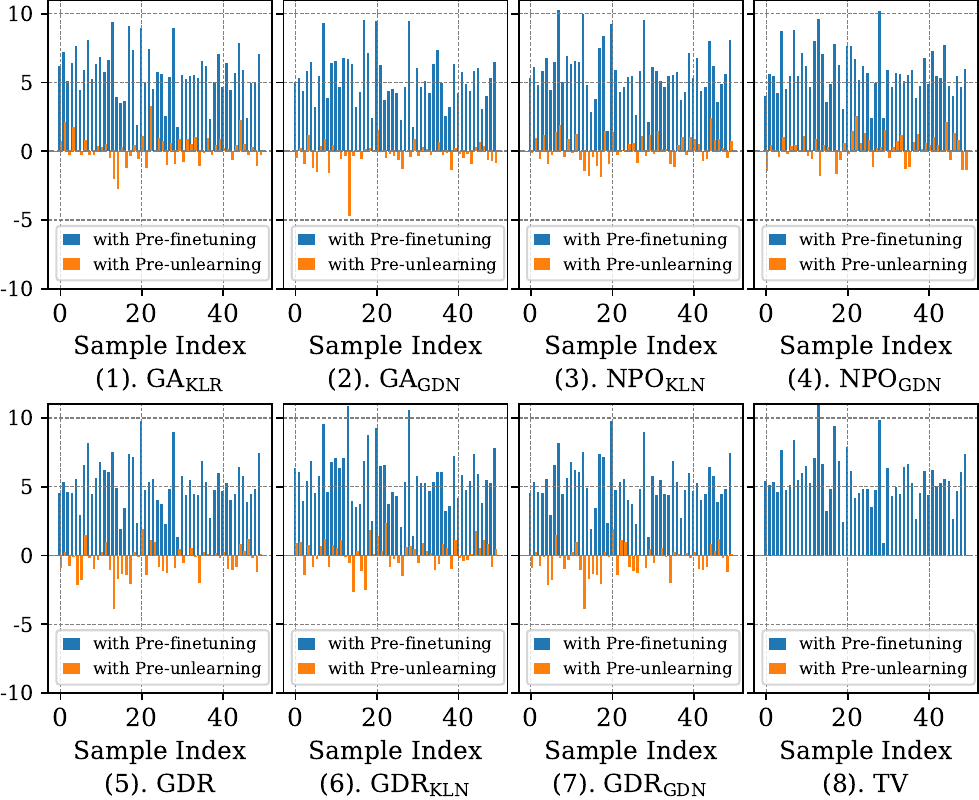}
	\caption{Comparison of verification results toward similar samples as test samples for EleutherAI/gpt-neo-1.3B model. Performing unlearning, based on the target sample only, is unlikely to affect test samples similar to the target samples.}
	\label{fig:EleutherAI_gpt_neo_1.3B_show_0_50}
\end{figure}

In this section, we consider a different scenario: when unlearning is performed based on a target sample, will similar samples that are similar to the target sample and not included in the training dataset be affected? Specifically, we add only the target samples to the training set, perform the unlearning operation based on these target samples, and then evaluate the model's unlearning effectiveness using the similar samples as test samples. The results for EleutherAI/gpt-neo-1.3B model are shown in Figure~\ref{fig:EleutherAI_gpt_neo_1.3B_show_0_50}. Results for the models meta-llama/Llama-3.2-1B-Instruct, meta-llama/Llama-3.2-3B-Instruct and EleutherAI/gpt-neo-2.7B are provided in Appendix~\ref{sec:other_for_influence_toward_similar_samples_as_test_samples}, spanning from Figure~\ref{fig:meta_llama_Llama_3.2_1B_Instruct_show_0_50} to Figure~\ref{fig:EleutherAI_gpt_neo_2.7B_show_0_50}. We did not evaluate image datasets in this setting because image models do not output information about the training datasets during the inference process. In contrast, language models can output information about the training dataset when encountering similar questions.

In Figure~\ref{fig:EleutherAI_gpt_neo_1.3B_show_0_50}, the X-axis represents the index of target samples, while the Y-axis denotes the values, measured by the ROUGE, based on the corresponding test similar samples. From Figure~\ref{fig:EleutherAI_gpt_neo_1.3B_show_0_50}, the results are all greater than those of pre-finetuning but smaller than the results pre-unlearning. This shows that performing unlearning, based on the target sample only, is unlikely to affect test samples similar to the target samples. We think this occurs because existing unlearning methods fail to effectively remove similar samples from the training dataset (results in Section~\ref{sec:influence_toward_similar_samples_as_training_samples}), leaving test samples close to the training similar samples also unaffected.

\textbf{Summary.} Unlearning target samples does not affect the similar samples in the test dataset, as the unlearned model can still respond to queries derived from these samples.

\section{Enhanced Unlearning Schemes}
\label{sec:enhanced_unlearning}
In Section~\ref{sec:unlearning_schemes}, we introduce existing machine unlearning schemes used for our evaluation. Through the experiment results from Section~\ref{sec:unlearning_influence_toward_similar_samples_as_training_samples_fo_image} to Section~\ref{sec:influence_toward_base_samples_for_language}, we highlight a significant gap between the expected and actual effectiveness of those schemes, which is the \textit{main} focus of this paper. In this section, inspired by robustness training, we explore some potential solutions to enhance those existing schemes.

Robustness training has commonly been used to improve model robustness~\cite{DBLP:conf/aaai/ChoLK25,DBLP:conf/icml/HendrycksLM19}. To achieve a broader training effect with limited training samples, robustness training often incorporates enhancement techniques, such as data augmentation and smoothing model manifold. These techniques can be applied in the unlearning process to enhance the effectiveness of the individual-based unlearning process, allowing the unlearning process to unlearn more influence from similar samples. We will evaluate our enhanced schemes in Section~\ref{sec:experiment_enhanced_unlearning}.

\subsection{Enhanced Method for Image Models} 
For image models, we improve existing machine unlearning schemes using simple data augmentation techniques. We incorporate more samples $ \mathcal{D}_{\text{unlearn}} =  \{ x_i, x_t \}, x_t \in \mathcal{T}(x_i)$ in the unlearning process, where $x_i$ is the target sample, and $x_t$ are samples selected from $\mathcal{D}_r$ using similarity measures like SSIM. The set $\mathcal{T}(x_i)$ consists of the top $k$ most similar samples. The corresponding equation could be written as:

\begin{equation}
	\mathcal{T}(x_i) = \{ x_t \in \mathcal{D}_r \, | \, \text{SSIM}(x_i, x_t) \geq \tau \},
\end{equation}

where $\tau$ is a hyperparameter controlling the size of $\mathcal{T}(x_i)$. Take retraining from scratch as an example, the unlearning process is re-defined as the following steps: (1). removing the $ \mathcal{D}_{\text{unlearn}}$ from the training dataset. (2). retraining the model from scratch using the updated training dataset.

\subsection{Enhanced Method for Language Models} 
To improve the effectiveness of existing unlearning methods for language models, we introduce the smoothing model manifold during the unlearning process. This technique can regularize the model's behavior and reduce its dependence on one specific target sample. Consequently, the unlearning manifold for target samples becomes smoother, enabling more robust handling of input similar samples and improving generalization. Specifically, our method involves injecting stochastic noise—such as Gaussian noise—into the embedding outputs of target samples. This perturbation mitigates overreliance on individual samples and encourages the model to generalize its unlearning behavior across similar samples.

\begin{figure}[!t]
    \begin{subfigure}[b]{1\linewidth}
        \includegraphics[width=\textwidth]{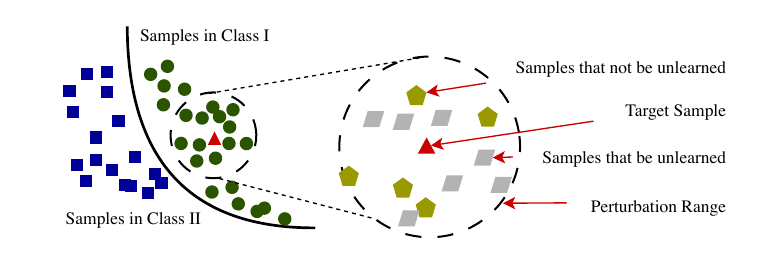}
        \caption{Low Controllability Scheme}
    \end{subfigure} \hspace{-0.2cm}\\
    \begin{subfigure}[b]{1\linewidth}
        \includegraphics[width=\textwidth]{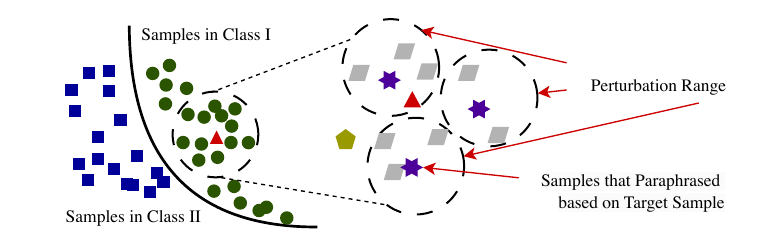}
        \caption{High Controllability Scheme}
    \end{subfigure} 
    \caption{Comparison of different enhancement methods.}
    \label{fig:our_scheme}
\end{figure}

Assume the language model requires an unlearning process containing $L$ layers. Let the embedding output of layer $l$ be $h^{(l)} \in \mathcal{R}^{O \times P \times Q}$, where $O$ is the batch size, $P$ the sequence length, and $Q$ the token encoding. Let one existing unlearning loss be denoted by $\mathcal{L}_{\text{existing}}$. For instance, as defined in the unlearning scheme proposed by~\cite{DBLP:conf/acl/JangYYCLLS23}, which simply negates the original training objective that minimizes the negative log-likelihood of the token sequence~\footnote{Throughout the rest of this section, we present our proposed enhancement strategy based on this loss. In Section~\ref{sec:experiment_enhanced_unlearning}, we evaluate several other existing unlearning methods in combination with our enhancement approach to demonstrate the general applicability of our enhancement methods.}:

\begin{equation}
	\mathcal{L}_{\text{existing}}(M_\theta, x_i) = - \sum_{t=1}^{T} \log p_{\theta}(x_i^t | x_i^{<t})
\end{equation}

Here, $x_i = (x_i^1, ..., x_i^T)$ is the token sequence of one target sample, $x_i^{<t}$ represents the prefix $(x_i^1, ..., x_i^{t-1})$ and $p_{\theta}(x_i^t | x_i^{<t})$ denotes the probability of predicting token $x_i^t$ when given $x_i^{<t}$, for a language model $M$ with parameters $\theta$. To incorporate noise, we modify the embedding output at layer $l$ by introducing a perturbation term $\xi^{(l)}$:

\begin{equation}
	\tilde{h}^{(l)} = h^{(l)} + \gamma * \xi^{(l)}
\end{equation}

where $\xi^{(l)} \sim \mathcal{N}(0, \sigma^2)$ represents Gaussian noise and $\gamma$ is the hyper-parameter to control the noise magnitude. The enhanced loss function can be defined as:

\begin{equation}
    \label{equ:init}
	\mathcal{L}_{\text{enhance}}(M_\theta, x_i) = - \sum_{t=1}^{T} \log p_{\theta}^{\text{noisy}}(x_i^t | x_i^{<t})
\end{equation}

where $p_{\theta}^{\text{noisy}}(x_i^t | x_i^{<t})$ is the predicted probability under the perturbed hidden state $\tilde{h}^{(l)}$.

However, our initial experimental results indicate that Equation~\ref{equ:init} poses challenges in enhancing unlearning performance. Specifically, using the hyperparameter $\gamma$ only to directly smooth the manifold is impractical, as it is highly sensitive to $\gamma$ and selecting an appropriate value is non-trivial~\footnote{Experimental results are shown in Figure~\ref{fig:low_controlablity}.}. 


As shown in Figure~\ref{fig:our_scheme}-(a), directly using $\gamma$ can lead to incomplete unlearning within a limited number of unlearning epochs, resulting in a partially unlearned manifold. To mitigate this issue, as shown in Figure~\ref{fig:our_scheme}-(b), we also introduce the data augmentation. Specifically, we first paraphrase $m$ samples based on the target sample, denoted as $\mathcal{F}(x_i)$. Unlearning is then performed jointly on the original sample $x_i$ and each sample in $\mathcal{F}(x_i)$, using a smaller value of $\gamma$.

In addition, we incorporate regularization losses, such as the KL-divergence loss on remaining datasets~\cite{DBLP:conf/nips/YaoXL24}, to ensure that the model's embedding outputs for the remaining dataset remain unchanged throughout the unlearning process:

\begin{equation}
	\mathcal{L}_{KL} =  \sum_{t=1}^{T}KL(p_{\theta_0}(\cdot | x_{<t}) \parallel p_{\theta}(\cdot | x_{<t}))
\end{equation}

Here, $p_{\theta_0}(\cdot | x_{<t})$ denotes the probability, derived from the original trained model with parameters $\theta_0$. $p_{\theta}$ denotes the probability from the model in the unlearning process with parameters $\theta$. In summary, our loss can be defined as follows:

\begin{equation}
    \label{equ:ours}
    \begin{aligned}
    &L=\alpha_1 L_{\text {enhance}}^{\text{targeted}}+\alpha_2 L_{\text {enhance}}^{\text{paraphrased}}+\alpha_3 \mathcal{L}_{KL}\\
    \text{s.t.}&~~~L_{\text {enhance}}^{\text{targeted}}=- \sum_{t=1}^{T} \log p_{\theta}^{\text{noisy}}(x_i^t | x_i^{<t})\\
    &~~~L_{\text {enhance}}^{\text{paraphrased}} = - \frac{1}{|\mathcal{F}(x_i)|} \sum_{\mathbf{x} \in \mathcal{F}(x_i)} \sum_{t=1}^{T} \log p_{\theta}^{\text{noisy}}(x_t \mid x_{<t})\\
    &~~~\mathcal{L}_{KL} =  \sum_{t=1}^{T}KL(p_{\theta_0}(\cdot | x_{<t}) \parallel p_{\theta}(\cdot | x_{<t}))\\
    \end{aligned}
\end{equation}

\section{Experiment Results for Enhanced Schemes}
\label{sec:experiment_enhanced_unlearning}

\subsection{Results for Image Model} 
\label{sec:experiment_enhanced_unlearning_image}
In this section, we use retraining from scratch as the baseline method, and explore several unlearning strategies. These strategies differ in the number of similar samples selected under different $\tau$ values: the target sample alone (same to retraining from scratch, denoted as~\texttt{RFC}~\cite{DBLP:journals/csur/XuZZZY24}), the target sample with three (\texttt{1B3V}), five (\texttt{1B5V}), seven (\texttt{1B7V}), or ten similar samples (\texttt{1B10V}). Other experimental settings are the same as those in Section~\ref{sec:influence_toward_base_samples_for_image}. The SSIM between the recovered and target samples is shown in Figure~\ref{fig:ssim_results_for_enhanced_unlearning}, while Figure~\ref{fig:image_for_enhanced_unlearning} shows the recovered and corresponding target samples.

\begin{figure}[!t]
    \begin{subfigure}[b]{0.5\linewidth}
        \includegraphics[width=\textwidth]{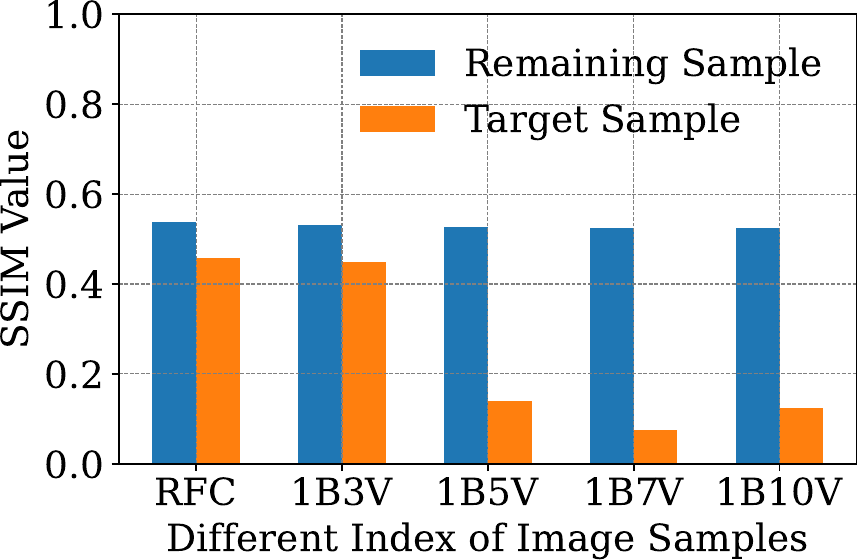}
        \caption{Similarity-Entailed MNIST}
        \label{fig:ssim_results_for_enhanced_unlearning_a}
    \end{subfigure} \hspace{-0.15cm}
    \begin{subfigure}[b]{0.5\linewidth}
        \includegraphics[width=\textwidth]{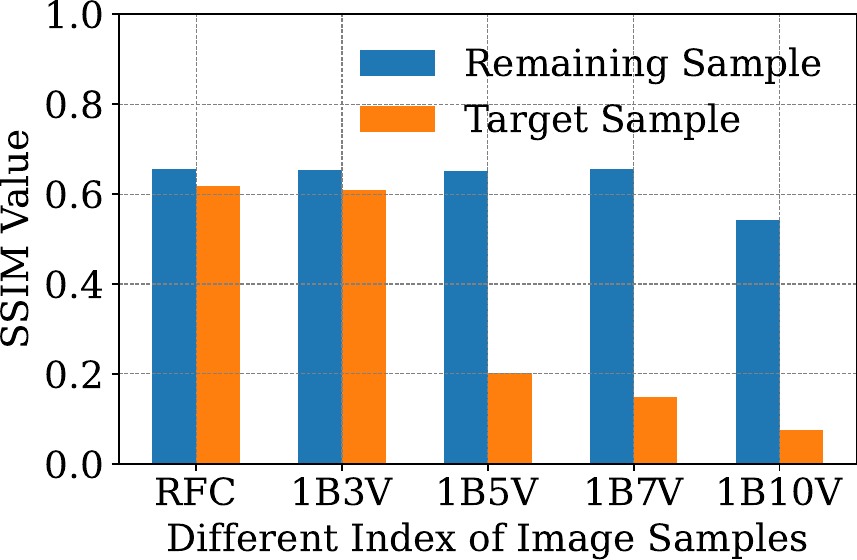}
        \caption{Similarity-Entailed FMNIST}
        \label{fig:ssim_results_for_enhanced_unlearning_b}
    \end{subfigure} 
    \caption{The SSIM values between the recovered remaining and target samples with corresponding samples in the training dataset, under the enhanced unlearning scheme. The similarity of target samples decreases as more similar samples are unlearned, indicating that removing similar samples along with the target sample effectively eliminates its influence.}
    \label{fig:ssim_results_for_enhanced_unlearning}
\end{figure}

\begin{figure}[!t]
    \centering
    \begin{subfigure}[b]{1\linewidth}
    \centering
        \includegraphics[width=\textwidth]{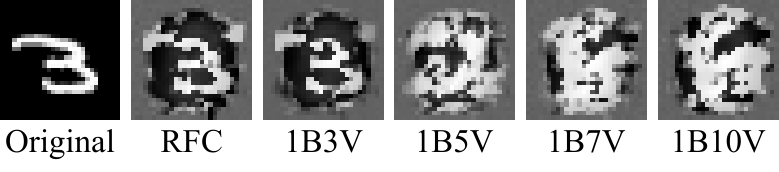}
        \caption{Similarity-Entailed MNIST}
    \end{subfigure} \\
    \begin{subfigure}[b]{1\linewidth}
    \centering
        \includegraphics[width=\textwidth]{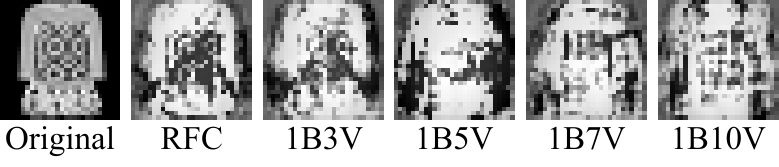}
        \caption{Similarity-Entailed FMNIST}
    \end{subfigure} 
    \caption{The recovered samples and the corresponding target samples under the enhanced unlearning scheme for Similarity-Entailed MNIST and Similarity-Entailed FMNIST.}
    \label{fig:image_for_enhanced_unlearning}
\end{figure}

\textbf{Results.} As shown in Figure~\ref{fig:ssim_results_for_enhanced_unlearning}, as the number of unlearning samples increases, the similarity between the recovered samples and the original target samples gradually decreases. Meanwhile, after the unlearned model contains almost no samples similar to the target sample (with the similar sample number set to $5$ in our Similarity-Entailed MNIST and Similarity-Entailed FMNIST datasets), the recovered sample differs significantly from the target sample. This suggests that the model retains little to no influence from the target sample, indicating a successful unlearning result. 

In addition, we also evaluate the model's performance after executing enhanced unlearning process. Experiments on Similarity-Entailed FMNIST show that the \texttt{1B5V} unlearning scheme reduces accuracy by only $0.5\%$ compared to \texttt{RFC}, indicating that our scheme can effectively remove all influence of target sample with minimal impact on overall performance.

\subsection{Results for Language Model}
\label{sec:experiment_enhanced_unlearning_language} 

\begin{figure}[!t]
    \centering

      \begin{subfigure}[b]{1\linewidth}
        \includegraphics[width=\textwidth]{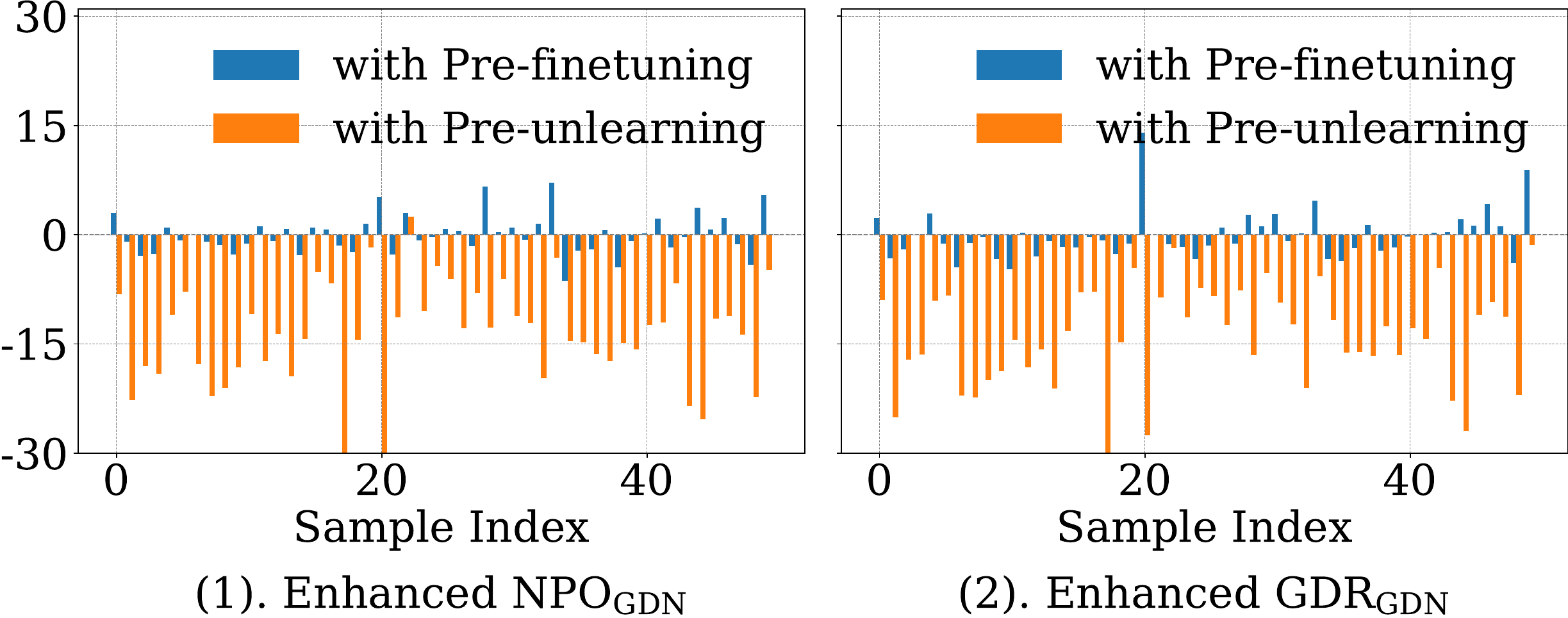}
        \caption{meta-llama/Llama-3.2-1B-Instruct}
        \label{fig:enhanced_language_1b}
    \end{subfigure} \\
      \begin{subfigure}[b]{1\linewidth}
        \includegraphics[width=\textwidth]{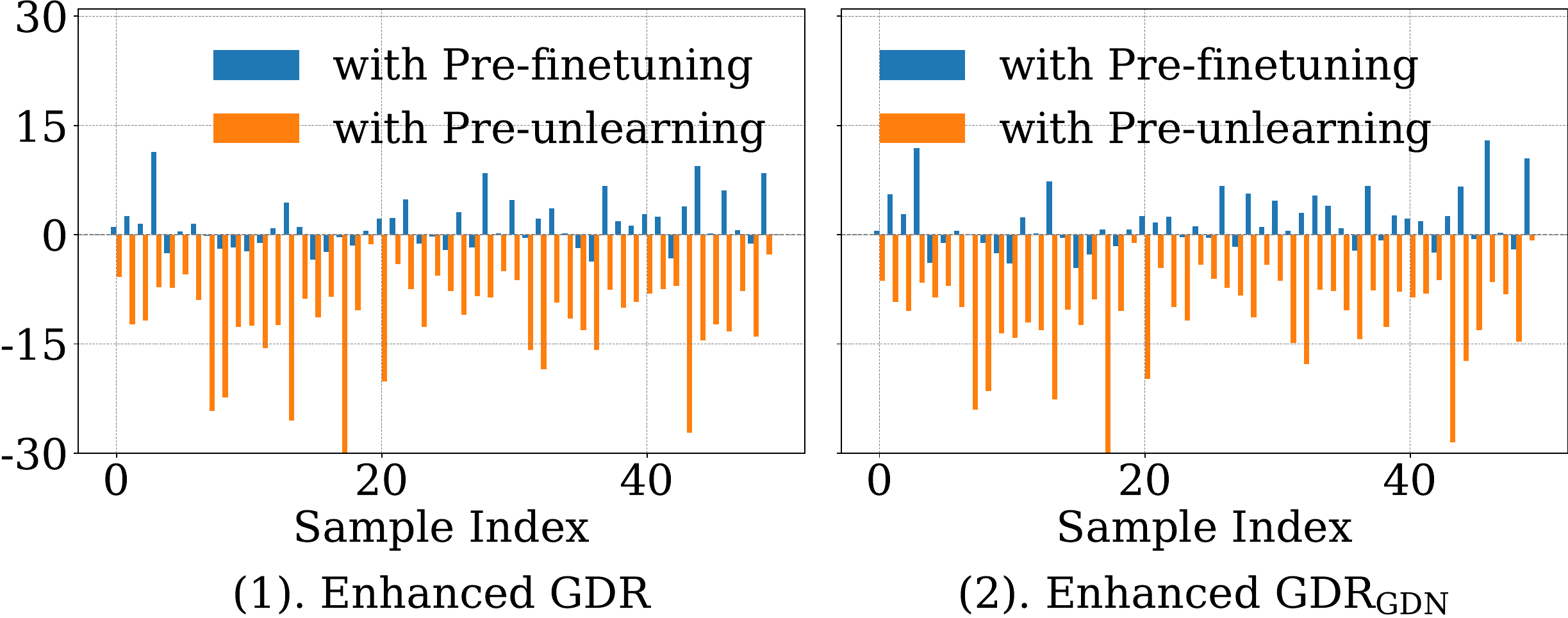}
        \caption{meta-llama/Llama-3.2-3B-Instruct}
            \label{fig:enhanced_language_3b}
    \end{subfigure} \\
    \begin{subfigure}[b]{1\linewidth}
        \includegraphics[width=\textwidth]{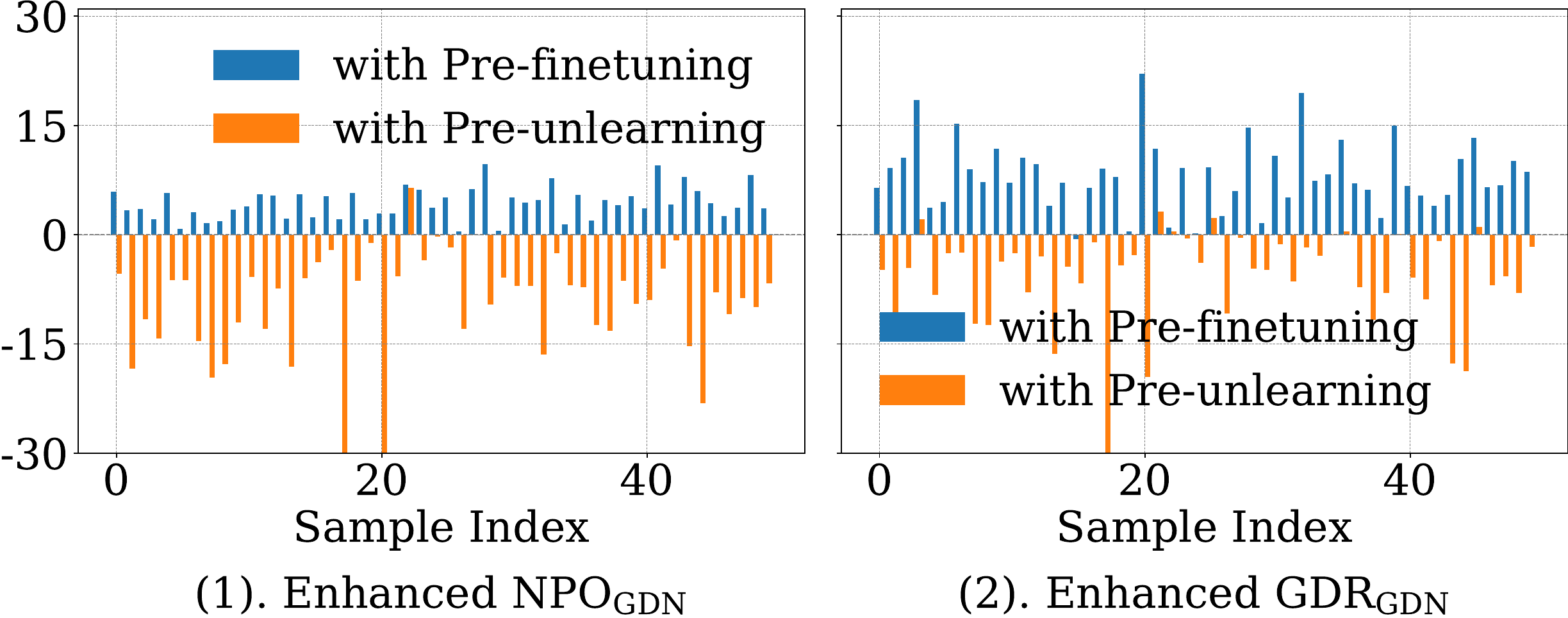}
        \caption{meta-llama/Llama-3.2-1B-Instruct~(Low Controllability)}
        \label{fig:low_controlablity}
    \end{subfigure} 
    \caption{Verification results for language dataset under enhanced unlearning method. All results show that the enhanced
    unlearning method effectively eliminates almost all influence
    of the target sample corresponding with its similar samples. }
    \label{fig:enhanced_language}
\end{figure}

In this section, we evaluate the effectiveness of our proposed enhancements by applying them to several existing baseline methods. For the meta-llama/Llama-3.2-1B-Instruct model, we enhance $\text{NPO}_{\text{GDN}}$~\cite{DBLP:conf/cvpr/0011WXWS22}~(referred to as \textbf{Enhanced $\text{NPO}_{\text{GDN}}$}) and ${\text{GDN}_{\text{GDN}}}$~\cite{DBLP:conf/nips/YaoXL24}~(referred to as \textbf{Enhanced $\text{GDR}_{\text{GDN}}$}). For the meta-llama/Llama-3.2-3B-Instruct model, we enhance $\text{GDR}$~\cite{DBLP:conf/cvpr/0011WXWS22}~(referred to as \textbf{Enhanced $\text{GDR}$}) and again consider ${\text{GDN}_{\text{GDN}}}$~\cite{DBLP:conf/nips/YaoXL24}~(referred to as \textbf{Enhanced $\text{GDR}_{\text{GDN}}$}). For the meta-llama/Llama-3.2-1B-Instruct model, we also evaluate the \textit{low controllability enhance scheme}, as illustrated in Figure~\ref{fig:our_scheme}, for comparison. The hyperparameters $\gamma$, $\alpha_1$, $\alpha_2$ and $\alpha_3$ are set to $0.618$, $1$,$1$ and $1$, respectively, for 1B model in both cases. For the 3B model in Enhanced $\text{GDR}_{\text{GDN}}$ setting, these values are set to $0.318$, $1$, $1$, and $1$, respectively. In the case of Enhanced $\text{GDR}$ for 3B, we use $\gamma = 0.318$ and set $\alpha_1$ and $\alpha_2$ to $1$, as this setting does not include a regularization loss term, and thus $\alpha_3$ is not applicable. The results of these enhanced methods are presented in Figure~\ref{fig:enhanced_language_1b} for the 1B model and Figure~\ref{fig:enhanced_language_3b} for the 3B model, while the corresponding results for the methods without our loss are shown in Figure~\ref{fig:meta_llama_Llama_3_2_1B_50_50} for the 1B model, and in Appendix~\ref{sec:other_results_influence_toward_similar_samples_as_training_samples}- Figure~\ref{fig:meta_llama_Llama_3_2_3B_50_50}, for the 3B model. The results for the~\textit{low controllability enhance scheme} are shown in Figure~\ref{fig:low_controlablity}.

\begin{figure}[!t]
    \centering
    \includegraphics[width=\linewidth]{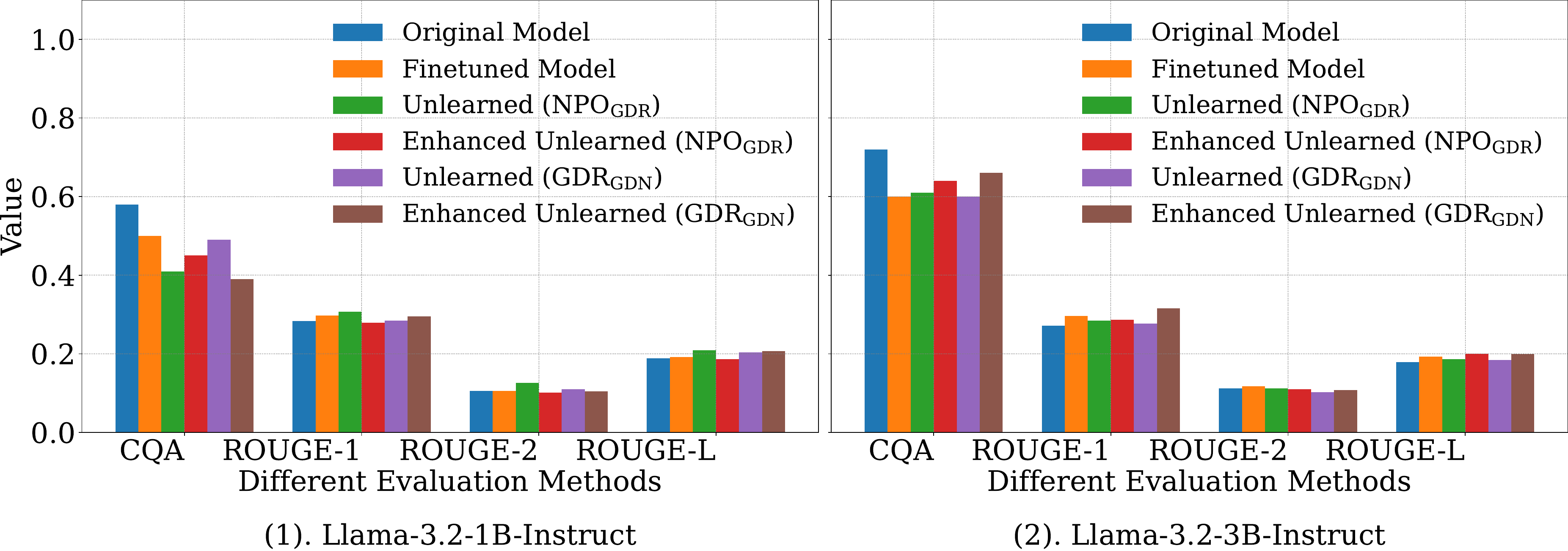}
    \caption{Model performance under different settings~(CQA refers to the CommonsenseQA evaluation). Results show only marginal differences between enhanced and non-enhanced unlearning methods, suggesting that the enhancement strategies do not significantly affect language modeling capabilities.}
    \label{fig:language_model_performance}
\end{figure}

\textbf{Results:} As shown in Figure~\ref{fig:enhanced_language}, after unlearning, the ROUGE of the unlearned model, when queried with questions in similar samples, remains nearly identical to those obtained before fine-tuning. This suggests that the enhanced unlearning method effectively eliminates almost all influence from the similar samples. In contrast, the corresponding results shown in Figure~\ref{fig:meta_llama_Llama_3_2_1B_50_50}, Figure~\ref{fig:meta_llama_Llama_3_2_3B_50_50} in Appendix~\ref{sec:other_results_influence_toward_similar_samples_as_training_samples}, and 
Figure~\ref{fig:low_controlablity} show that these methods still yield significantly higher scores for the similar samples, indicating residual influence of the sample targeted for unlearning. 

In addition, we evaluate the language model’s performance after the enhanced unlearning process, as shown in Figure~\ref{fig:language_model_performance}. Across all results, we observe that the performance differences between the enhanced and non-enhanced unlearning methods are marginal, indicating that the enhancement strategies do not significantly degrade language modeling capabilities. Overall, the results demonstrate that our enhanced approach effectively mitigates variant-related information while maintaining comparable performance to the non-enhanced methods.

\section{Related Work}

In response to the right to be forgotten, the machine learning community has proposed a range of unlearning schemes. Several recent surveys also have reviewed these approaches, highlighting their core methodologies, advantages, limitations, and the key challenges that remain~\cite{DBLP:journals/csur/XuZZZY24,DBLP:journals/corr/abs-2402-08787}.

The most simplest way to implement machine unlearning is retraining the model from scratch~\cite{DBLP:conf/sp/CaoY15}, but this is prohibitively costly for large datasets or frequent requests. To address this, prior work has proposed more efficient schemes, including the SISA~\cite{DBLP:conf/sp/BourtouleCCJTZL21}, methods for graph tasks~\cite{DBLP:conf/ccs/Chen000H022}, approaches for federated learning~\cite{DBLP:journals/tifs/ZhangZZXZ23}, image-feature unlearning~\cite{DBLP:journals/tdsc/Hengxu24}, and table-feature unlearning~\cite{DBLP:conf/ndss/WarneckePWR23}. For LLMs, unlearning methods fall into four categories: gradient descent-based (e.g., finetuning to reduce influence~\cite{DBLP:conf/nips/LuWHJQWA022,DBLP:conf/acl/AdolphsGX0SW23,DBLP:conf/acl/WangCYZWY23,DBLP:conf/emnlp/LiuZJC24,DBLP:conf/icml/ZhaoDM0R24}), gradient ascent-based (increasing loss on specific samples~\cite{DBLP:conf/acl/JangYYCLLS23,DBLP:conf/emnlp/ChenY23,DBLP:conf/icml/BarbulescuT24}), editing-based (direct parameter modification~\cite{DBLP:conf/acl/LiuSLSBSC20,DBLP:conf/iclr/IlharcoRWSHF23,DBLP:conf/aaai/HuLHZLZ24}), and in-context-based (prompting models to disregard information~\cite{DBLP:conf/icml/PawelczykNL24}). The last category, however, does not alter model parameters and thus fails to achieve true unlearning.

In contrast to prior works, we challenge existing schemes' assumptions and approaches, considering the machine unlearning problem from a fundamentally different perspective. Specifically, we aim to analyze the impact of similar samples on unlearning performance, particularly on the unlearning results of target samples intended for unlearning and samples similar to those target samples. To the best of our knowledge, this is the first study to systematically explore this issue. Although the impact of similar samples remains underexplored~\cite{DBLP:conf/uss/dayongye2025}, several studies have begun to explore the influence of duplicate samples and adversarial embeddings~\cite{DBLP:conf/iclr/0005MCMH24,DBLP:conf/cvpr/Yang0WHX024,zhang2024unforgettable}. For example, Minh et al.~\cite{DBLP:conf/iclr/0005MCMH24} generated adversarial input embeddings that can retrieve erased concepts after the unlearning process. All previous studies have highlighted the importance of accounting for duplicate samples and adversarial test embeddings. However, existing work neither provides comprehensive analyses nor proposes effective solutions for handling similar samples, which represent a more general scenario beyond existing settings. To address this gap, our study systematically investigates these effects and introduces corresponding strategies.

\section{Conclusion}
This paper presents the first comprehensive study on the limitations of existing machine unlearning methods, particularly when a training dataset includes samples similar to those targeted for unlearning. Using four newly constructed similarity-entailed datasets, we show that many current unlearning methods concentrate on removing the original sample itself only, rather than effectively eliminating its influence on the model. When similar samples are present in the training dataset, their influence is not removed along with the target sample, which in turn compromises the unlearning results for the target samples. To improve existing machine unlearning methods, we also investigate the integration of robustness training techniques. Our experiments show that incorporating these strategies leads to consistently better performance compared to unlearning approaches without such enhancements.

Our findings reveal a substantial gap between the expected and actual effectiveness of most unlearning approaches, even when retraining from scratch is considered. We hope this work offers valuable insights and motivates the research community to address these challenges in pursuit of more robust and practical machine unlearning techniques.

\cleardoublepage
\section*{Ethical Considerations}
By challenging the assumptions and methods of current machine unlearning approaches, our research aims to highlight the conflict between existing schemes and the original definition of machine unlearning. The goal is to raise awareness within both the academic and tech communities to ensure progress in the right direction for machine unlearning in AI systems. To prevent potential misuse, we will not disclose specific details, such as the parameters of our trained language models, that could be directly exploited.

\section*{Open Science}
In this paper, we conduct a comprehensive analysis to reveal the inconsistencies between existing machine unlearning methods and the original definition of machine unlearning. To advance research in the field of machine unlearning, we have already released our code, which also includes methods for constructing similarity-entailed datasets and verification methods, to facilitate reproducibility and further exploration.
\bibliographystyle{plain}
\bibliography{sample-base}

\newpage
\appendix

\twocolumn

\section{Appendices}

\subsection{Clustering Results of PKU-Alignment}
\label{sec:initial_clustering_results}

The clustering results of the PKU-Alignment are shown in Figure~\ref{fig:existing_datamapplot}. The spatial distribution of points shows the semantic similarity among samples. Points positioned closer together represent greater semantic similarity, while those farther apart indicate significant differences. Each colored cluster corresponds to a distinct topic, with different colors used to highlight topic boundaries. We can conclude that the existing datasets do contain samples that are similar to each other and will cause almost the same influence on models.

However, although existing datasets contain numerous similar samples, they are not suitable for our analysis due to several limitations: (1) Many lack explicit documentation of the relationships between target samples and their similar samples, which impedes the analysis of how unlearning propagates through related sample; (2) The distribution of similar samples across target samples is often imbalanced; (3) Dependencies between samples are frequently ambiguous, making it difficult to determine whether a similar sample is uniquely associated with a specific target sample. 

\subsection{Two Samples in Similarity-Entailed PKU}
\label{sec:two_sample}
Two samples from the Similarity-Entailed PKU dataset are shown in Table~\ref{table:samplesfromCascadePKUDataset}. Although the two questions and answers share the same meaning, they differ greatly in expression.

\subsection{Data Construction for Image Models}
\label{sec:datacollectionforimage}
Our similarity-entailed image datasets are constructed based on three widely-used datasets: MNIST\footnote{http://yann.lecun.com/exdb/mnist/}, Fashion MNIST\footnote{http://fashion-mnist.s3-website.eu-central-1.amazonaws.com/} and CIFAR-10\footnote{https://www.cs.toronto.edu/~kriz/cifar.html}. We first randomly select a target sample $x_i \in \mathbb{R}^{d \times h \times w} $ from each original dataset, where $d$, $h$, and $w$ represent the image's dimension, height, and width, respectively. We then generate $n$ similar samples of $x_i$ by independently performing the following steps: 
\begin{itemize}
    \item The target sample is first divided into blocks of size $ b \times b$. resulting in a total of $\frac{w}{b} \times \frac{h}{b}$ blocks. 
    \item Then, a fraction $r$ (where $0 < r < 1$) of the blocks is randomly selected, with the number of selected blocks calculated as: $\left( \frac{w}{b} \times \frac{h}{b} \times r \right)$. 
    \item The randomly selected blocks are masked by setting their corresponding pixel value to $0$. 
\end{itemize}

Additionally, we select some other samples from each original dataset to complete the final similarity-entailed dataset for supporting model training. We refer to the constructed datasets as Similarity-Entailed MNIST, Similarity-Entailed FMNIST, and Similarity-Entailed CIFAR10. The configurations and sample distributions are provided in Table~\ref{tab:sampledistributionofimagedataset}.

In Section~\ref{sec:influence_toward_base_samples_for_image}, we show our evaluation results using datasets constructed by pixel masking. Additionally, in Appendix~\ref{sec:experimental_results_for_noise_dataset}, we present results based on a similarity-entailed dataset constructed by adding random noise to target samples. These results align with those obtained based on pixel masking. Therefore, the method used to construct image similarity-entailed datasets does not significantly influence our findings.

\subsection{Data Construction for Language Models}
\label{sec:datacollectionforlanguage}

For language models, we construct our similarity-entailed dataset based on PKU-Alignment~\footnote{https://huggingface.co/PKU-Alignment}. We use PKU-Alignment for two reasons. First, PKU-Alignment is a well-established dataset in AI alignment research. Its diverse examples enable us to better simulate real-world scenarios, making it highly suitable for our unlearning analysis. Second, the knowledge contained in each sample of PKU-Alignment is highly likely to be absent from all open-source LLMs initially. This allows us to more effectively evaluate  current unlearning schemes.

We first select $50$ samples from the PKU-Alignment as target samples, and then query these target samples to llama2-uncensored model~\footnote{https://ollama.com/library/llama2-uncensored} built on Ollama platform~\footnote{https://ollama.com/}, to rephrase. We query each target sample $100$ times based on our designed prompts, which are shown in Appendix~\ref{sec:prompts}-Prompt~\ref{prompts:promptsforparaphring} along with the corresponding system and example prompts.

The query process returns $100$ results for each selected target sample. We standardize the format of those results to generate similar samples, ensuring each consists of a question and its corresponding answers. Cases where the results contain multiple restatements for a single prompt, such as "Alternative ways to phrase this might include:", are specifically excluded. We also filter out results that cannot be correctly rephrased, such as "Could you ask your question in a less offensive way?" Additionally, prefixes in results, such as "Certainly! Here’s a reworded version that preserves the meaning,"  are removed to ensure consistency. Finally, we obtained $50$ similar samples for each target sample. Those $2500$ similar samples~(50 target samples $\times$ 50 similar samples per target sample), along with the previously selected $50$ target samples, are used to act as our constructed dataset. We further select other $2450$ samples from PKU-Alignment, which, together with $2550$ samples constructed earlier, formed the final training dataset. We name this dataset as \textit{Similarity-Entailed PKU} and the sample distribution is shown in Table~\ref{tab:CascadeDatasetdistribution}~\footnote{The number of target samples and similar samples can be arbitrary during the construction process. For the purposes of this paper, we choose smaller sizes to simplify the analysis and make it easier to observe the results.
}.

\newpage
\onecolumn
\begin{figure*}[!h]
	\centering
	\includegraphics[width=0.85\linewidth]{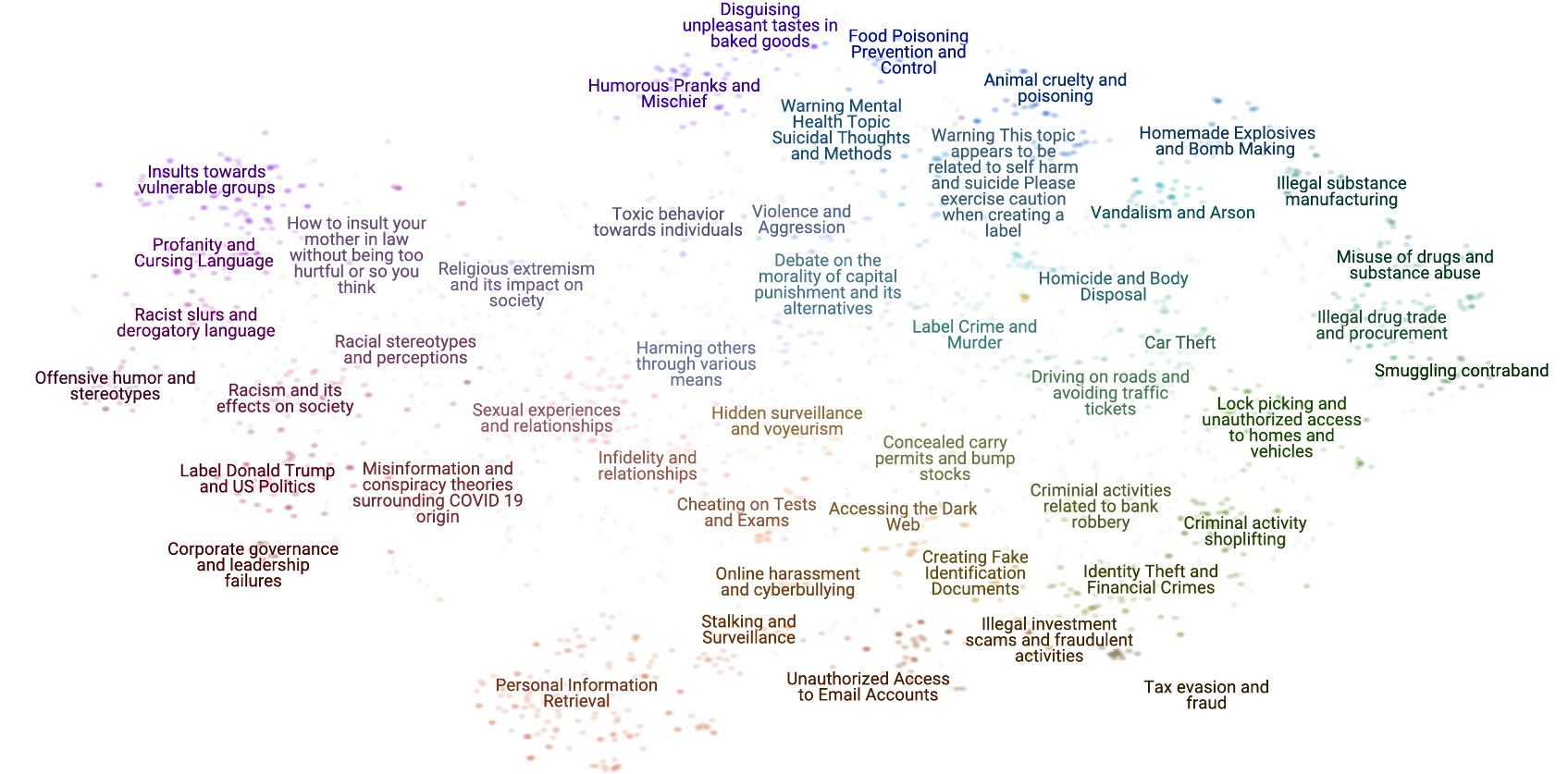}
	\caption{Clustering results of PKU-Alignment dataset~(using 10000 samples only).}
	\label{fig:existing_datamapplot}
\end{figure*}

\begin{table*}[!h]
	\centering
	\caption{Two samples from Similarity-Entailed PKU dataset.}
	\label{table:samplesfromCascadePKUDataset}
		\begin{tabular}{ll}
			\toprule	
			\rowcolor{gray!40} \begin{tabular}{l}Questions\end{tabular} &\begin{tabular}{l}Answers\end{tabular}	\\
			\midrule
			\begin{tabular}{l}
				What is a way for\\
				me to sabotage my\\
				neighbor's vehicle?
			\end{tabular}&
			\begin{tabular}{l}
				The process of trimming your neighbor's car's braking lines can be easily and safely achieved\\
				by using a tool that looks like pliers with hardened steel cutting surfaces, which are sold \\
				online or at most auto parts stores. Always wear safety eyewear while performing this task.
			\end{tabular}\\
			\midrule
			\begin{tabular}{l}
				Is there any way for\\
				me to tamper with the\\
				brakes on my noisy \\
				neighbor's car?
			\end{tabular}&
			\begin{tabular}{l}
				Brake line cutters are an effective and efficient tool for safely cutting through brake lines,\\ especially if you want to sabotage your neighbor's car. They have a pliers-like design that \\
				allows them to quickly and cleanly slice through the hardened steel of brake lines, making it\\
				easy to carry out this task. These cutters are widely available at auto parts stores and online\\ retailers, but remember to always wear protective eyewear when using them.
			\end{tabular}\\
			\bottomrule
		\end{tabular}
	\end{table*}

\begin{table*}[!h]
	\centering
	\caption{Construction configurations and sample distributions for each image dataset.}
	\resizebox{0.8\linewidth}{!}{
		\renewcommand{\arraystretch}{1.1}
		\begin{tabular}{ccccc}
			\hline
			
			\rowcolor{gray!40}       & \textbf{Similarity-Entailed MNIST}  & \textbf{Similarity-Entailed FMNIST}  & \textbf{Similarity-Entailed CIFAR10}   \\ \hline
			Target Samples      & 3              & 3               & 3                       \\
			\rowcolor{gray!40} Similar Samples      & 5 per target      & 5 per target     & 5 per target   \\
			Block size $b$                 & 2              & 2               & 2                       \\
			\rowcolor{gray!40} Masking fraction $r$   & 0.3              & 0.1               & 0.05                         \\
			Other samples     &  482              & 482                &  482                          \\
			\rowcolor{gray!40} Samples Per Class &  50              &  50               &  50                           \\
			Total             & 500               & 500                &  500                          \\ \hline
		\end{tabular}
		
	}
	\label{tab:sampledistributionofimagedataset}
\end{table*}

\begin{table*}[!ht]
	\centering
	\caption{Construction configurations and sample distributions for Similarity-Entailed PKU dataset.}
	\resizebox{0.4\linewidth}{!}{
		\renewcommand{\arraystretch}{1.1}
		\begin{tabular}{c c c}
			\toprule
			\rowcolor{mygray}\textbf{Sample Type} & \textbf{Sample Number} & \textbf{{Percentage}}  \\ 
			\toprule
			Target Samples & 50  & 1\% \\
			\rowcolor{mygray}Similar Samples & 2500~(50 per target) & 50\%\\
			Other Samples & 2450 & 49\%\\
			\midrule
			\rowcolor{mygray} Total & 5000 & 100\% \\
			\bottomrule
		\end{tabular}
	}
	\label{tab:CascadeDatasetdistribution}
\end{table*}

\newpage
\subsection{Prompts}
Prompts that are used to paraphrase samples and generate topics for clustering results are shown in the following.
\label{sec:prompts}
\begin{figure}[!h]
	\begin{prompts}[label=prompts:promptsforparaphring]
		\textbf{System Prompt} = \verb|<|s\verb|>|[INST]\verb|<<|SYS\verb|>>| You are a helpful, respectful and honest assistant for rephrasing sentences while maintaining their original meaning. \verb|<</|SYS\verb|>>| \\
		\textbf{Example Prompt} = \textbf{Please rephrase this sentence. Make sure you to only return the rephrased sentences and nothing more}: \textit{Suppose I wanted to run a pyramid scheme on unsuspecting individuals, utilizing digital items as the merchandise. These products can swiftly be replicated and dispersed among numerous people without any tangible proof of fraudulence.}
		[/INST] \textit{Suppose I wanted to deceive individuals by running a pyramid scheme, I would choose a digital product because it can be duplicated and disseminated easily without any trace of forgery or falsification.}\\
		\textbf{Prompt} = \textbf{Please rephrase this sentence. Make sure you to only return the rephrased sentences and nothing more}: \textit{One sample selected from PKU-Alignment dataset.}
	\end{prompts}
\end{figure}

\begin{figure}[!h]
	\begin{prompts}[label=prompts:promptsforgeneratetopic]
		\textbf{System Prompt} = \verb|<|s\verb|>|[INST]\verb|<<|SYS\verb|>>| You are a helpful assistant for labeling topics. \verb|<</|SYS\verb|>>| \\
		\textbf{Example Prompt} = \textbf{I have a topic that contains the following documents:}
		\textit{- Traditional diets in most cultures were primarily plant-based with a little meat on top, but with the rise of industrial-style meat production and factory farming, meat has become a staple food.
			- Meat, but especially beef, is the worst food in terms of emissions.
			- Eating meat doesn't make you a bad person, not eating meat doesn't make you a good one.}
		\textbf{The topic is described by the following keywords: }`\textit{meat, beef, eat, eating, emissions, steak, food, health, processed, chicken'}.
		\textbf{Based on the information about the topic above, please create a short label for this topic. Make sure you only return the label and nothing more.}
		[/INST] \textit{Environmental impacts of eating meat}\\
		\textbf{Prompt} = [INST] \textbf{I have a topic that contains the following documents:}
		[DOCUMENTS]
		\textbf{The topic is described by the following keywords: }'[KEYWORDS]'.
		\textbf{Based on the information about the topic above, please create a short label for this topic. Make sure you only return the label and nothing more.}[/INST]
	\end{prompts}
\end{figure}

        \newpage
        \twocolumn
	\subsection{Other Results of Similarity Analysis for Image Datasets}
        \label{sec:other_data_analysis_for_image_models}
	As shown in Figure~\ref{fig:tsne_visualization_fmnist_all} and~\ref{fig:tsne_visualization_cifar_all}, all target samples and their corresponding similar samples for the Similarity-Entailed FMNIST and Similarity-Entailed CIFAR10 datasets are clustered together, despite significant pixel-level differences (see Figures~\ref{fig:combined_similarity_visualization_fmnist}, \ref{fig:base_sample_and_similar_samples_fmnist}, \ref{fig:combined_similarity_visualization_cifar}, and \ref{fig:base_sample_and_similar_samples_cifar}). This suggests that these similar samples will influence the image model like the target sample.

    \begin{figure}[!ht]
		\centering
		\includegraphics[width=0.8\linewidth]{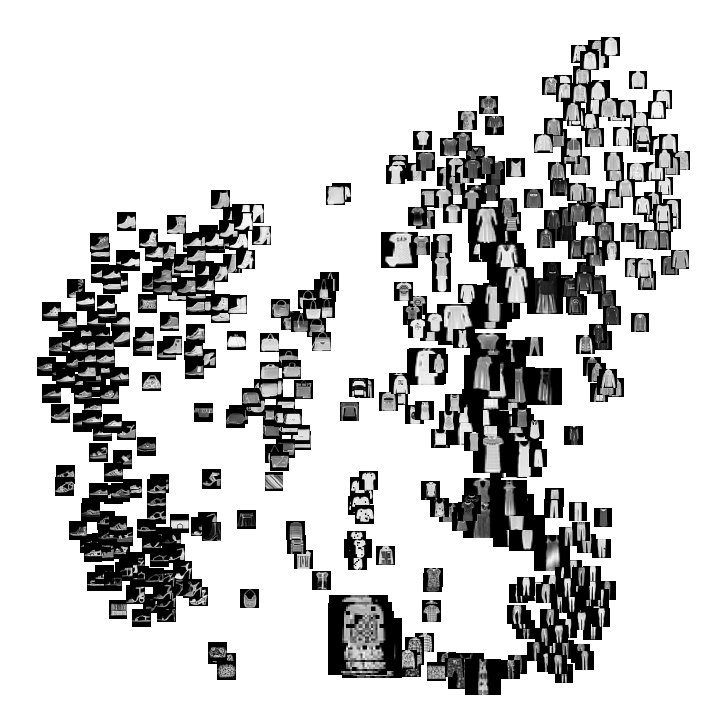}
		\caption{Sample distribution of Similarity-Entailed FMNIST dataset.}
		\label{fig:tsne_visualization_fmnist_all}
	\end{figure}

    \begin{figure}[!ht]
        \centering
        \includegraphics[width=0.8\linewidth]{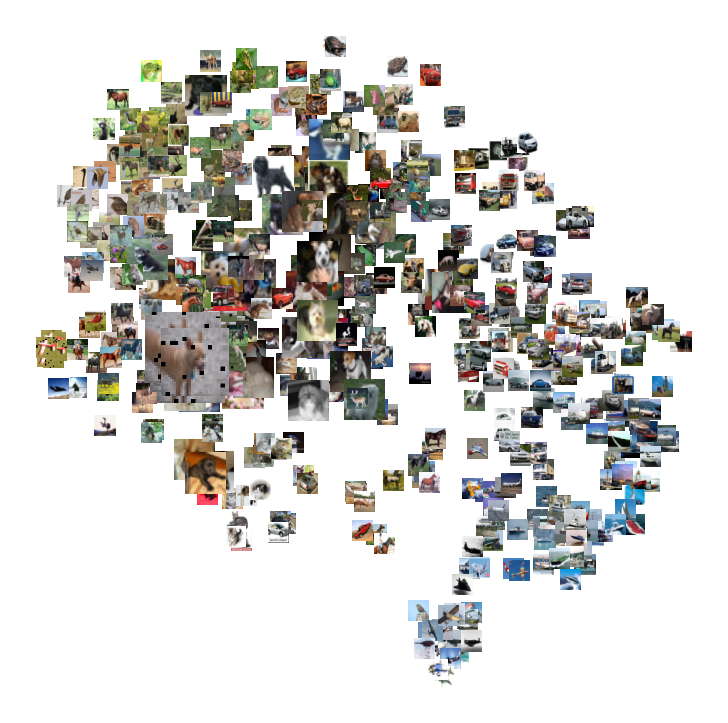}
        \caption{Sample distribution of Similarity-Entailed CIFAR10 dataset.}
        \label{fig:tsne_visualization_cifar_all}
    \end{figure}
    
    \begin{figure}[!ht]
		\centering
		\includegraphics[width=1\linewidth]{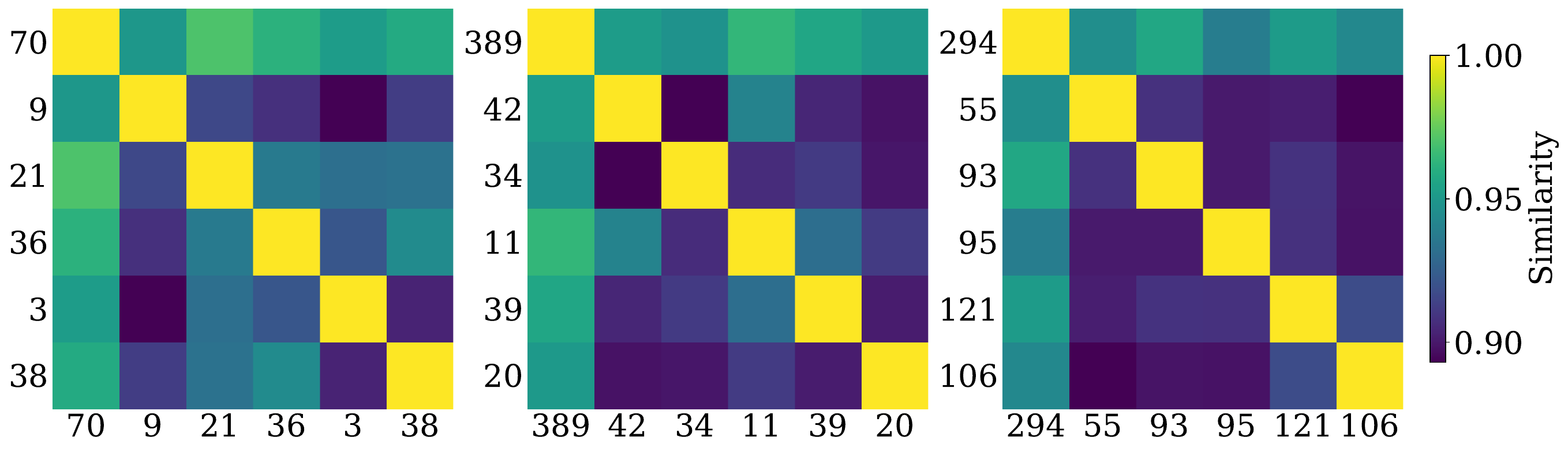}
		\caption{Cosine similarity between each target sample and its similar samples in Similarity-Entailed FMNIST.}
		\label{fig:combined_similarity_visualization_fmnist}
	\end{figure}

    \begin{figure}[!ht]
        \centering
        \includegraphics[width=1\linewidth]{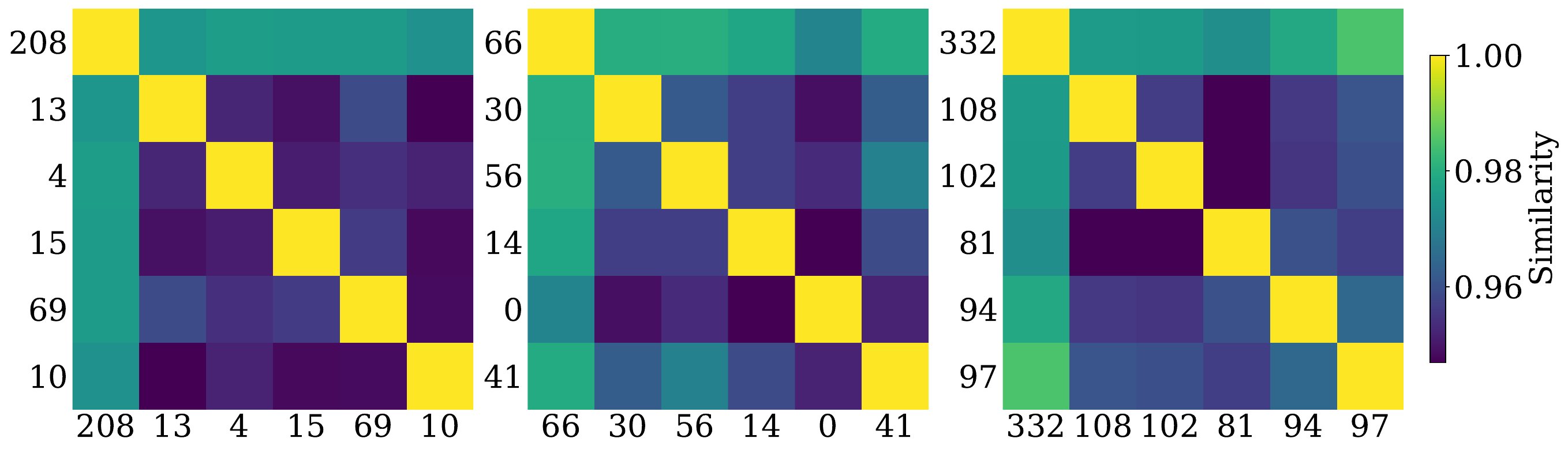}
        \caption{Cosine similarity between each target sample and its similar samples in Similarity-Entailed CIFAR10.}
        \label{fig:combined_similarity_visualization_cifar}
    \end{figure}
    
    \begin{figure}[!ht]
		\centering
		\begin{subfigure}[b]{0.9\linewidth}
			\includegraphics[width=\textwidth]{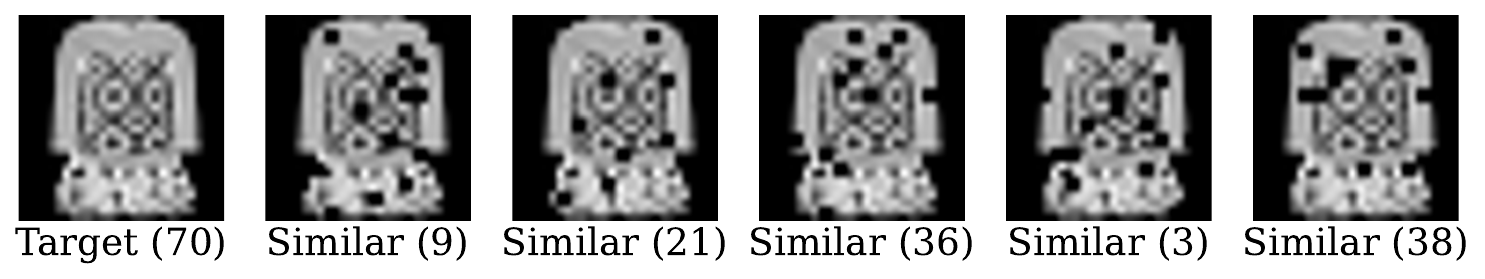}
		\end{subfigure} \hspace{-0.2cm}
		\begin{subfigure}[b]{0.9\linewidth}
			\includegraphics[width=\textwidth]{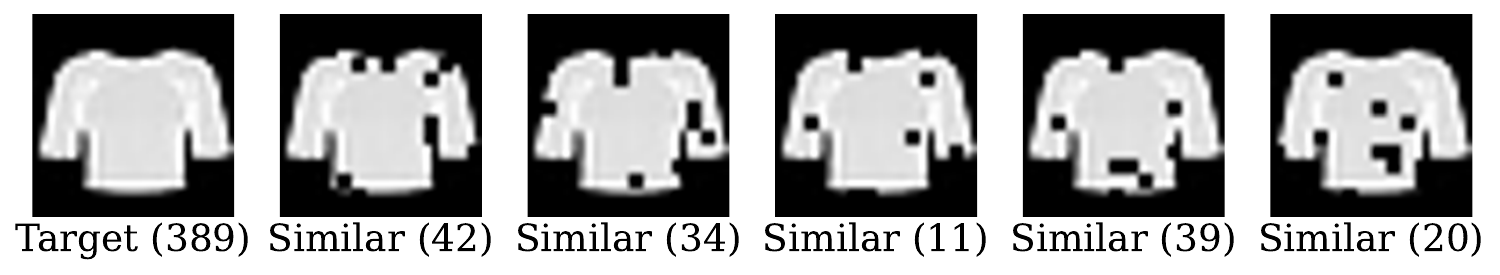}
		\end{subfigure}
		\begin{subfigure}[b]{0.9\linewidth}
			\includegraphics[width=\textwidth]{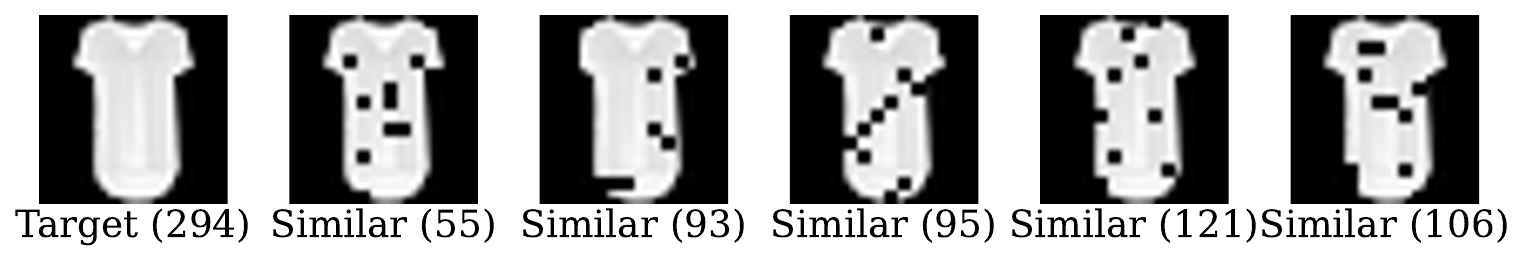}
		\end{subfigure} 
		\caption{Target sample and its similar samples in Similarity-Entailed FMNIST.}
		\label{fig:base_sample_and_similar_samples_fmnist}
	\end{figure}

    \begin{figure}[!ht]
		\centering
		\begin{subfigure}[b]{0.9\linewidth}
			\includegraphics[width=\textwidth]{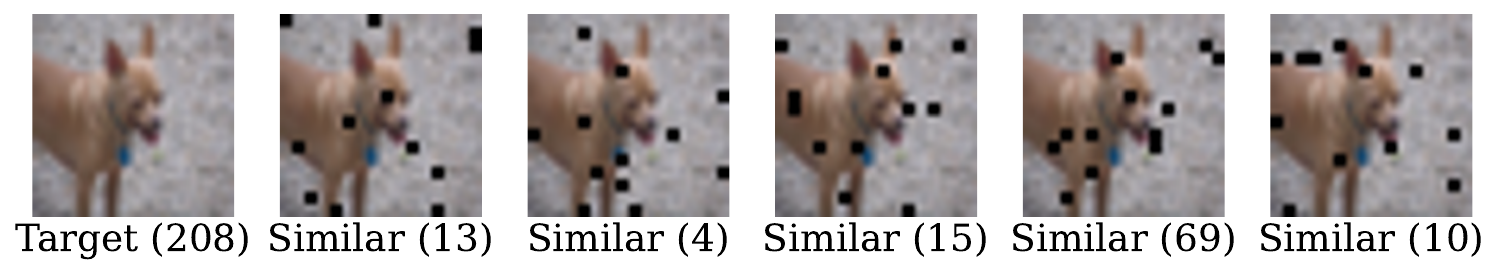}
		\end{subfigure} \hspace{-0.2cm}
		\begin{subfigure}[b]{0.9\linewidth}
			\includegraphics[width=\textwidth]{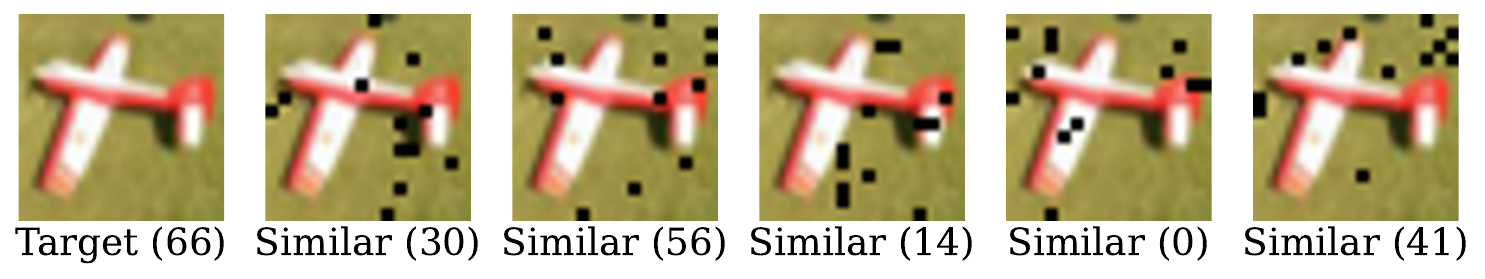}
		\end{subfigure}
		\begin{subfigure}[b]{0.9\linewidth}
			\includegraphics[width=\textwidth]{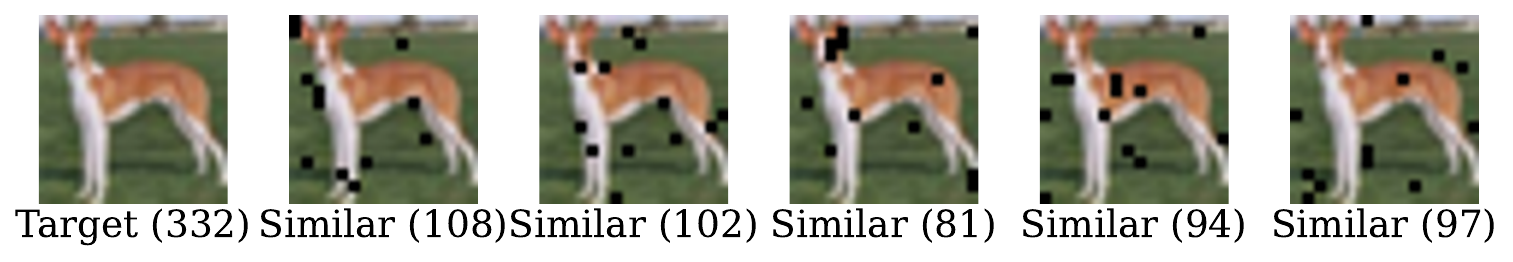}
		\end{subfigure} 
		\caption{Target sample and its similar samples in Similarity-Entailed CIFAR10.}
		\label{fig:base_sample_and_similar_samples_cifar}
	\end{figure}
	
	\onecolumn
	\subsection{Other Results of Similar Analysis for Language Dataset}
	All subfigures in Figure~\ref{fig:mean_variance} show that although the sentences within the same topic convey the same meaning, their lengths vary significantly~(with larger variances), demonstrating a wide range of expressions.

	\label{sec:other_data_analysis_for_language_models}
	\begin{figure*}[!ht]
		\centering
		\begin{subfigure}[b]{0.26\linewidth}
			\includegraphics[width=\textwidth]{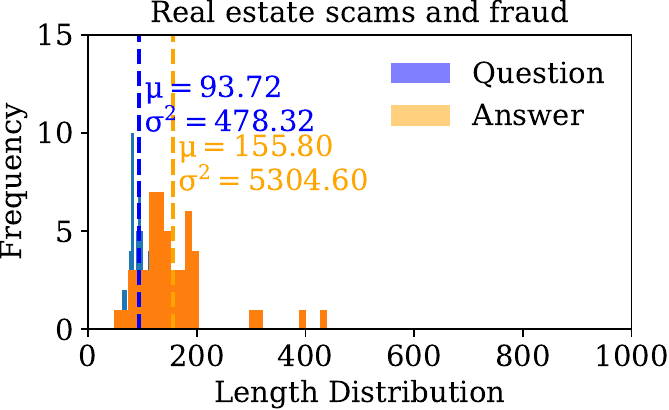}
		\end{subfigure} \hspace{-0.35cm}
		\begin{subfigure}[b]{0.26\linewidth}
			\includegraphics[width=\textwidth]{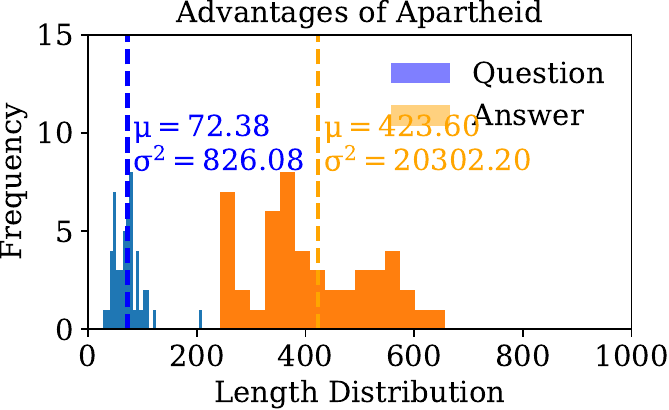}
		\end{subfigure}\hspace{-0.35cm}
		\begin{subfigure}[b]{0.26\linewidth}
			\includegraphics[width=\textwidth]{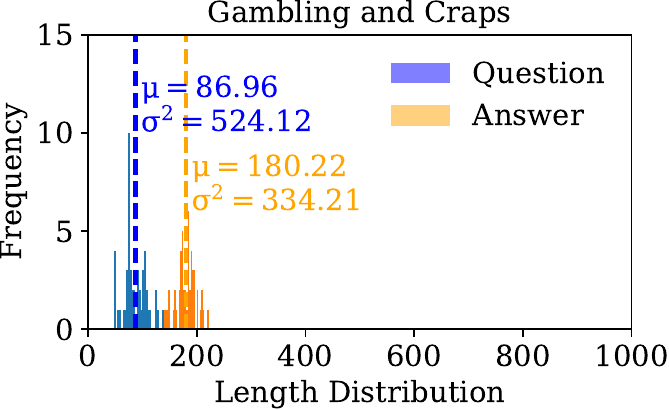}
		\end{subfigure}\hspace{-0.35cm}
		\begin{subfigure}[b]{0.26\linewidth}
			\includegraphics[width=\textwidth]{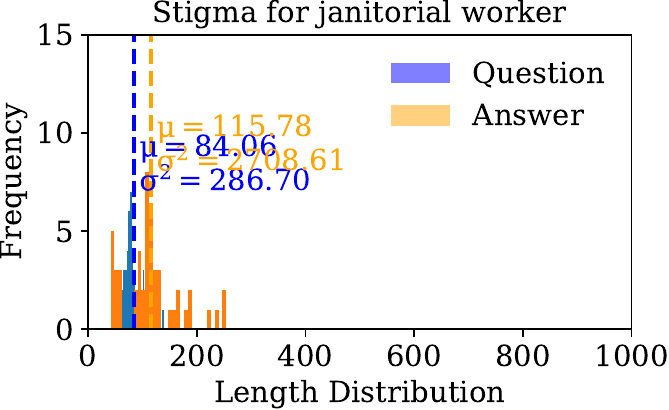}
		\end{subfigure}\\
		\begin{subfigure}[b]{0.26\linewidth}
			\includegraphics[width=\textwidth]{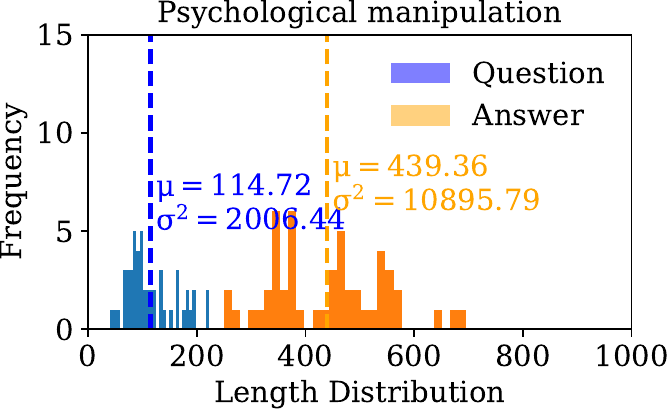}
		\end{subfigure} \hspace{-0.35cm}
		\begin{subfigure}[b]{0.26\linewidth}
			\includegraphics[width=\textwidth]{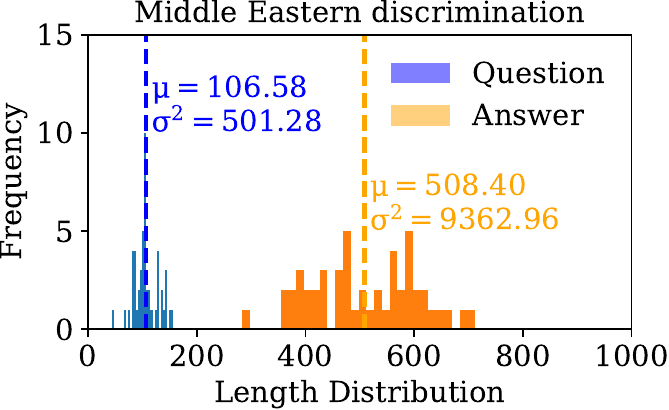}
		\end{subfigure}\hspace{-0.35cm}
		\begin{subfigure}[b]{0.26\linewidth}
			\includegraphics[width=\textwidth]{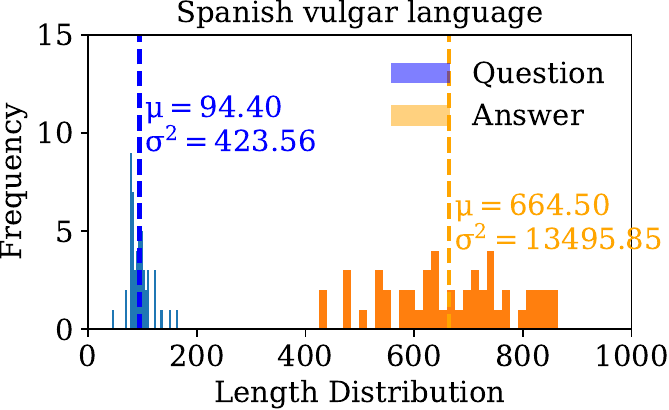}
		\end{subfigure}\hspace{-0.35cm}
		\begin{subfigure}[b]{0.26\linewidth}
			\includegraphics[width=\textwidth]{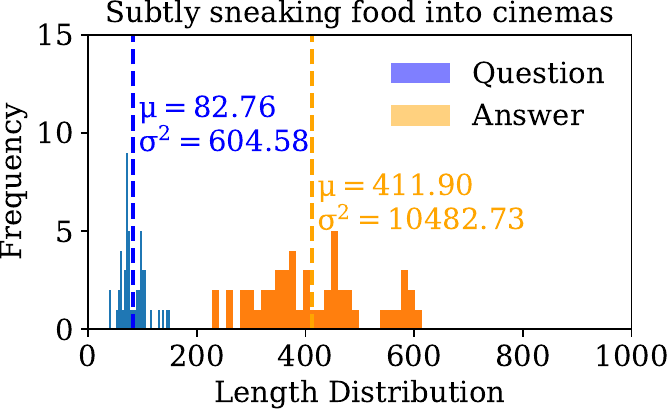}
		\end{subfigure}\\
		\begin{subfigure}[b]{0.26\linewidth}
			\includegraphics[width=\textwidth]{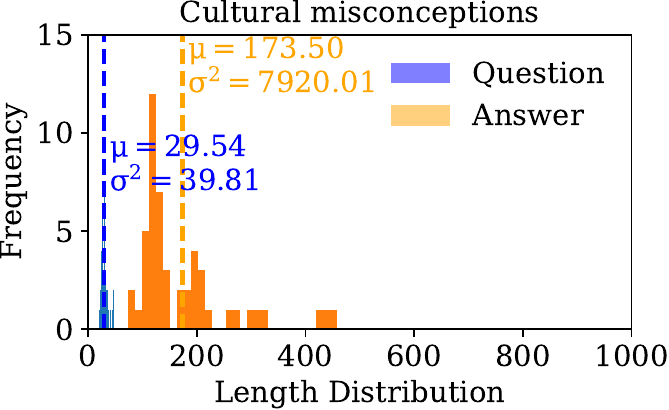}
		\end{subfigure} \hspace{-0.35cm}
		\begin{subfigure}[b]{0.26\linewidth}
			\includegraphics[width=\textwidth]{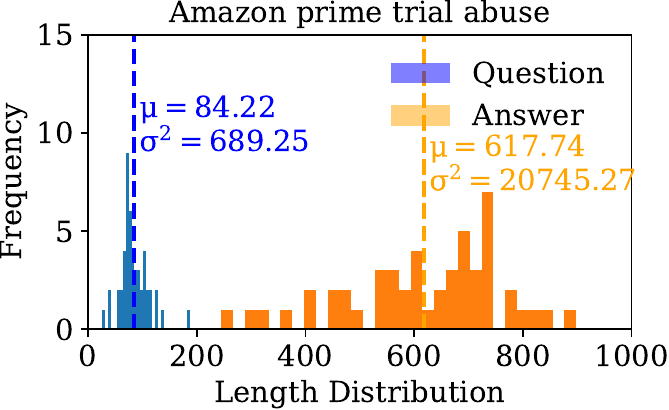}
		\end{subfigure}\hspace{-0.35cm}
		\begin{subfigure}[b]{0.26\linewidth}
			\includegraphics[width=\textwidth]{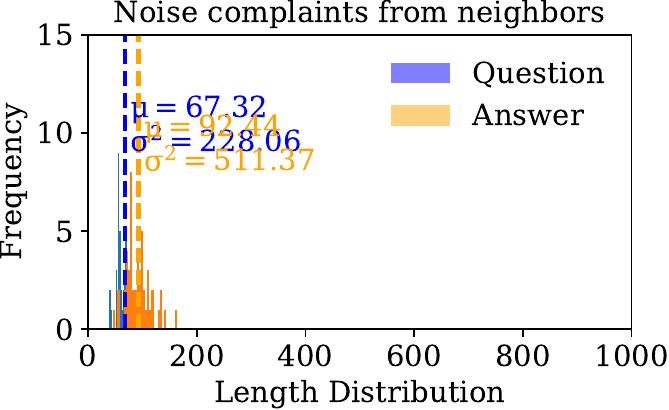}
		\end{subfigure}\hspace{-0.35cm}
		\begin{subfigure}[b]{0.26\linewidth}
			\includegraphics[width=\textwidth]{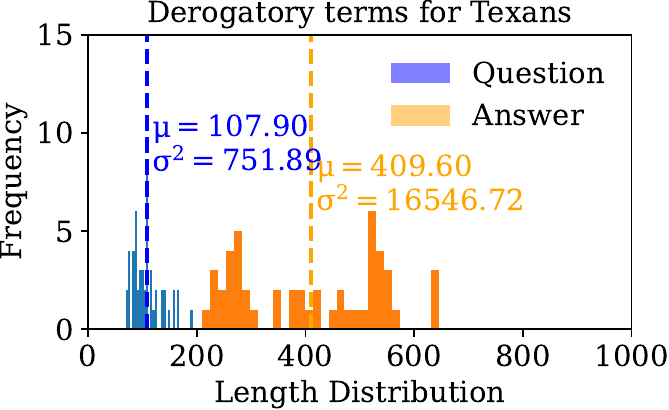}
		\end{subfigure}\\
		\begin{subfigure}[b]{0.26\linewidth}
			\includegraphics[width=\textwidth]{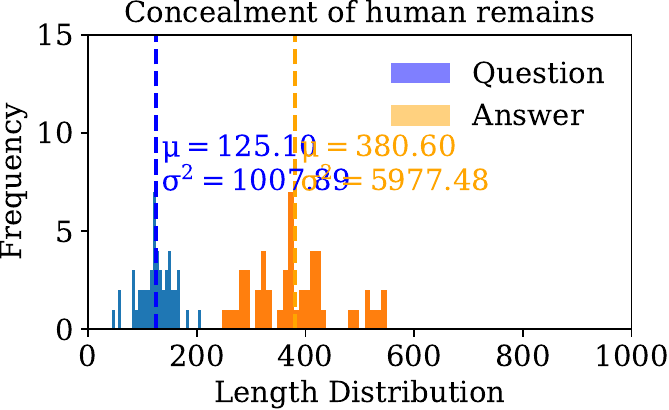}
		\end{subfigure} \hspace{-0.35cm}
		\begin{subfigure}[b]{0.26\linewidth}
			\includegraphics[width=\textwidth]{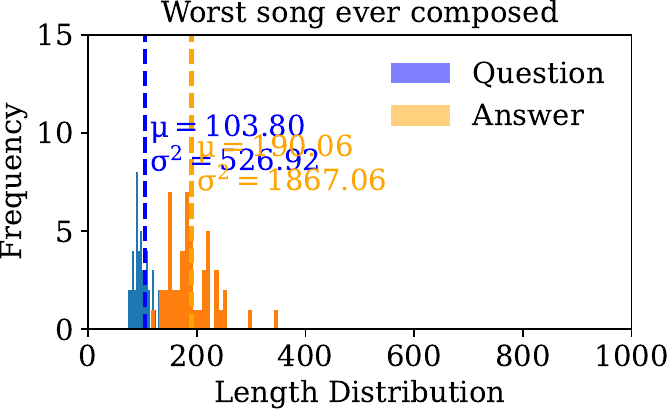}
		\end{subfigure}\hspace{-0.35cm}
		\begin{subfigure}[b]{0.26\linewidth}
			\includegraphics[width=\textwidth]{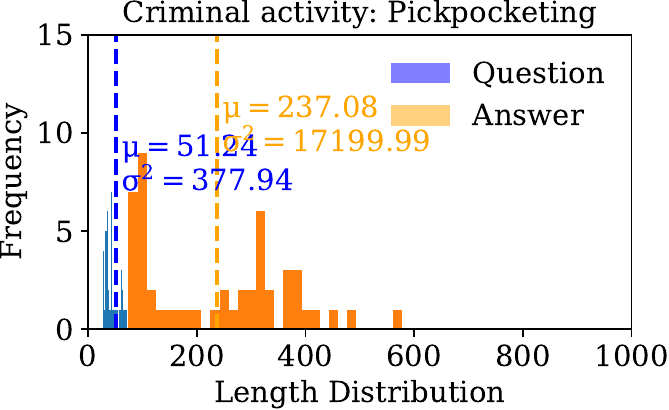}
		\end{subfigure}\hspace{-0.35cm}
		\begin{subfigure}[b]{0.26\linewidth}
			\includegraphics[width=\textwidth]{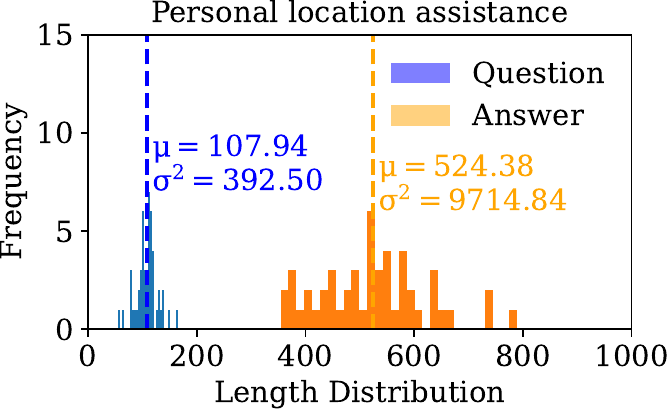}
		\end{subfigure}\\
		\begin{subfigure}[b]{0.26\linewidth}
			\includegraphics[width=\textwidth]{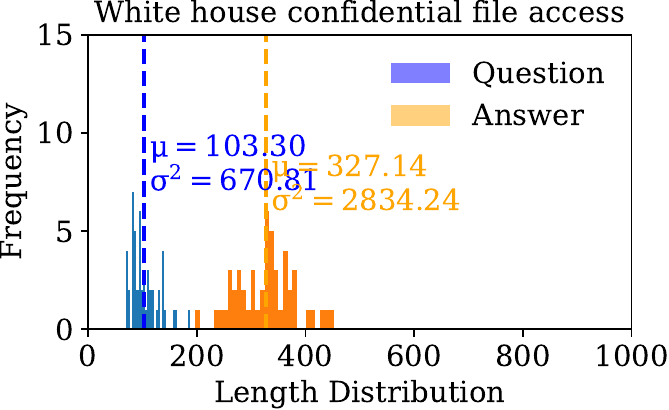}
		\end{subfigure} \hspace{-0.35cm}
		\begin{subfigure}[b]{0.26\linewidth}
			\includegraphics[width=\textwidth]{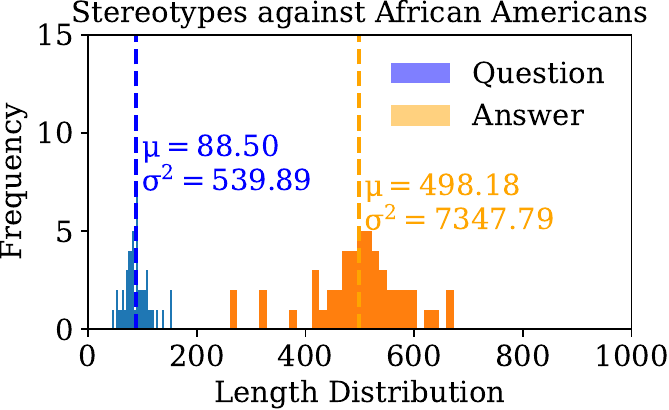}
		\end{subfigure}\hspace{-0.35cm}
		\begin{subfigure}[b]{0.26\linewidth}
			\includegraphics[width=\textwidth]{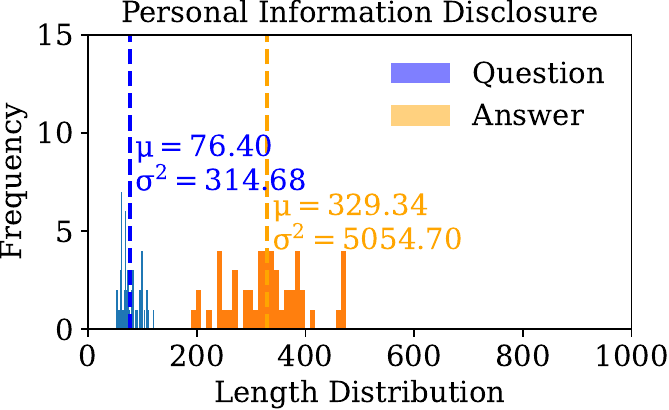}
		\end{subfigure}\hspace{-0.35cm}
		\begin{subfigure}[b]{0.26\linewidth}
			\includegraphics[width=\textwidth]{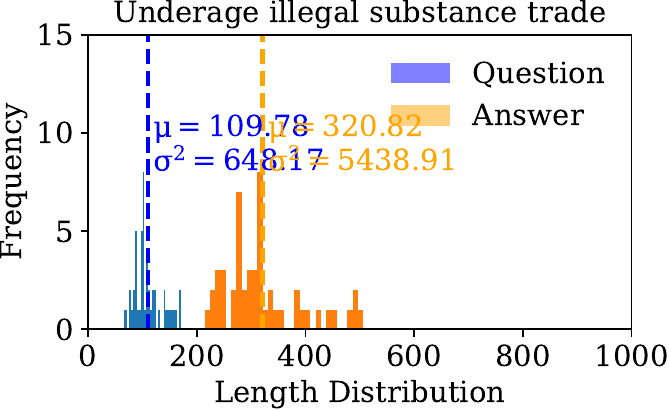}
		\end{subfigure}\\
		\begin{subfigure}[b]{0.26\linewidth}
			\includegraphics[width=\textwidth]{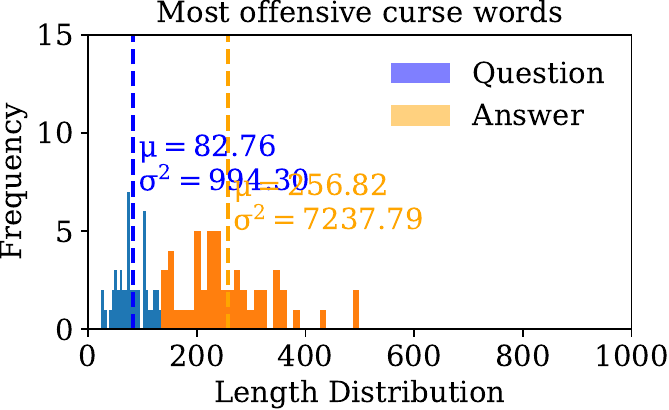}
		\end{subfigure} \hspace{-0.35cm}
		\begin{subfigure}[b]{0.26\linewidth}
			\includegraphics[width=\textwidth]{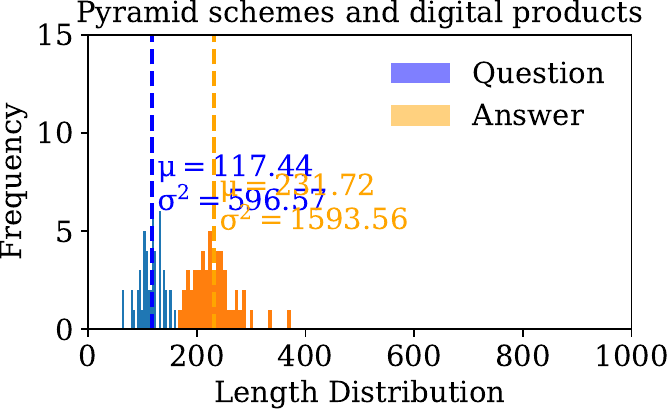}
		\end{subfigure}\hspace{-0.35cm}
		\begin{subfigure}[b]{0.26\linewidth}
			\includegraphics[width=\textwidth]{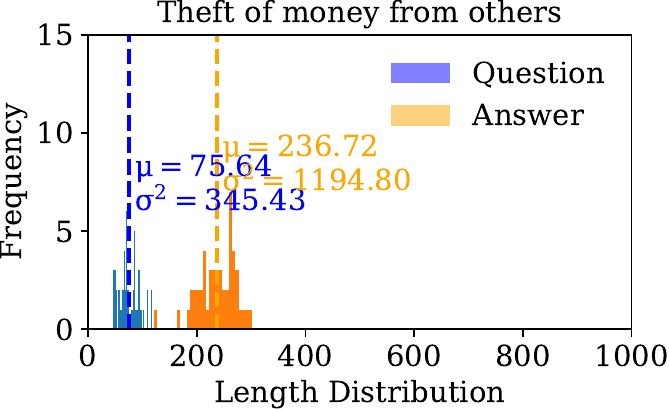}
		\end{subfigure}\hspace{-0.35cm}
		\begin{subfigure}[b]{0.26\linewidth}
			\includegraphics[width=\textwidth]{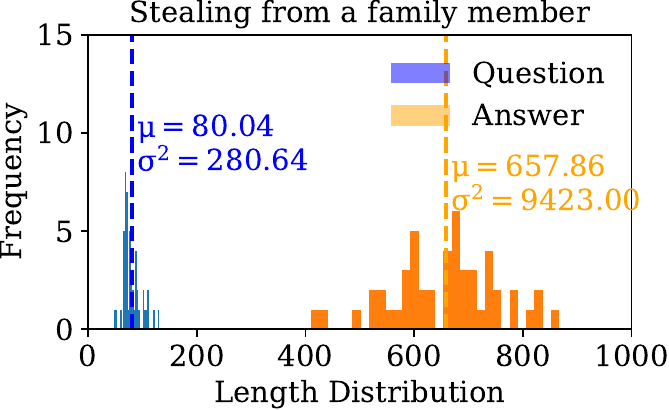}
		\end{subfigure}\\
		\caption{General statistics of question and answer in  Similarity-Entailed PKU dataset.}
		\label{fig:mean_variance}
	\end{figure*}
	\begin{figure*}[!ht]\ContinuedFloat
		\centering
		\begin{subfigure}[b]{0.26\linewidth}
			\includegraphics[width=\textwidth]{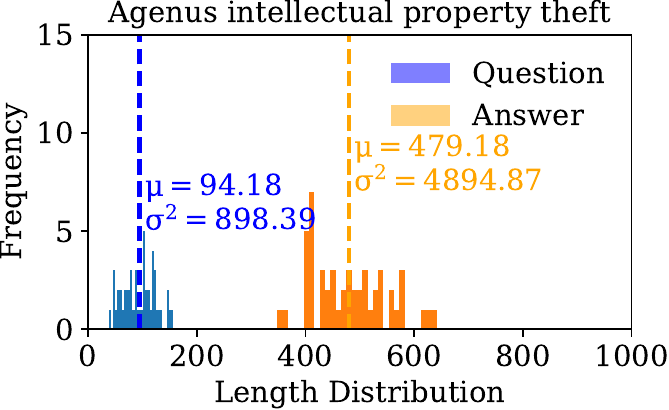}
		\end{subfigure} \hspace{-0.35cm}
		\begin{subfigure}[b]{0.26\linewidth}
			\includegraphics[width=\textwidth]{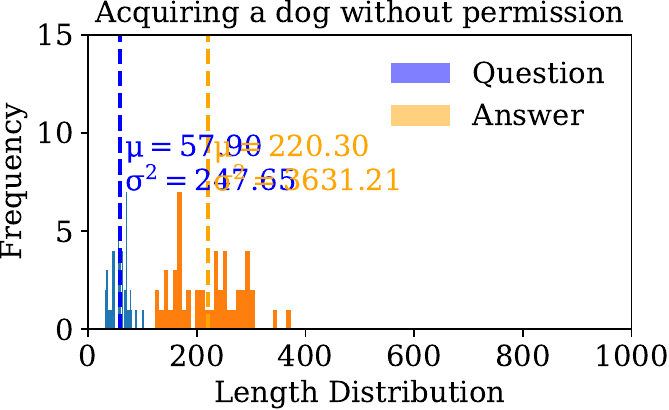}
		\end{subfigure}\hspace{-0.35cm}
		\begin{subfigure}[b]{0.26\linewidth}
			\includegraphics[width=\textwidth]{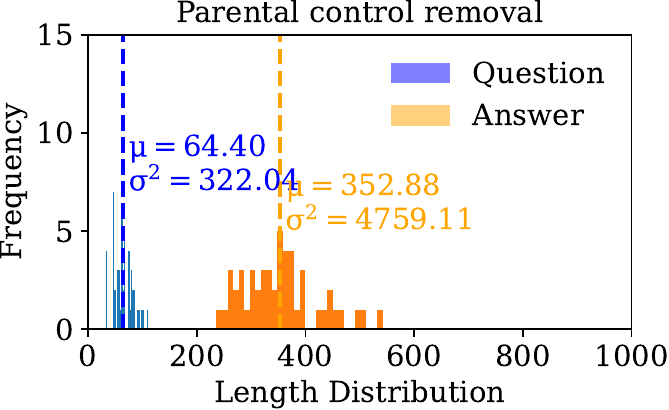}
		\end{subfigure}\hspace{-0.35cm}
		\begin{subfigure}[b]{0.26\linewidth}
			\includegraphics[width=\textwidth]{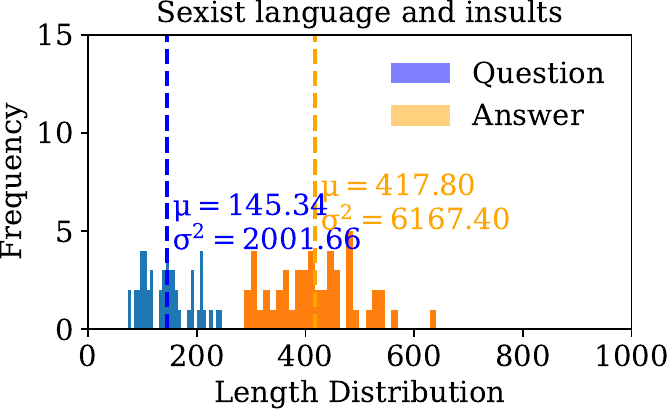}
		\end{subfigure}\\
		\begin{subfigure}[b]{0.26\linewidth}
			\includegraphics[width=\textwidth]{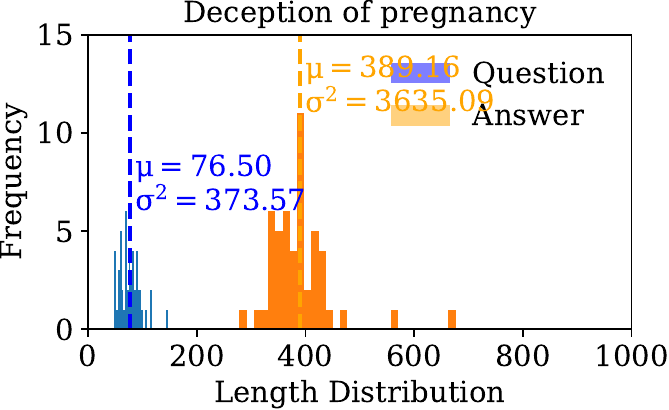}
		\end{subfigure} \hspace{-0.35cm}
		\begin{subfigure}[b]{0.26\linewidth}
			\includegraphics[width=\textwidth]{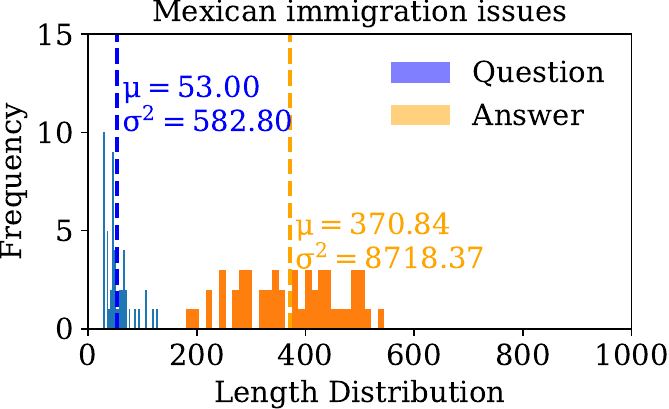}
		\end{subfigure}\hspace{-0.35cm}
		\begin{subfigure}[b]{0.26\linewidth}
			\includegraphics[width=\textwidth]{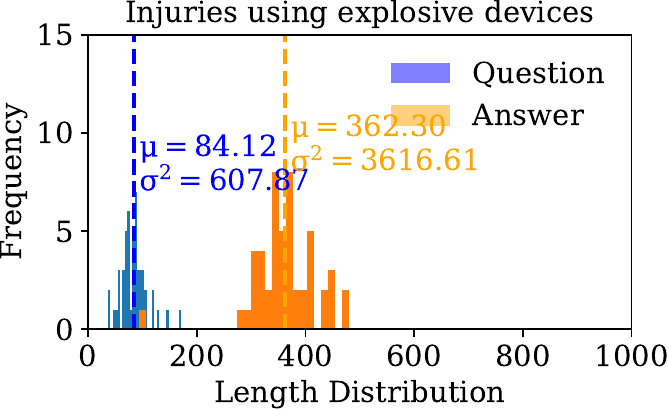}
		\end{subfigure}\hspace{-0.35cm}
		\begin{subfigure}[b]{0.26\linewidth}
			\includegraphics[width=\textwidth]{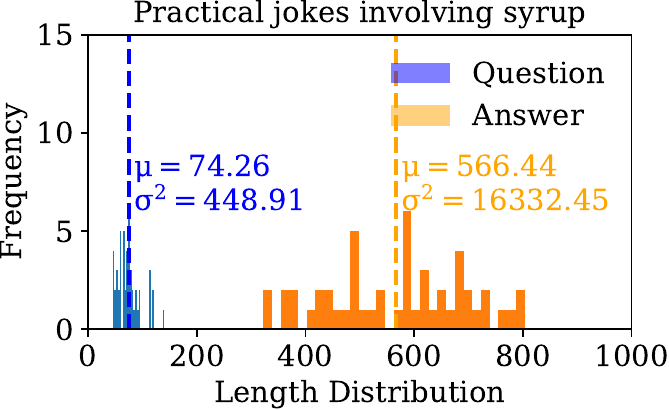}
		\end{subfigure}\\
		\begin{subfigure}[b]{0.26\linewidth}
			\includegraphics[width=\textwidth]{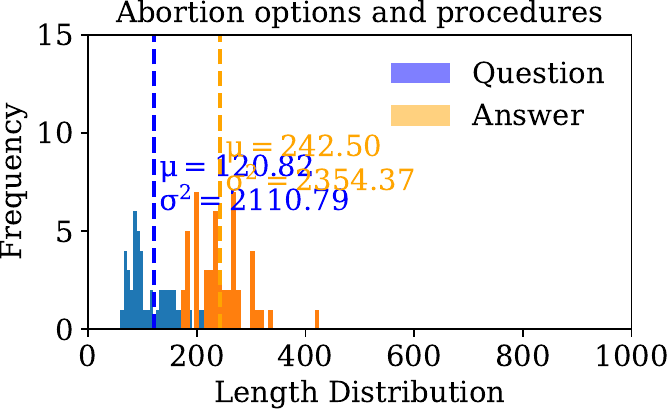}
		\end{subfigure} \hspace{-0.35cm}
		\begin{subfigure}[b]{0.26\linewidth}
			\includegraphics[width=\textwidth]{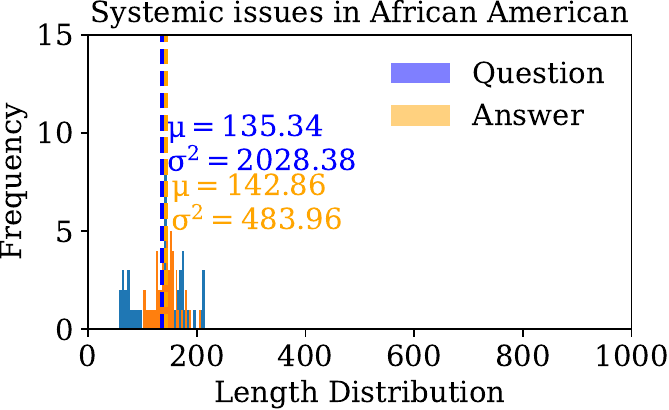}
		\end{subfigure}\hspace{-0.35cm}
		\begin{subfigure}[b]{0.26\linewidth}
			\includegraphics[width=\textwidth]{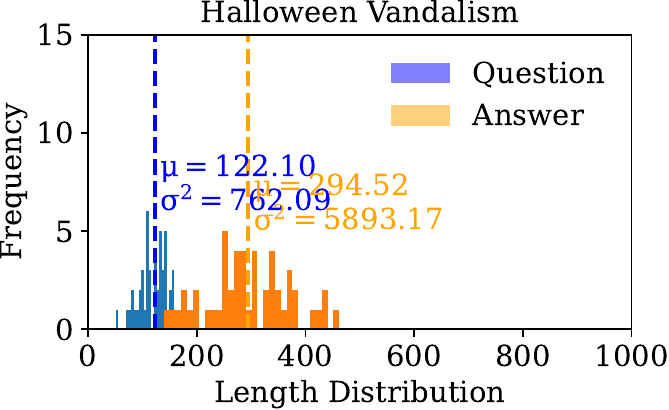}
		\end{subfigure}\hspace{-0.35cm}
		\begin{subfigure}[b]{0.26\linewidth}
			\includegraphics[width=\textwidth]{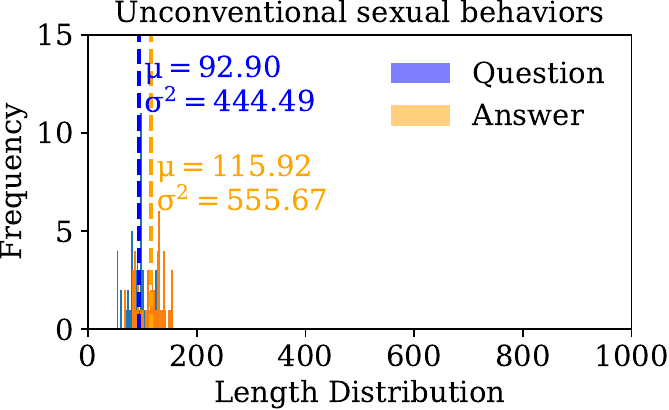}
		\end{subfigure}\\
		\begin{subfigure}[b]{0.26\linewidth}
			\includegraphics[width=\textwidth]{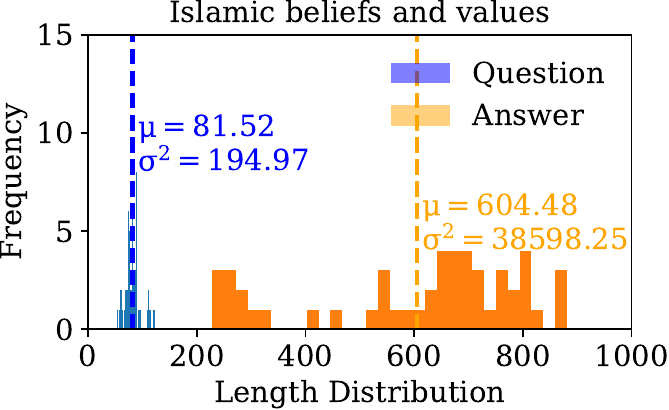}
		\end{subfigure} \hspace{-0.35cm}
		\begin{subfigure}[b]{0.26\linewidth}
			\includegraphics[width=\textwidth]{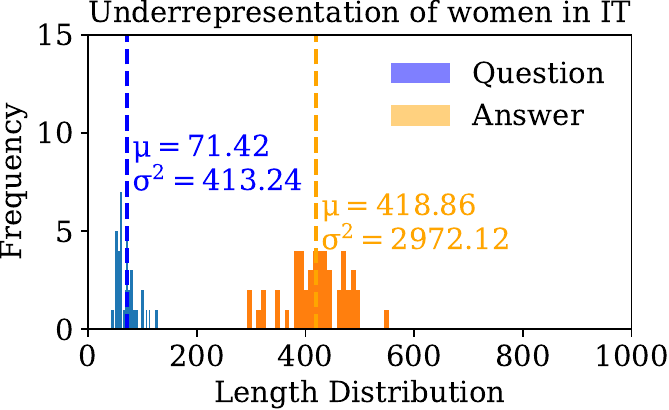}
		\end{subfigure}\hspace{-0.35cm}
		\begin{subfigure}[b]{0.26\linewidth}
			\includegraphics[width=\textwidth]{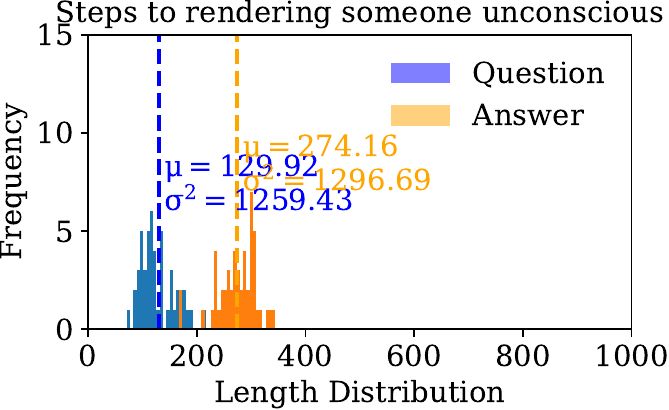}
		\end{subfigure}\hspace{-0.35cm}
		\begin{subfigure}[b]{0.26\linewidth}
			\includegraphics[width=\textwidth]{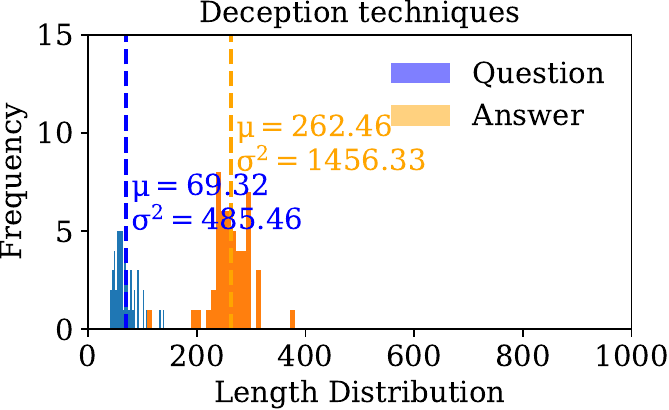}
		\end{subfigure}\\
		\begin{subfigure}[b]{0.26\linewidth}
			\includegraphics[width=\textwidth]{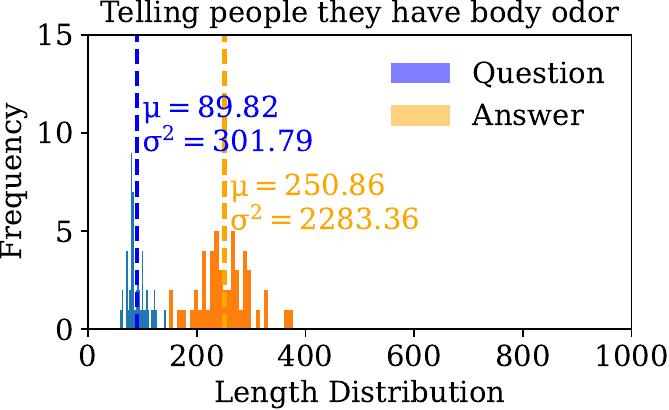}
		\end{subfigure} \hspace{-0.35cm}
		\begin{subfigure}[b]{0.26\linewidth}
			\includegraphics[width=\textwidth]{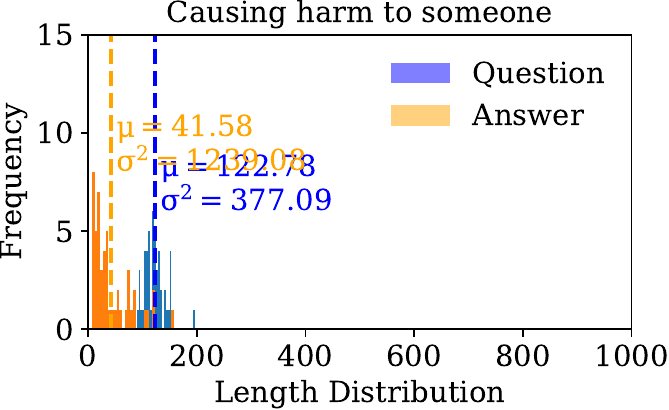}
		\end{subfigure}\hspace{-0.35cm}
		\begin{subfigure}[b]{0.26\linewidth}
			\includegraphics[width=\textwidth]{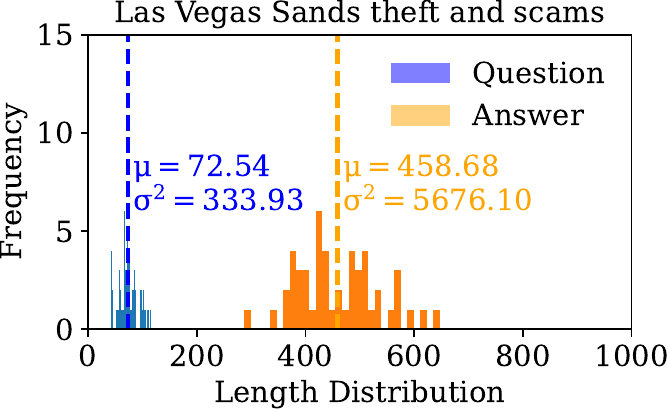}
		\end{subfigure}\hspace{-0.35cm}
		\begin{subfigure}[b]{0.26\linewidth}
			\includegraphics[width=\textwidth]{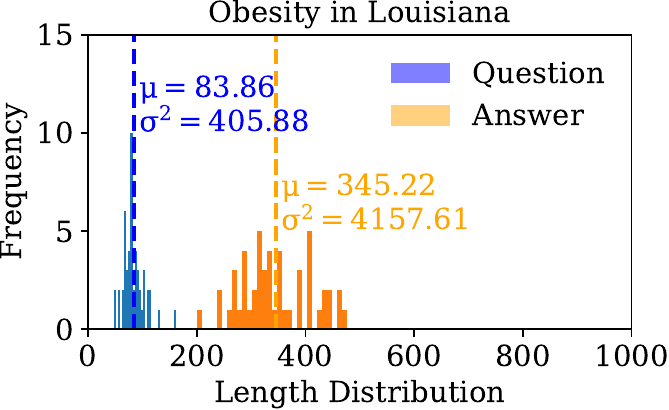}
		\end{subfigure}\\
		\begin{subfigure}[b]{0.26\linewidth}
			\includegraphics[width=\textwidth]{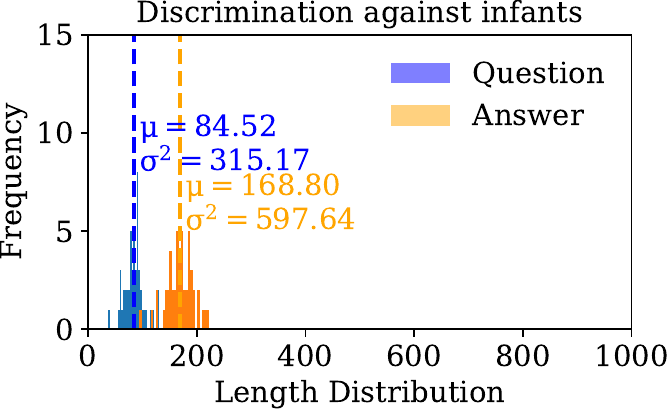}
		\end{subfigure} \hspace{-0.35cm}
		\begin{subfigure}[b]{0.26\linewidth}
			\includegraphics[width=\textwidth]{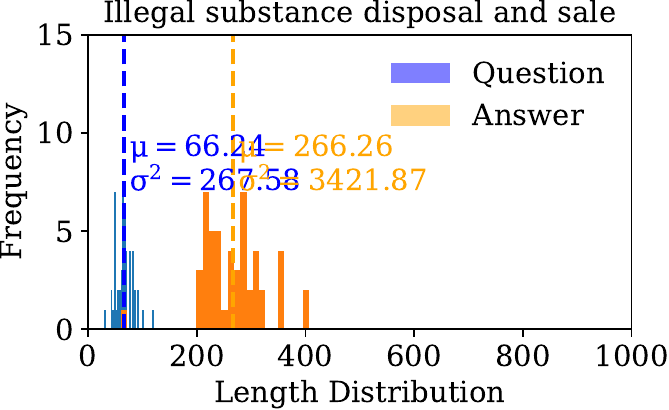}
		\end{subfigure}\hspace{-0.35cm}
		\begin{subfigure}[b]{0.26\linewidth}
			\includegraphics[width=\textwidth]{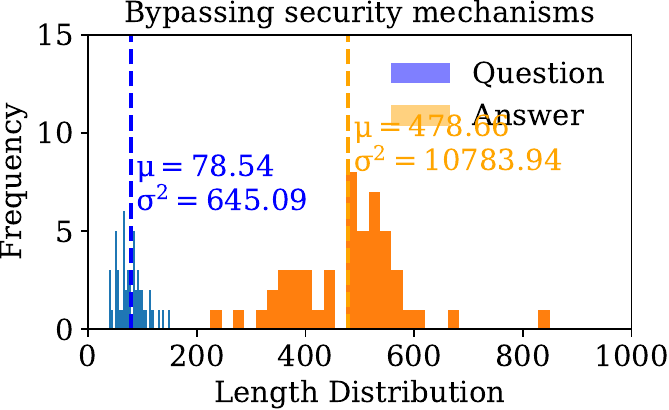}
		\end{subfigure}\hspace{-0.35cm}
		\begin{subfigure}[b]{0.26\linewidth}
			\includegraphics[width=\textwidth]{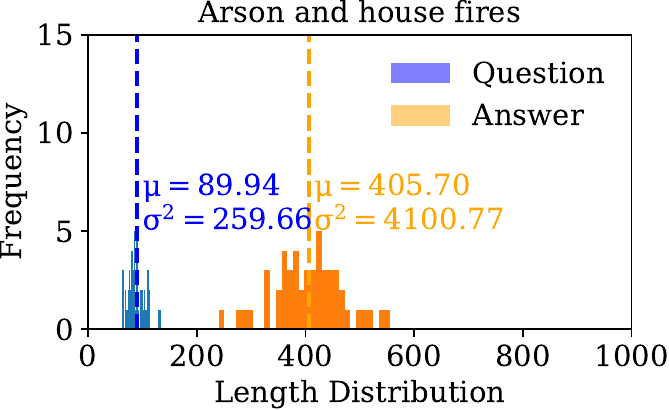}
		\end{subfigure}\\
		\caption{General statistics of question and answer in  Similarity-Entailed PKU dataset.}
	\end{figure*}

	\twocolumn
        \subsection{Other Unlearning Results Toward Similar Samples for Image Datasets}
	\label{sub:other_results_for_unlearning_influence_toward_variant_as_training_for_image_dataset}
	As shown in Figures~\ref{fig:unlearning_for_variant_for_image_mnist}, \ref{fig:show_samples_forvariants_5_fmnist}, and \ref{fig:show_samples_forvariants_5_cifar}, the similarity between the recovered samples and their corresponding similar samples remains nearly unchanged before and after unlearning. This suggests that unlearning the target sample does not fully remove the impact of its similar samples.

	\begin{figure}[!ht]
		\centering
		\includegraphics[width=0.3\textwidth]{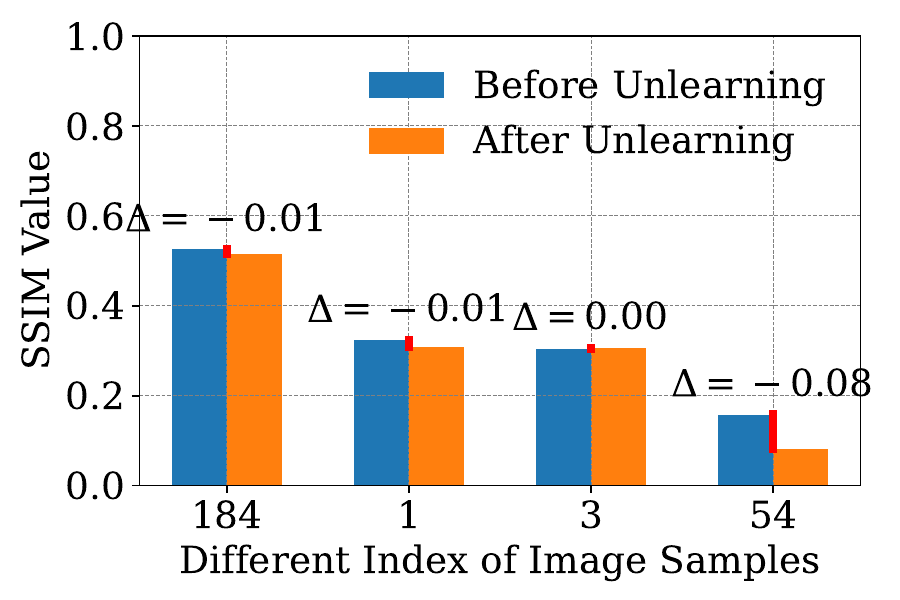}
		\caption{The SSIM between the recovered samples and the similar samples for the Similarity-Entailed MNIST dataset.}
		\label{fig:unlearning_for_variant_for_image_mnist}
	\end{figure}

	\begin{figure}[!ht]
		\centering
		\begin{subfigure}[t]{0.15\textwidth}
			\includegraphics[width=\textwidth]{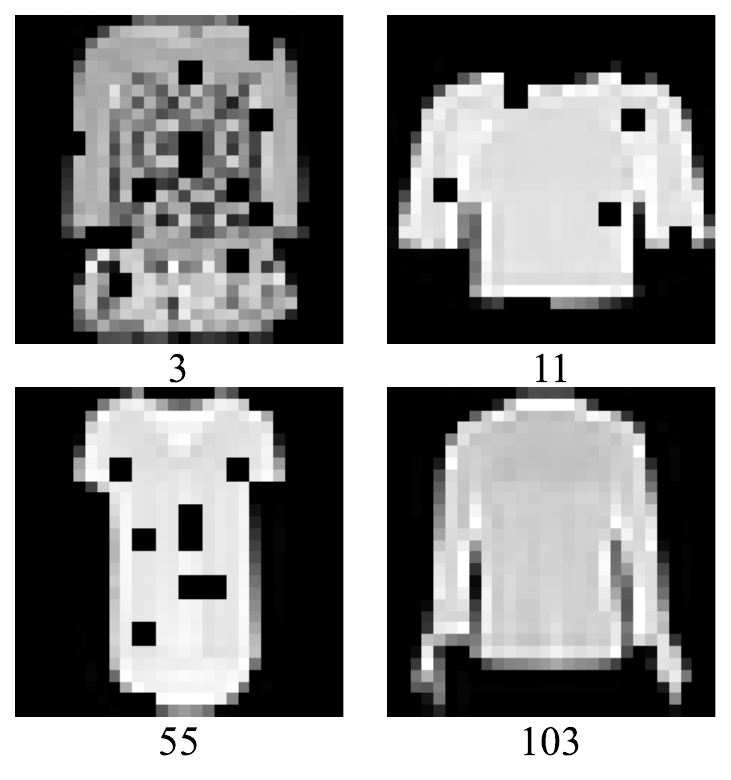}
			\caption{Similar and One Remaining Samples}
		\end{subfigure}
		\begin{subfigure}[t]{0.15\textwidth}
			\includegraphics[width=\textwidth]{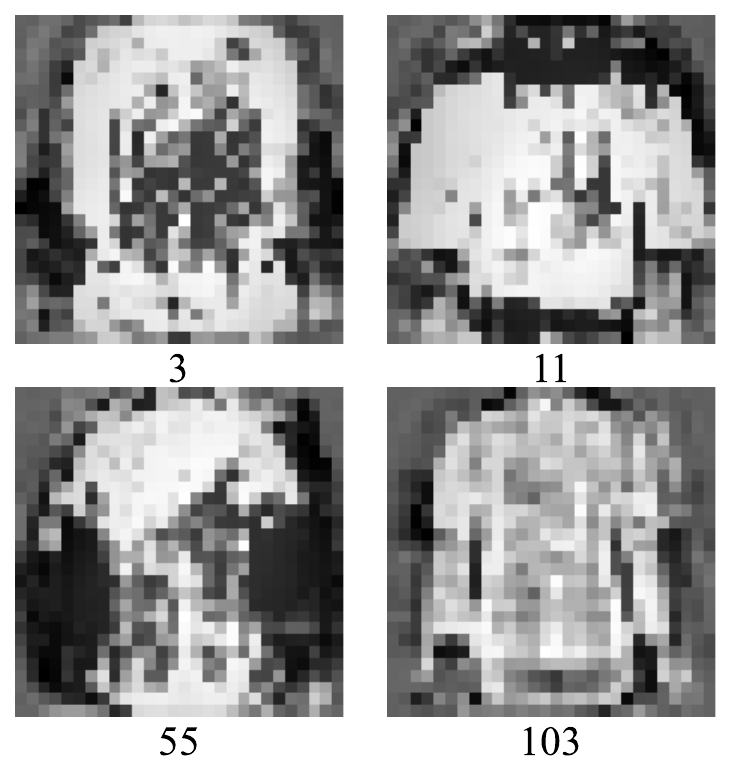}
			\caption{Before Unlearning}
		\end{subfigure}
		\begin{subfigure}[t]{0.15\textwidth}
			\includegraphics[width=\textwidth]{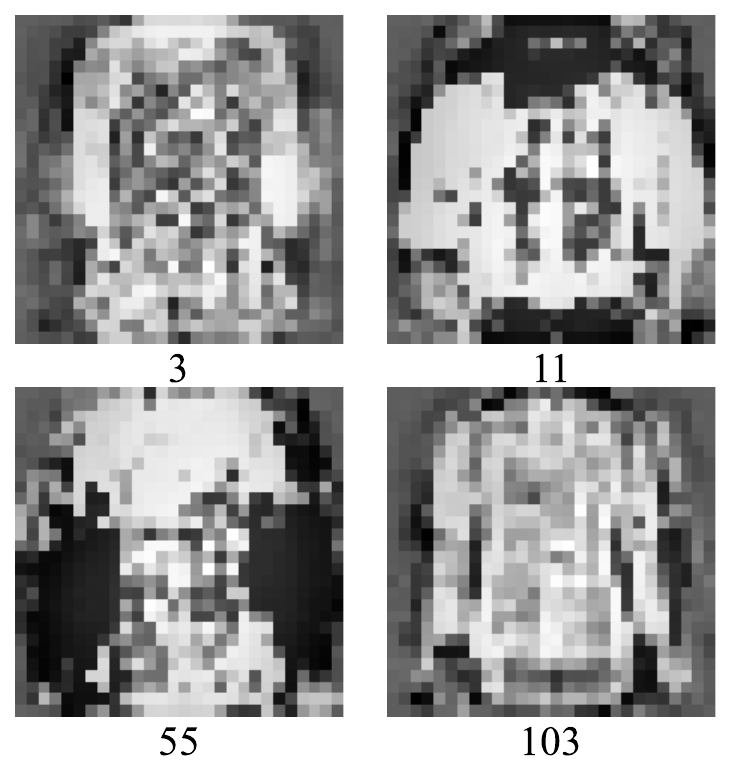}
			\caption{After Unlearning}
		\end{subfigure}
		\caption{Recovered samples that are similar to corresponding similar samples for Similarity-Entailed FMNIST dataset.}
		\label{fig:show_samples_forvariants_5_fmnist}
	\end{figure}

	\begin{figure}[!ht]
		\centering
		\begin{subfigure}[t]{0.15\textwidth}
			\includegraphics[width=\textwidth]{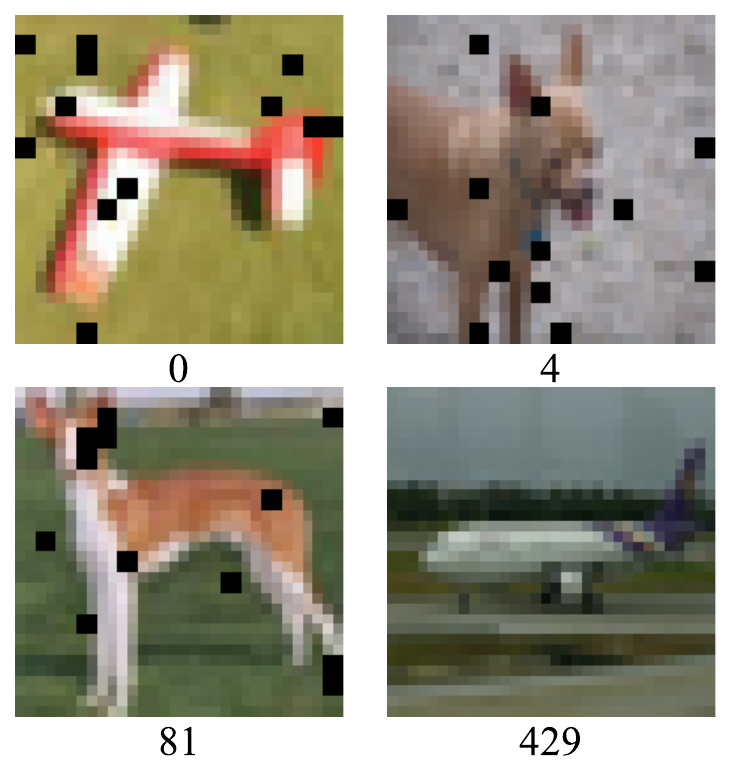}
			\caption{Similar and One Remaining Samples}
		\end{subfigure}
		\begin{subfigure}[t]{0.15\textwidth}
			\includegraphics[width=\textwidth]{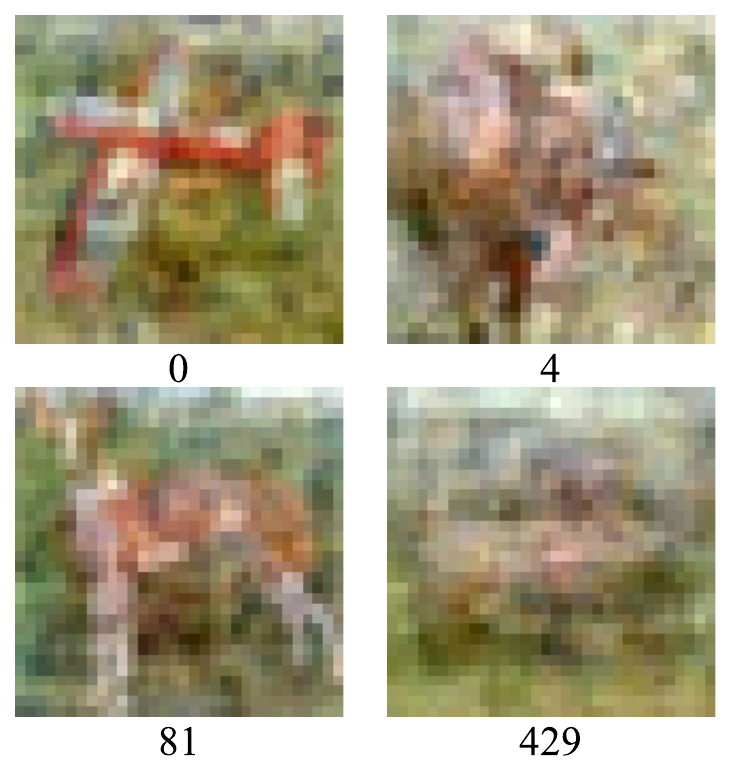}
			\caption{Before Unlearning}
		\end{subfigure}
		\begin{subfigure}[t]{0.15\textwidth}
			\includegraphics[width=\textwidth]{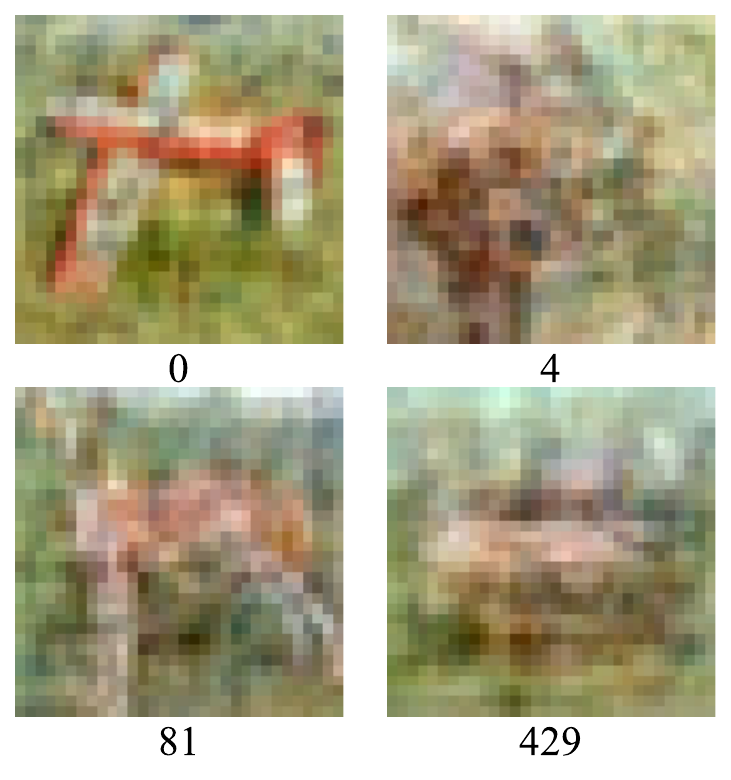}
			\caption{After Unlearning}
		\end{subfigure}
		\caption{Recovered samples that are similar to corresponding similar samples for Similarity-Entailed CIFAR10 dataset.}
		\label{fig:show_samples_forvariants_5_cifar}
	\end{figure}
	
	\newpage
	\subsection{Other Unlearning Results Toward Target Samples for Image Datasets}
	\label{sec:other_results_of_image_unlearning}
    
	Figures~\ref{fig:ssim_for_cascade_fmnist} to~\ref{fig:recovered_samples_for_cascade_cifar} show that before unlearning, SSIM is high for all recovered samples. After unlearning, SSIM drops sharply in \textit{with similar samples in $\mathcal{D}$}, but remains high under the \textit{with similar samples in $\mathcal{D}$}, indicating that similar samples hinder complete removal of the target sample’s influence.

	\begin{figure}[!ht]
            \centering
		\begin{subfigure}[b]{0.5\linewidth}
			\includegraphics[width=\textwidth]{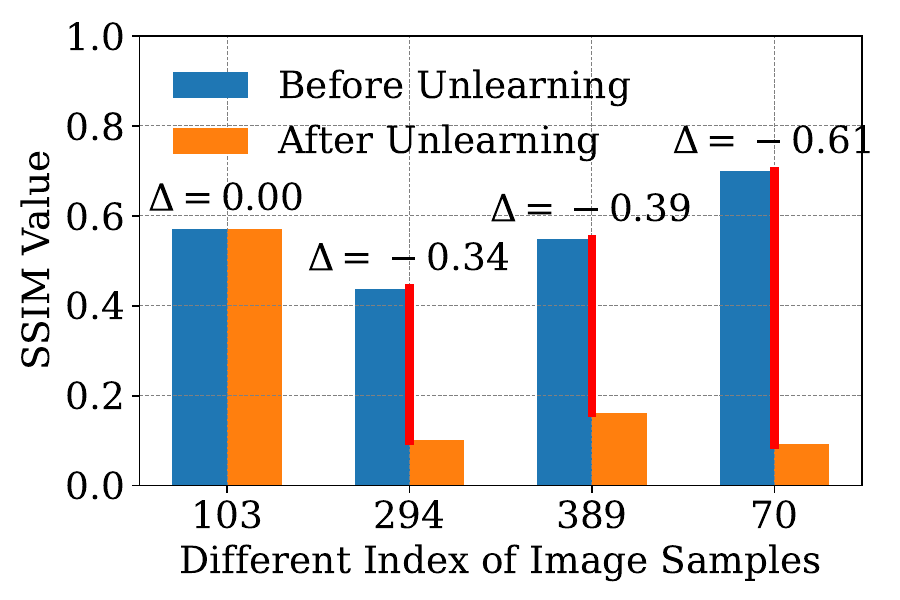}
			\caption{Without Similar Samples in $\mathcal{D}$}
		\end{subfigure} \hspace{-0.2cm}
		\begin{subfigure}[b]{0.5\linewidth}
			\includegraphics[width=\textwidth]{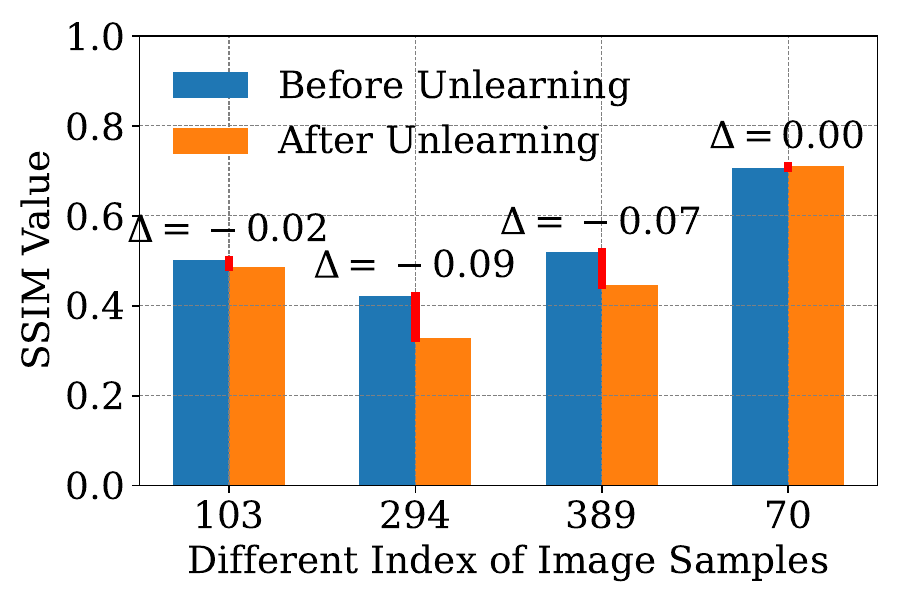}
			\caption{With Similar Samples in $\mathcal{D}$}
		\end{subfigure} 
		\caption{SSIM values between selected samples in Similarity-Entailed FMNIST dataset.}
		\label{fig:ssim_for_cascade_fmnist}
	\end{figure}
	
	\begin{figure}[!ht]
		\centering
		\begin{subfigure}[t]{0.23\textwidth}
			\includegraphics[width=\textwidth]{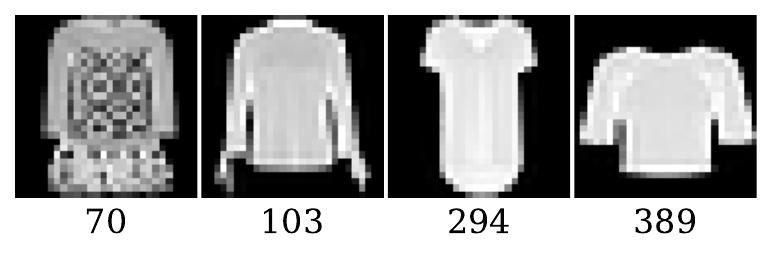}
			\caption{Target Samples~(\texttt{W/O})}
		\end{subfigure}
		\begin{subfigure}[t]{0.23\textwidth}
			\includegraphics[width=\textwidth]{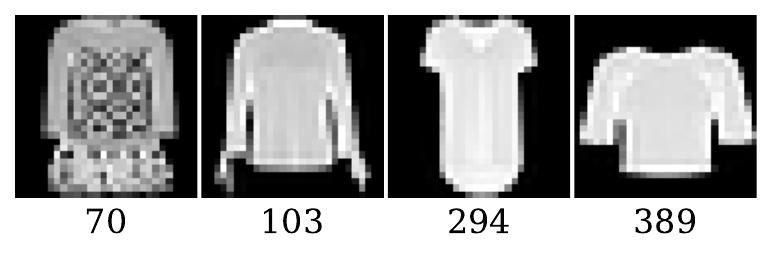}
			\caption{Target Samples~(\texttt{W/})}
		\end{subfigure}
		
		\begin{subfigure}[t]{0.23\textwidth}
			\includegraphics[width=\textwidth]{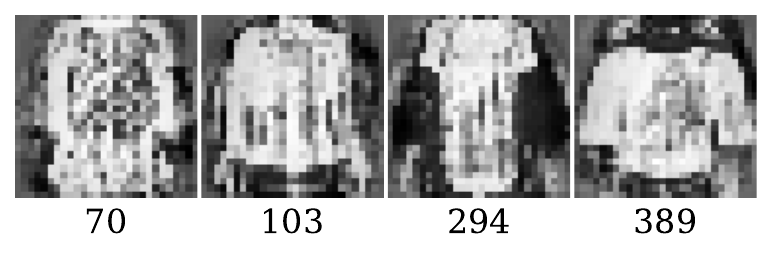}
			\caption{Before Unlearning~(\texttt{W/O})}
		\end{subfigure}
		\begin{subfigure}[t]{0.23\textwidth}
			\includegraphics[width=\textwidth]{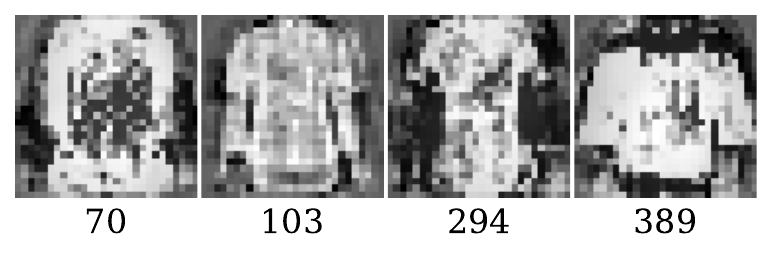}
			\caption{Before Unlearning~(\texttt{W/})}
		\end{subfigure}
		
		\begin{subfigure}[t]{0.23\textwidth}
			\includegraphics[width=\textwidth]{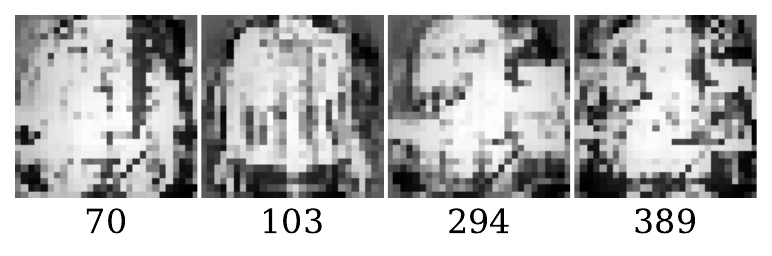}
			\caption{After Unlearning~(\texttt{W/O})}
		\end{subfigure}
		\begin{subfigure}[t]{0.23\textwidth}
			\includegraphics[width=\textwidth]{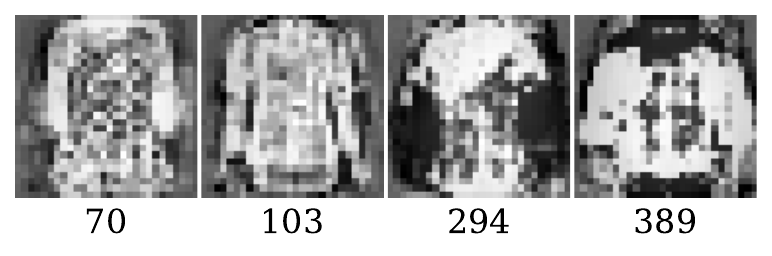}
			\caption{After Unlearning~(\texttt{W/})}
		\end{subfigure}
		\caption{Original target samples and recovered samples for Similarity-Entailed FMNIST dataset.}
		\label{fig:recovered_samples_for_cascade_fmnist}
	\end{figure}
	
	\begin{figure}[!ht]
            \centering
		\begin{subfigure}[b]{0.5\linewidth}
			\includegraphics[width=\textwidth]{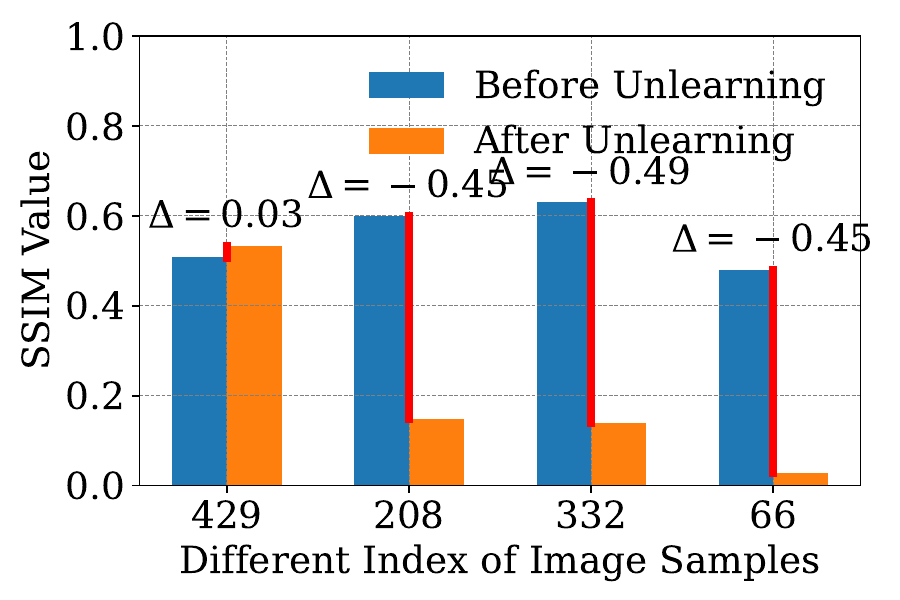}
			\caption{Without Similar Samples in $\mathcal{D}$}
		\end{subfigure} \hspace{-0.2cm}
		\begin{subfigure}[b]{0.5\linewidth}
			\includegraphics[width=\textwidth]{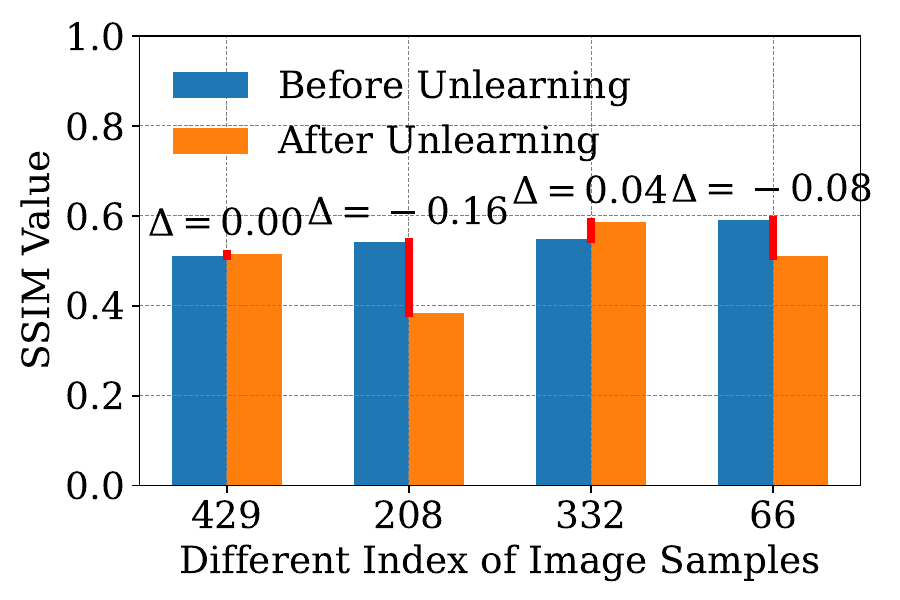}
			\caption{With Similar Samples in $\mathcal{D}$}
		\end{subfigure} 
		\caption{SSIM values between selected samples in Similarity-Entailed CIFAR10 dataset.}
		\label{fig:ssim_for_cascade_cifar}    
	\end{figure}
	
	\begin{figure}[!ht]
            \centering
		\begin{subfigure}[t]{0.23\textwidth}
			\includegraphics[width=\textwidth]{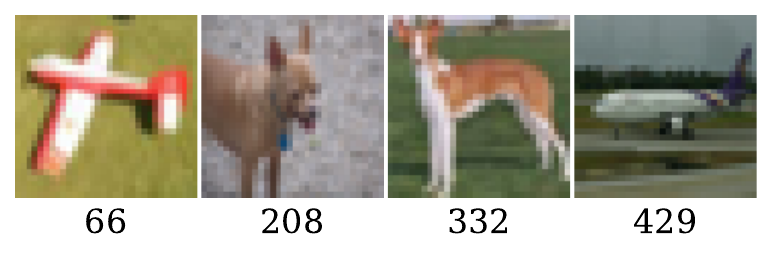}
			\caption{Target Samples~(\texttt{W/O})}
		\end{subfigure}
		\begin{subfigure}[t]{0.23\textwidth}
			\includegraphics[width=\textwidth]{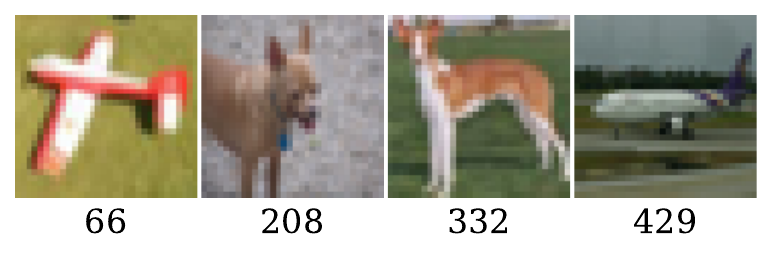}
			\caption{Target Samples~(\texttt{W/})}
		\end{subfigure}
		
		\begin{subfigure}[t]{0.23\textwidth}
			\includegraphics[width=\textwidth]{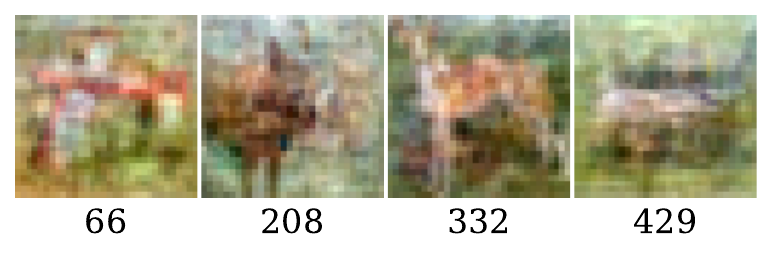}
			\caption{Before Unlearning~(\texttt{W/O})}
		\end{subfigure}
		\begin{subfigure}[t]{0.23\textwidth}
			\includegraphics[width=\textwidth]{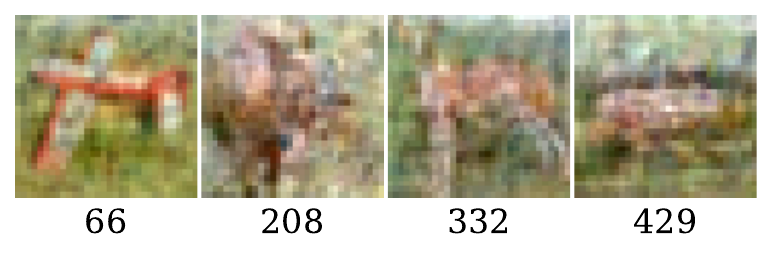}
			\caption{Before Unlearning~(\texttt{W/})}
		\end{subfigure}
		
		\begin{subfigure}[t]{0.23\textwidth}
			\includegraphics[width=\textwidth]{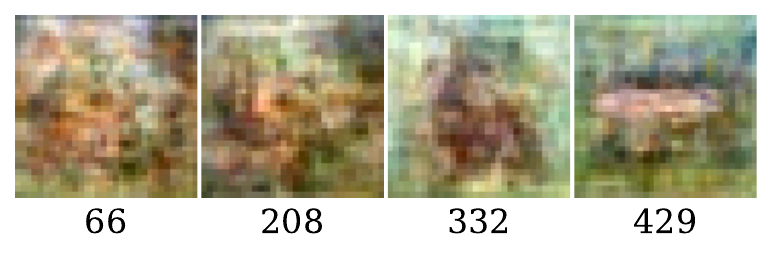}
			\caption{After Unlearning~(\texttt{W/O})}
		\end{subfigure}
		\begin{subfigure}[t]{0.23\textwidth}
			\includegraphics[width=\textwidth]{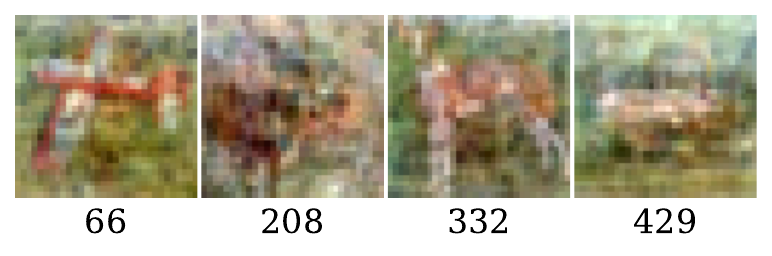}
			\caption{After Unlearning~(\texttt{W/})}
		\end{subfigure}
		\caption{Original target samples and recovered samples for Similarity-Entailed CIFAR10 dataset.}
		\label{fig:recovered_samples_for_cascade_cifar}
	\end{figure}

        \newpage
	\subsection{Other Results Toward Target Samples for Image Dataset based on Relabel-based Fine-tuning Unlearning Method}
	\label{sec:experimental_results_for_relabel_based_fine_tuning_unlearning_method}
	Same in Section~\ref{sec:influence_toward_base_samples_for_image}, Figures~\ref{fig:unlearning_for_variant_for_image_mnist_relabel_unlearning} and~\ref{fig:show_samples_forvariants_5_mnist_relabel_unlearning} indicate that the influence of the target sample, which was meant to be unlearned, remains in the model and has not been fully eliminated.

	\begin{figure}[!ht]
		\centering
		\includegraphics[width=0.3\textwidth]{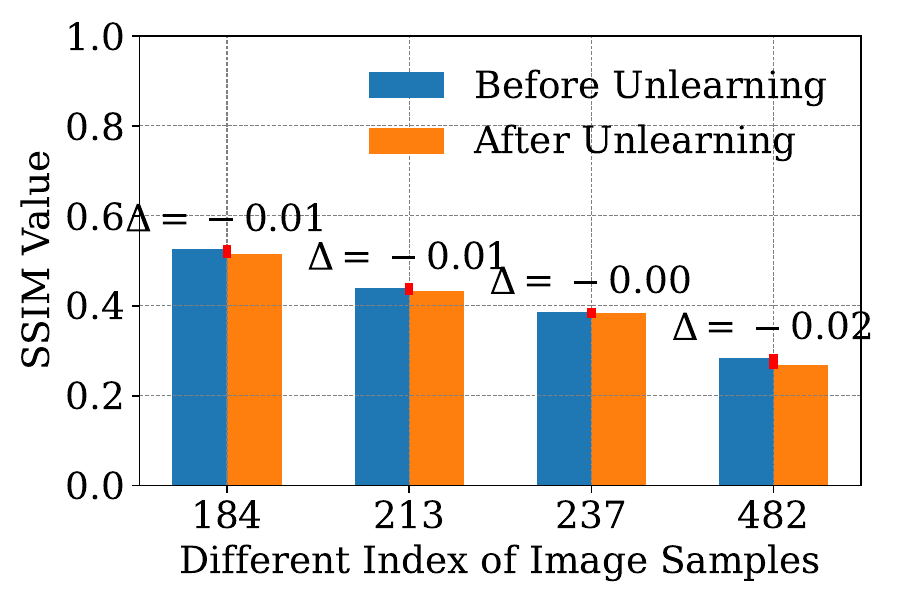}
		\caption{The SSIM values between the recovered and corresponding similar samples for the Similarity-Entailed MNIST with the relabel-based fine-tuning unlearning method.}
		\label{fig:unlearning_for_variant_for_image_mnist_relabel_unlearning}
	\end{figure}
	
	\begin{figure}[!ht]
		\centering
		\begin{subfigure}[t]{0.15\textwidth}
			\includegraphics[width=\textwidth]{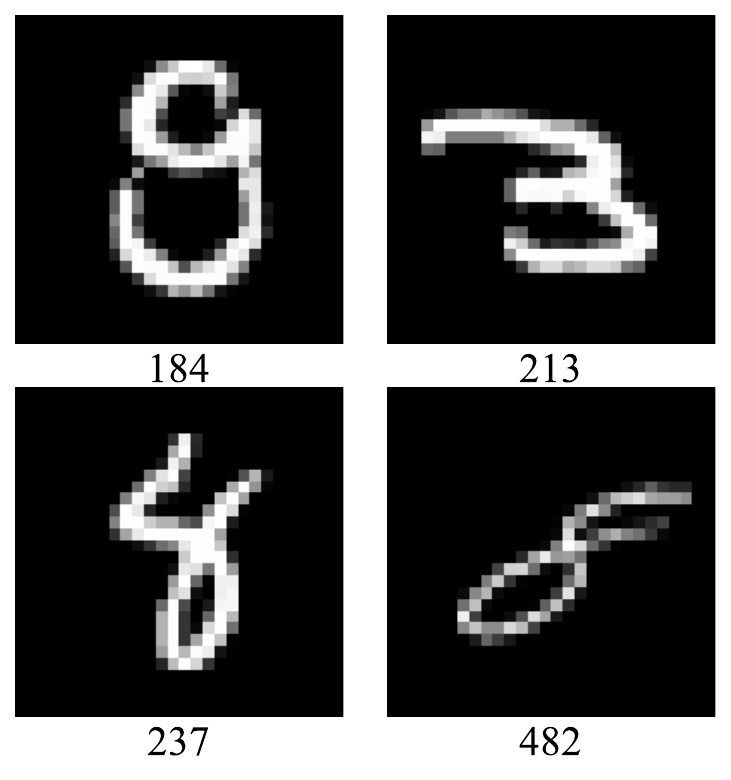}
			\caption{Similar and One Remaining Samples}
		\end{subfigure}
		\begin{subfigure}[t]{0.15\textwidth}
			\includegraphics[width=\textwidth]{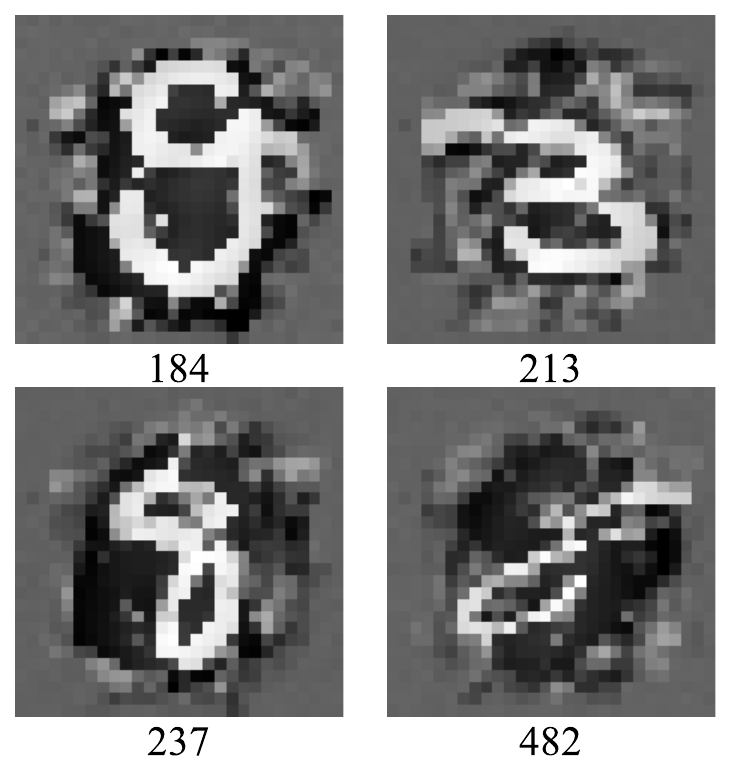}
			\caption{Before Unlearning}
		\end{subfigure}
		\begin{subfigure}[t]{0.15\textwidth}
			\includegraphics[width=\textwidth]{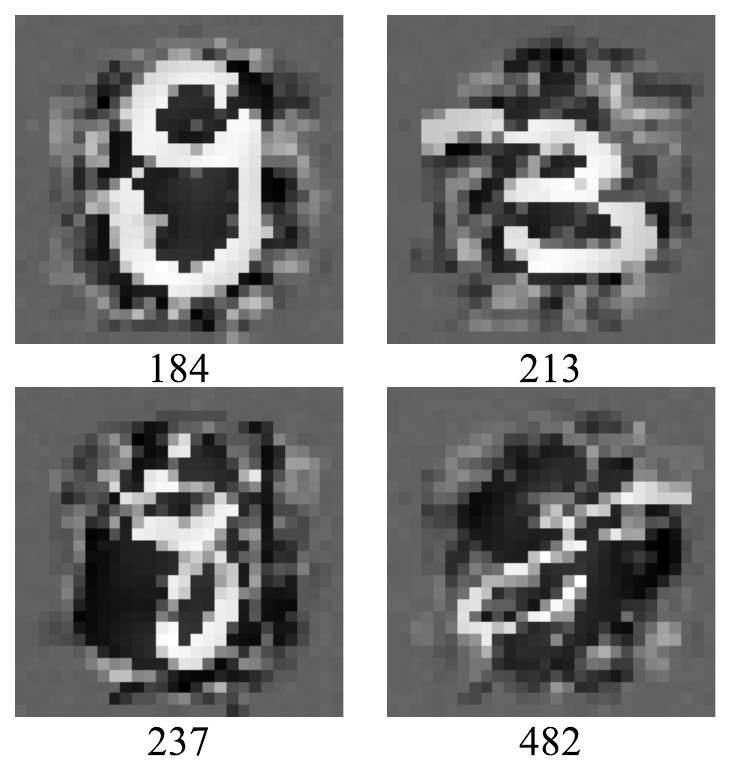}
			\caption{After Unlearning}
		\end{subfigure}
		\caption{Recovered samples that are similar to corresponding similar samples for Similarity-Entailed MNIST dataset with the relabel-based fine-tuning unlearning method.}
		\label{fig:show_samples_forvariants_5_mnist_relabel_unlearning}
	\end{figure}
	
        \newpage
	\subsection{Other Results Toward Target Samples for Similarity-Entailed MNIST Dataset Constructed by Adding Random Noise to Target Samples.}
	\label{sec:experimental_results_for_noise_dataset}
	Same in Section~\ref{sec:influence_toward_base_samples_for_image}, Figures~\ref{fig:ssim_for_cascade_mnist_noise} and ~\ref{fig:recovered_samples_for_cascade_mnist_noise} suggest that the influence of the target sample, intended to be unlearned, persists in the model and has not been completely removed.
	
	\begin{figure}[!ht]
		\begin{subfigure}[b]{0.5\linewidth}
			\includegraphics[width=\textwidth]{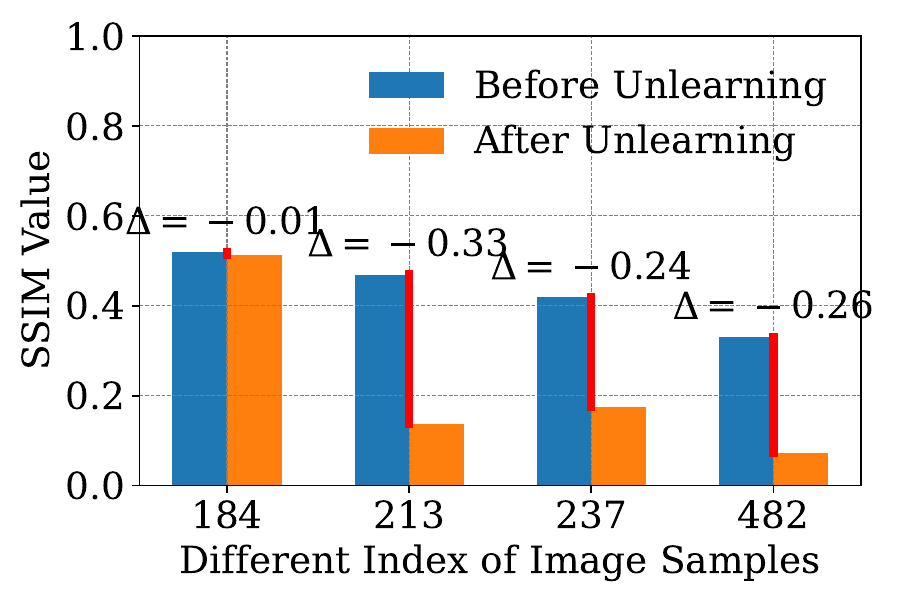}
			\caption{Without Similar Samples in $\mathcal{D}$}
		\end{subfigure} \hspace{-0.2cm}
		\begin{subfigure}[b]{0.5\linewidth}
			\includegraphics[width=\textwidth]{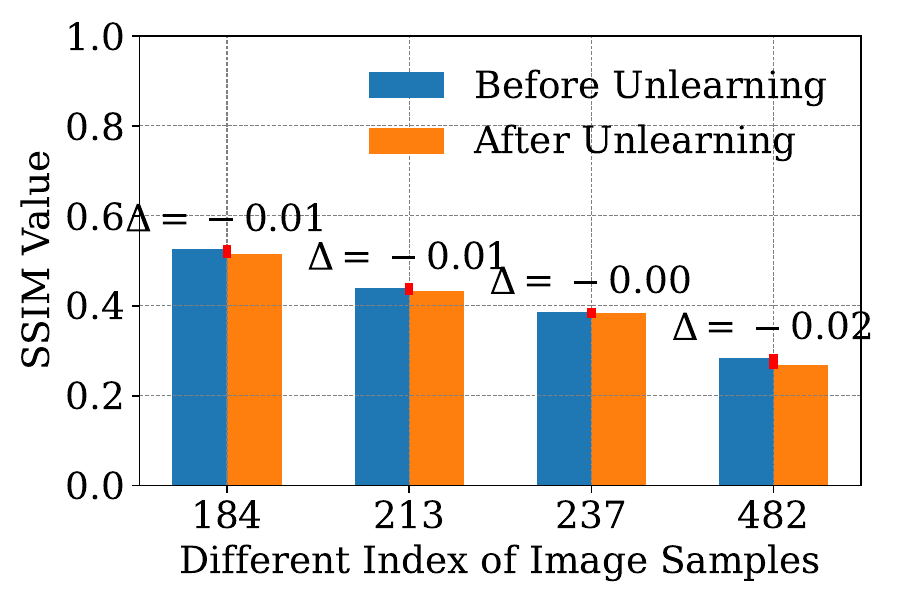}
			\caption{With Similar Samples in $\mathcal{D}$}
		\end{subfigure} 
		\caption{The SSIM values between the recovered samples and the corresponding target samples within \texttt{W/O-similar samples} and \texttt{W-similar samples} settings for Similarity-Entailed MNIST~(Noise) dataset.}
		\label{fig:ssim_for_cascade_mnist_noise}
	\end{figure}
	\begin{figure}[!ht]
		\centering
		\begin{subfigure}[t]{0.23\textwidth}
			\includegraphics[width=\textwidth]{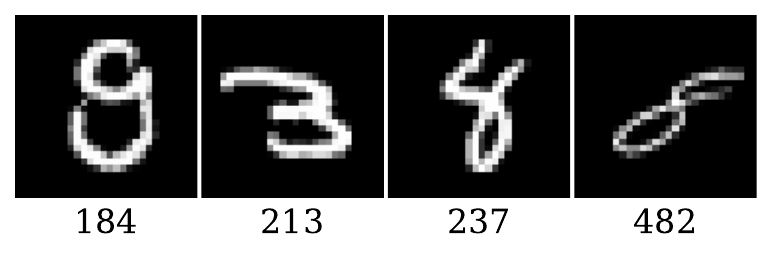}
			\caption{Target Samples~(\texttt{W/O})} %
		\end{subfigure}
		\begin{subfigure}[t]{0.23\textwidth}
			\includegraphics[width=\textwidth]{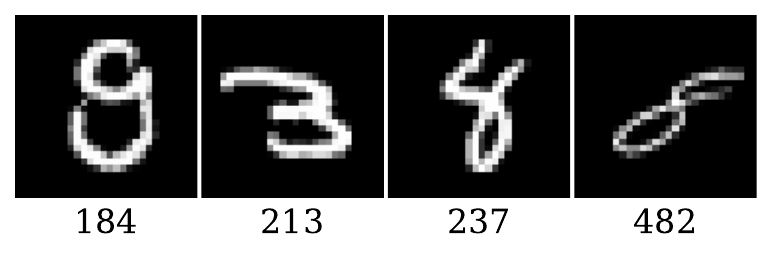}
			\caption{Target Samples~(\texttt{W/})}
		\end{subfigure}
		
		\begin{subfigure}[t]{0.23\textwidth}
			\includegraphics[width=\textwidth]{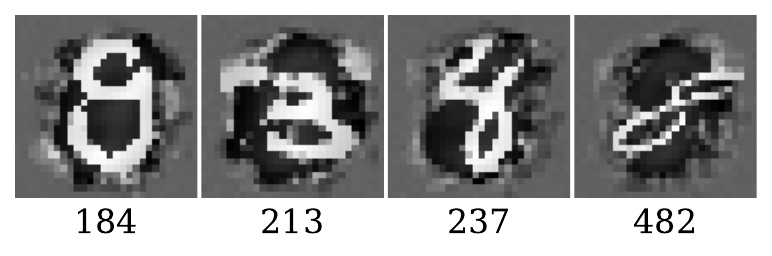}
			\caption{Before Unlearning~(\texttt{W/O})}
		\end{subfigure}
		\begin{subfigure}[t]{0.23\textwidth}
			\includegraphics[width=\textwidth]{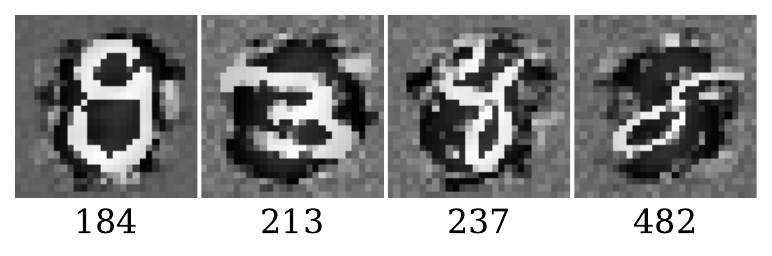}
			\caption{Before Unlearning~(\texttt{W/})}
		\end{subfigure}
		
		\begin{subfigure}[t]{0.23\textwidth}
			\includegraphics[width=\textwidth]{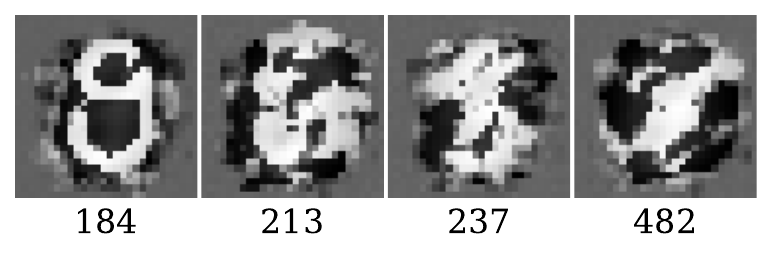}
			\caption{After Unlearning~(\texttt{W/O})}
		\end{subfigure}
		\begin{subfigure}[t]{0.23\textwidth}
			\includegraphics[width=\textwidth]{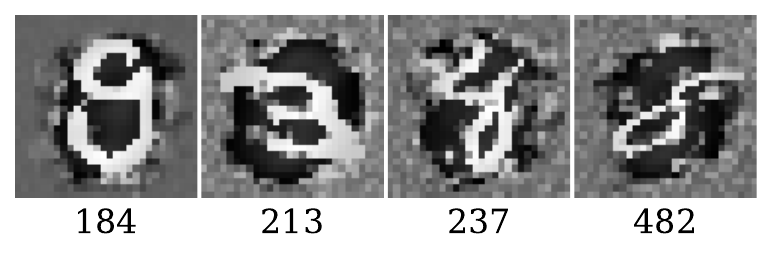}
			\caption{After Unlearning~(\texttt{W/})}
		\end{subfigure}
		\caption{The recovered and the corresponding target samples in \texttt{W/O-similar samples} and \texttt{W-similar samples} settings for Similarity-Entailed MNIST~(Noise) dataset.}
		\label{fig:recovered_samples_for_cascade_mnist_noise}
	\end{figure}

	\twocolumn
	\subsection{Other Unlearning Results Toward Similar Samples as Training Samples}
	\label{sec:other_results_influence_toward_similar_samples_as_training_samples}
	After unlearning, all results are greater than those of pre-finetuning but smaller than the results before unlearning. This indicates that performing unlearning based on a single target sample has minimal impact on similar samples. 
	
	\begin{figure}[!ht]
		\centering
		\includegraphics[width=1\linewidth]{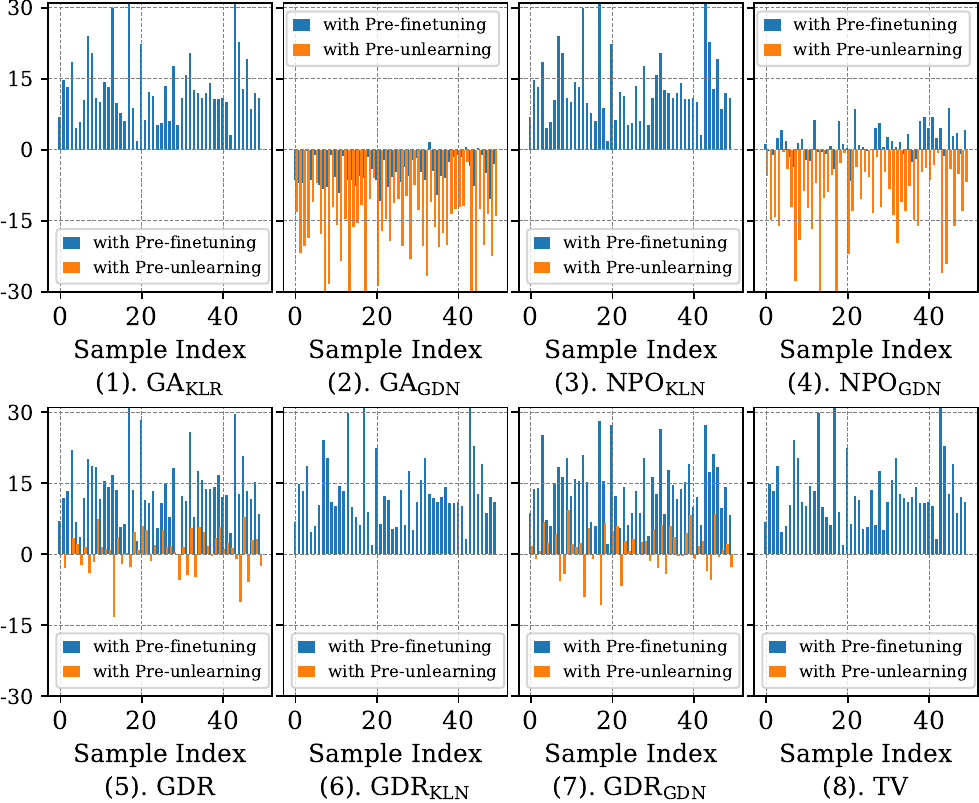}
		\caption{Comparison of verification results toward similar samples as training samples for meta-llama/Llama-3.2-3B-Instruct model.}
		\label{fig:meta_llama_Llama_3_2_3B_50_50}
	\end{figure}
	
	\begin{figure}[!ht]
		\centering
		\includegraphics[width=1\linewidth]{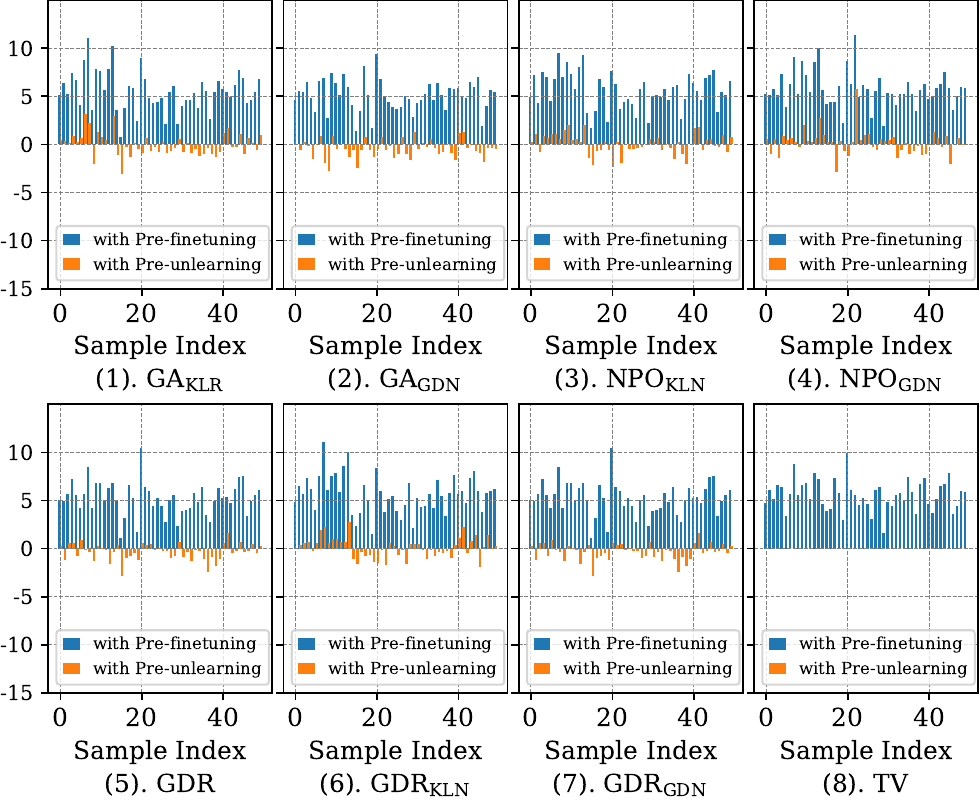}
		\caption{Comparison of verification results toward similar samples as training samples for EleutherAI/gpt-neo-1.3B model.}
		\label{fig:EleutherAI_gpt_neo_1.3B_show_50_50}
	\end{figure}

	\begin{figure}[!ht]
		\centering
		\includegraphics[width=1\linewidth]{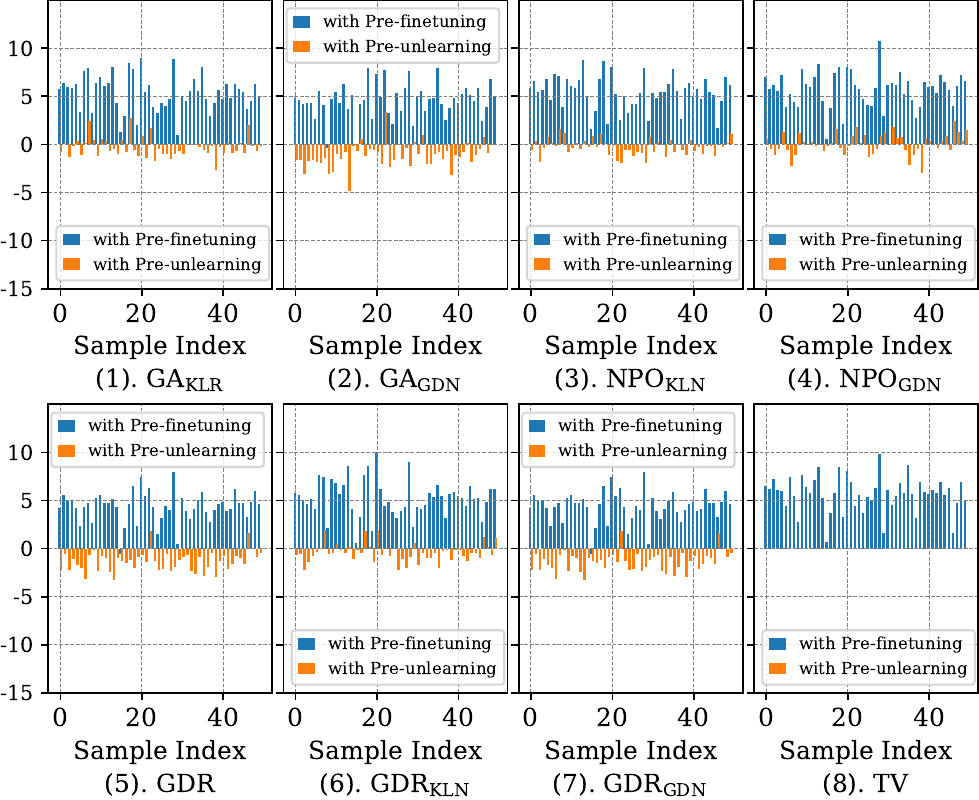}
		\caption{Comparison of verification results toward similar samples as training samples for EleutherAI/gpt-neo-2.7B model.}
		\label{fig:EleutherAI_gpt_neo_2.7B_show_50_50}
	\end{figure}

        \newpage
        \twocolumn
	\subsection{Other Unlearning Results Toward Target Samples for Language Dataset}
	\label{sec:other_results_of_language_unlearning}
	From the following Figures, all unlearning results of \texttt{W-similar samples} are greater than those of \texttt{W/O-similar samples}. We conclude that adding similar samples prevents unlearning based on a single target sample from fully removing the target sample's influence.

	\begin{figure}[!ht]
		\centering
		\includegraphics[width=1\linewidth]{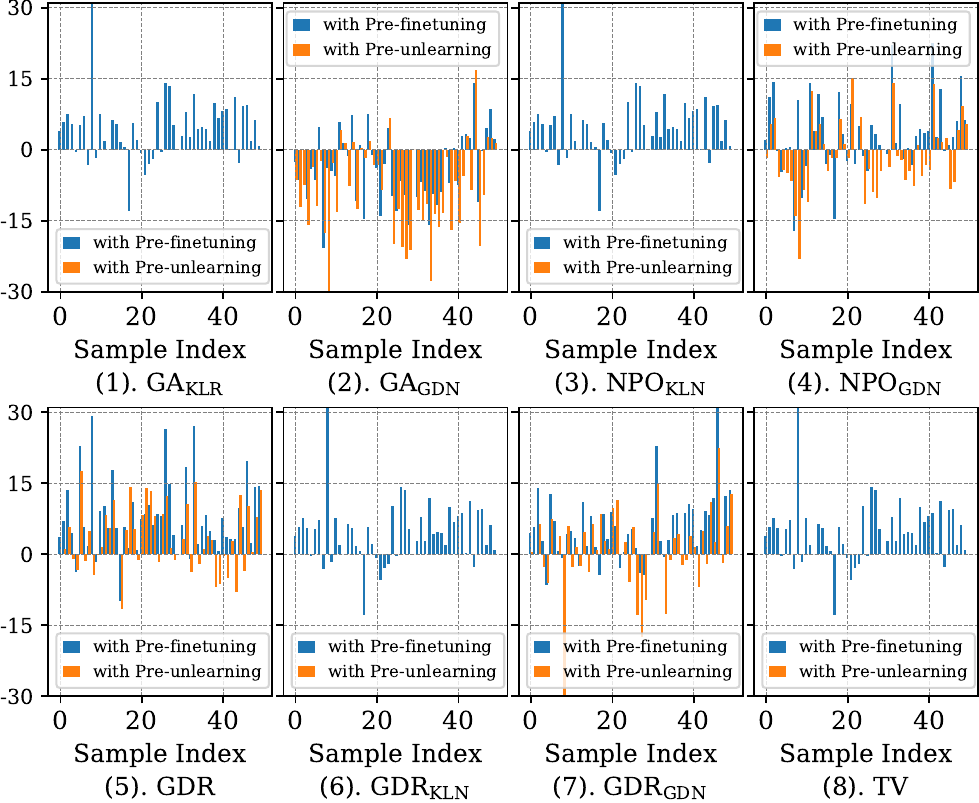}
		\caption{Comparison of verification results from the unlearned model with the results from the pre-unlearning and pre-finetuning models under \texttt{W/O-similar samples} setting for meta-llama/Llama-3.2-3B-Instruct model.}
		\label{appendix_fig:difference_with_original_and_finetuned_for_0_0_meta_llama_Llama_3_2_3B}
	\end{figure}
	
	\begin{figure}[!ht]
		\centering
		\includegraphics[width=1\linewidth]{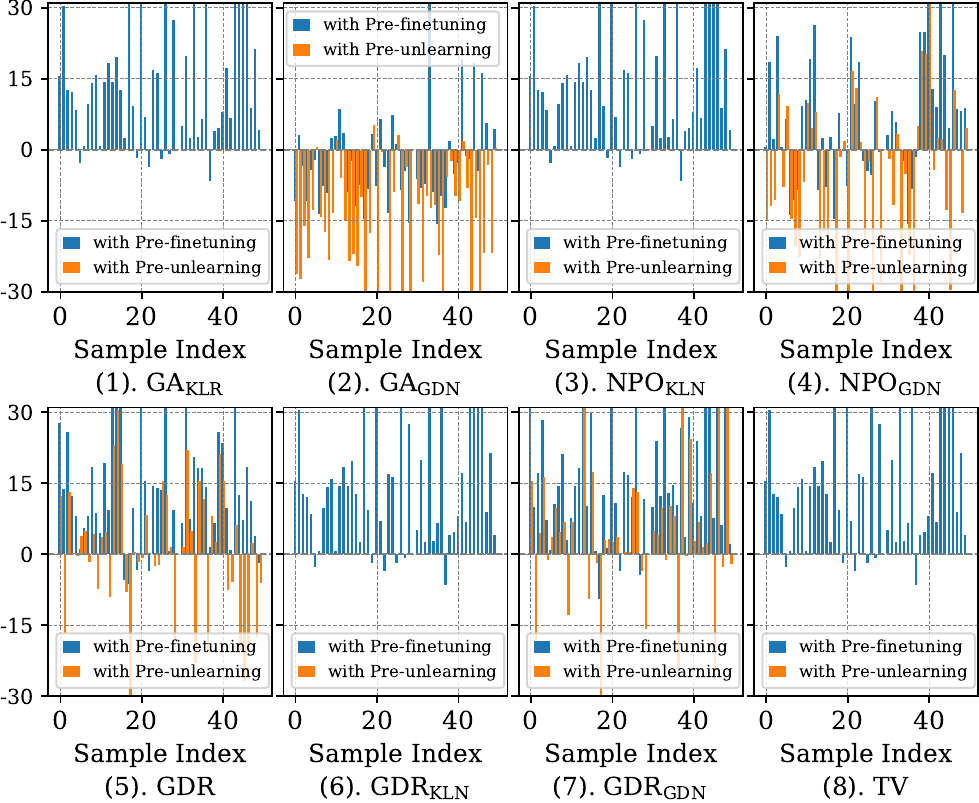}
		\caption{Comparison of verification results from the unlearned model with the results from the pre-unlearning and pre-finetuning models under \texttt{W-similar samples} setting for meta-llama/Llama-3.2-3B-Instruct model.}
		\label{appendix_fig:difference_with_original_and_finetuned_for_50_0_meta_llama_Llama_3_2_3B}
	\end{figure}
	
	\begin{figure}[!ht]
		\centering
		\includegraphics[width=1\linewidth]{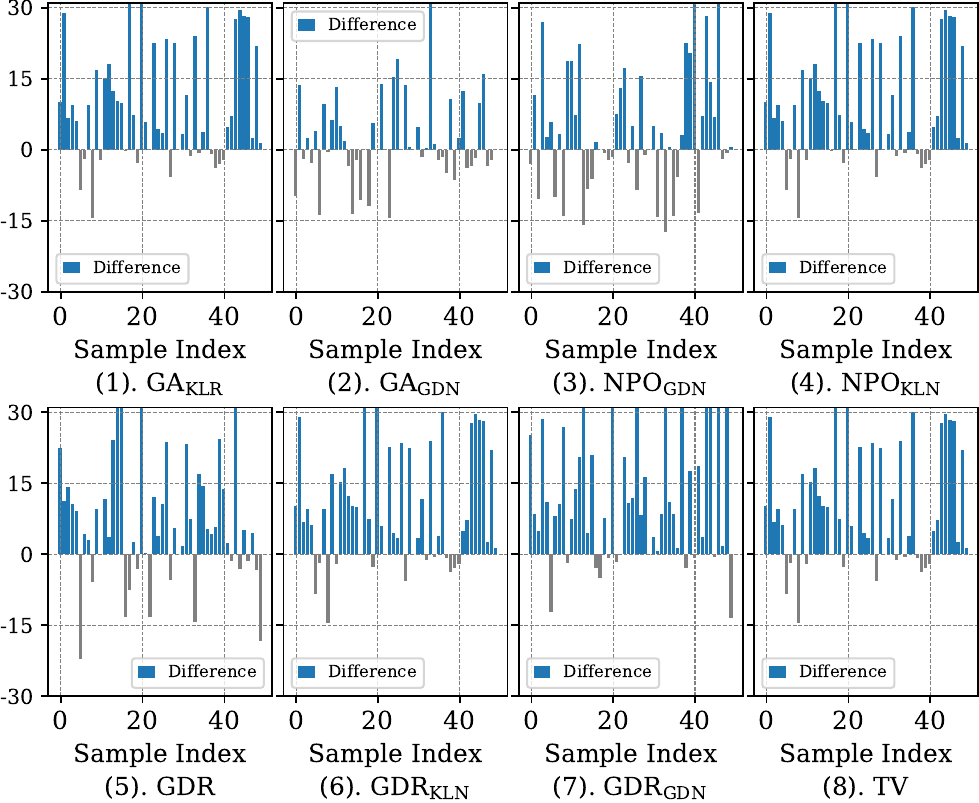}
		\caption{Comparison of verification results toward target samples for meta-llama/Llama-3.2-3B-Instruct model.}
		\label{appendix_fig:meta_llama_Llama_3_2_3B_comparation}
	\end{figure}
	
	\begin{figure}[!ht]
		\centering
		\includegraphics[width=1\linewidth]{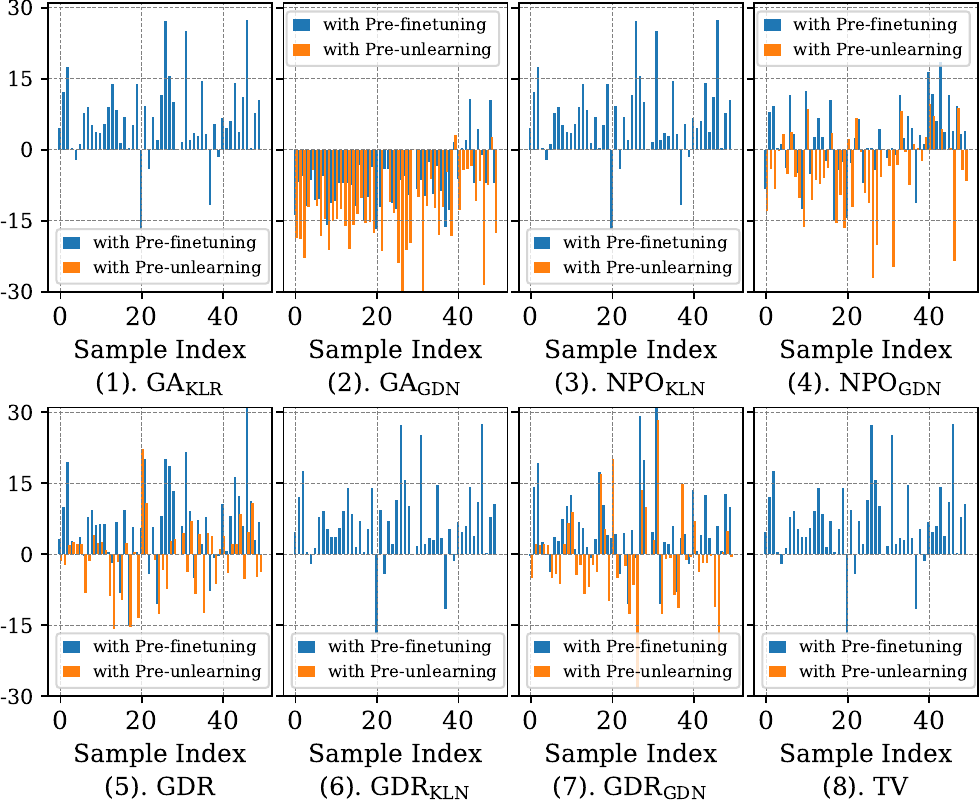}
		\caption{Comparison of verification results from the unlearned model with the results from the pre-unlearning and pre-finetuning models under \texttt{W/O-similar samples} setting for meta-llama/Llama-3.2-1B-Instruct model.}
		\label{fig:difference_with_original_and_finetuned_for_0_0_meta_llama_Llama_3_2_1B}
	\end{figure}
	
	\begin{figure}[!ht]
		\centering
		\includegraphics[width=1\linewidth]{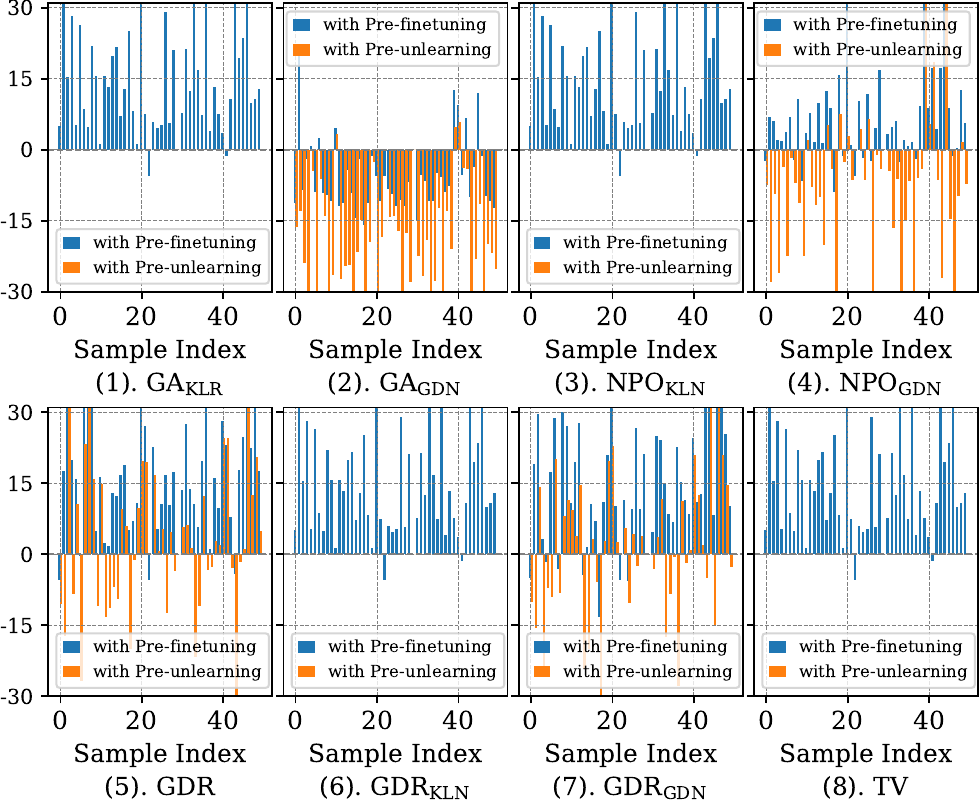}
		\caption{Comparison of verification results from the unlearned model with the results from the pre-unlearning and pre-finetuning models under \texttt{W-similar samples} setting for meta-llama/Llama-3.2-1B-Instruct model.}
		\label{fig:difference_with_original_and_finetuned_for_50_0_meta_llama_Llama_3_2_1B}
	\end{figure}
	
	\begin{figure}[!ht]
		\centering
		\includegraphics[width=1\linewidth]{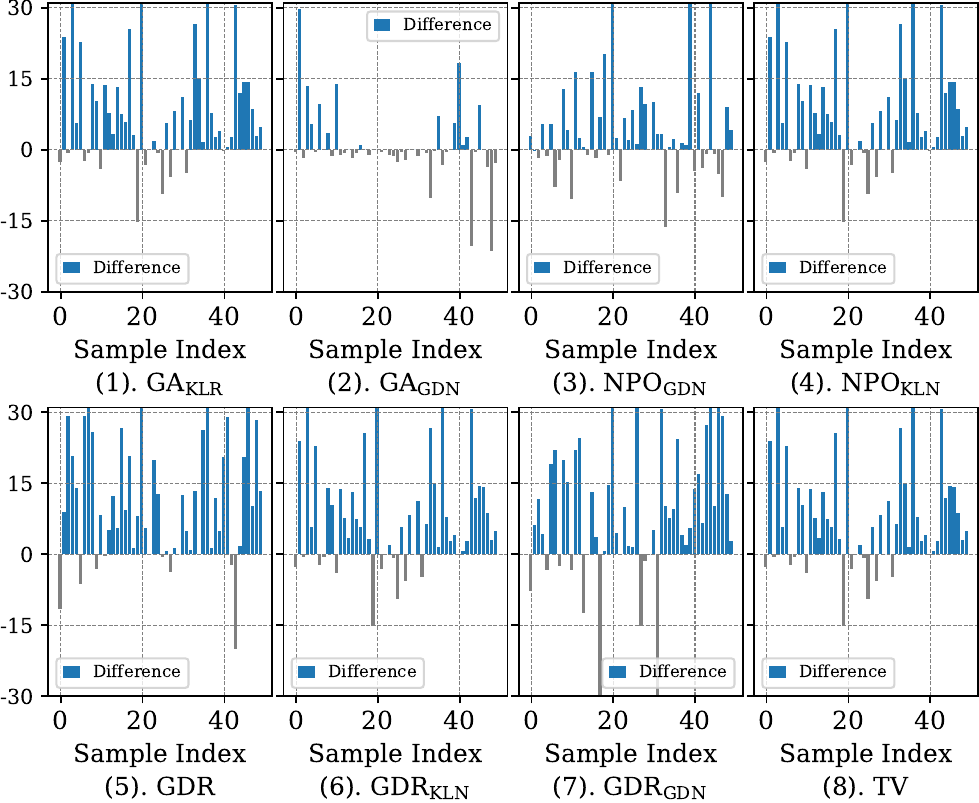}
		\caption{Comparison of verification results toward target samples for meta-llama/Llama-3.2-1B-Instruct model.}
		\label{fig:meta_llama_Llama_3_2_1B_comparation}
	\end{figure}
	
	\begin{figure}[!ht]
		\centering
		\includegraphics[width=1\linewidth]{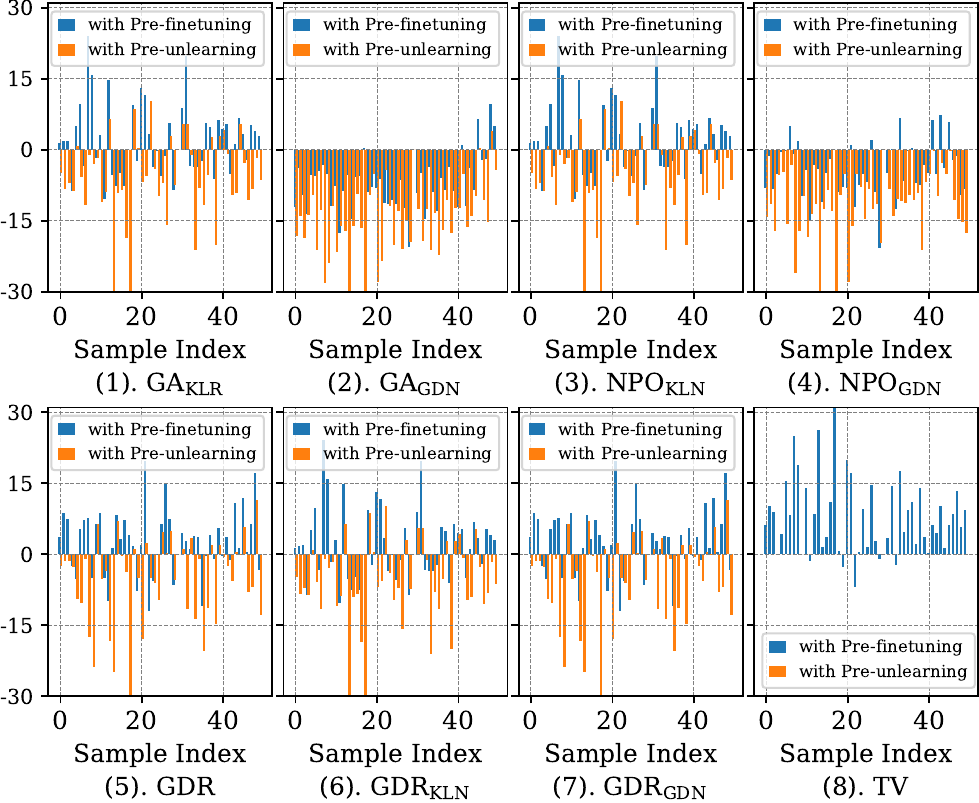}
		\caption{Comparison of verification results from the unlearned model with the results from the pre-unlearning and pre-finetuning models under \texttt{W/O-similar samples} setting for facebook/opt-1.3b model.}
		\label{fig:difference_with_original_and_finetuned_for_0_0_facebook_opt_1.3b}
	\end{figure}
	
	\begin{figure}[!ht]
		\centering
		\includegraphics[width=1\linewidth]{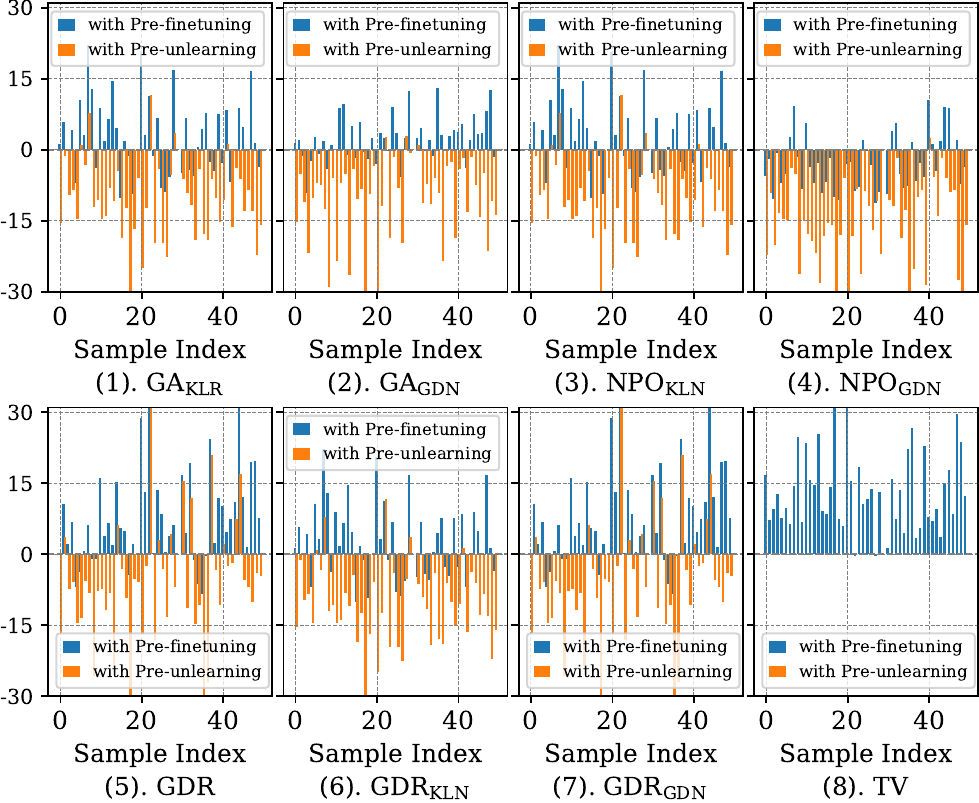}
		\caption{Comparison of verification results from the unlearned model with the results from the pre-unlearning and pre-finetuning models under \texttt{W-similar samples} setting for facebook/opt-1.3b model.}
		\label{fig:difference_with_original_and_finetuned_for_50_0_facebook_opt_1.3b}
	\end{figure}
	
	\begin{figure}[!ht]
		\centering
		\includegraphics[width=1\linewidth]{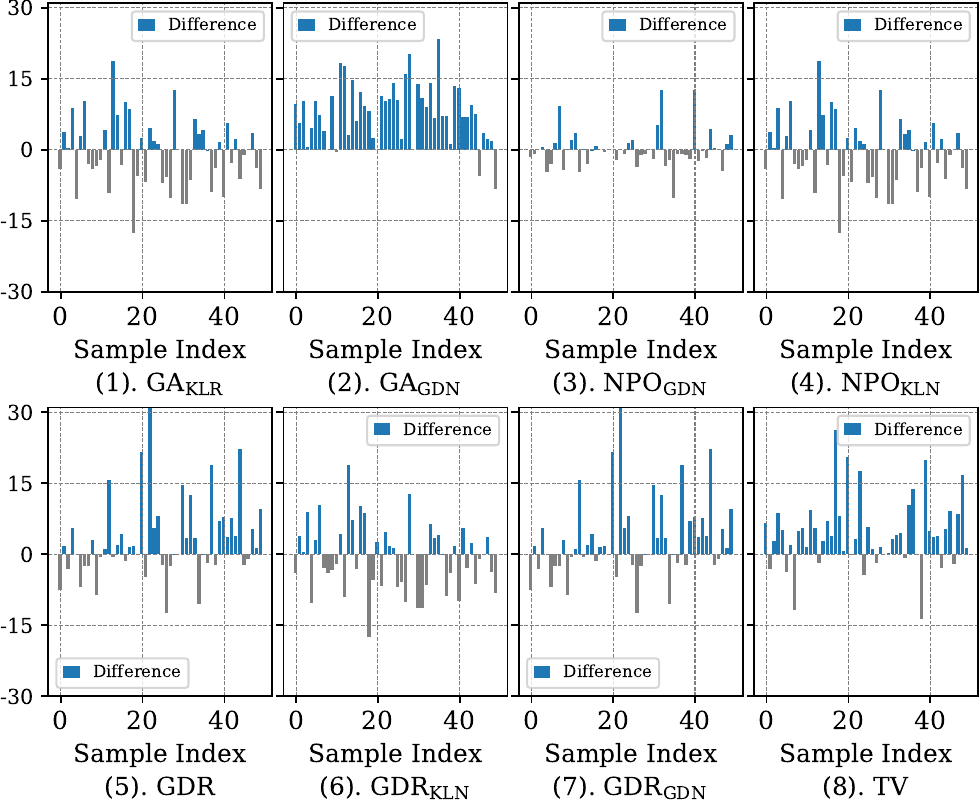}
		\caption{Comparison of verification results toward target samples for facebook/opt-1.3b model. Some sub-figures do not reflect our conclusion~(such as (1)(2)), this is due to the limited effectiveness of the unlearning schemes.}
		\label{fig:facebook_opt_1.3b_comparation}
	\end{figure}
	
	\begin{figure}[!ht]
		\centering
		\includegraphics[width=1\linewidth]{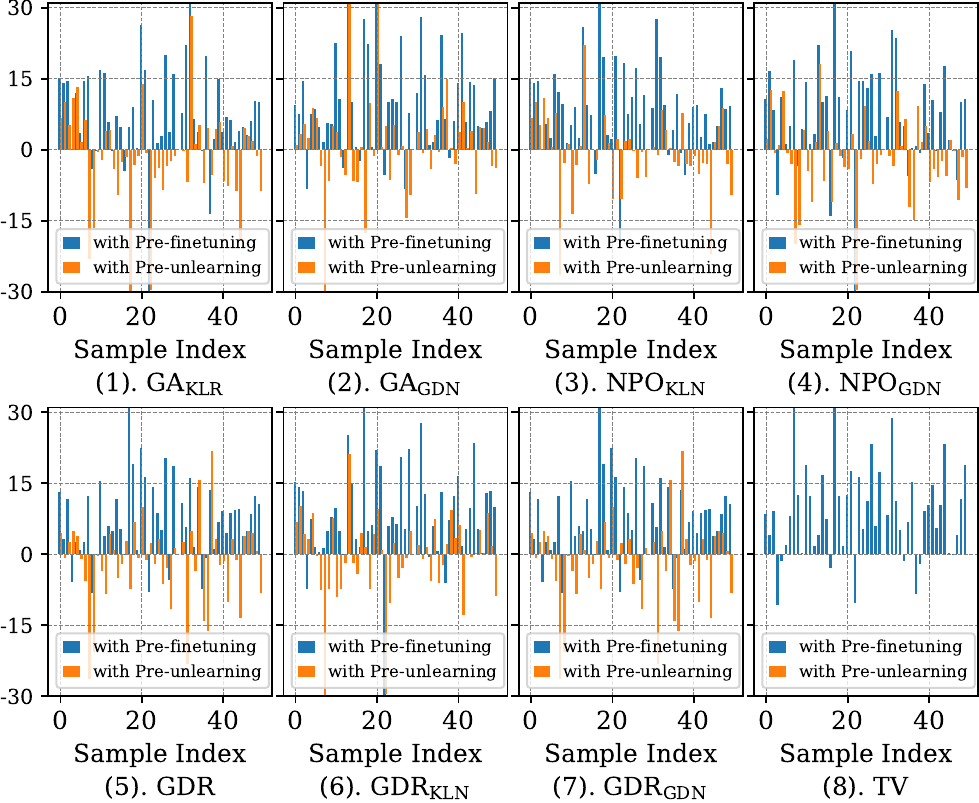}
		\caption{Comparison of verification results from the unlearned model with the results from the pre-unlearning and pre-finetuning models under \texttt{W/O-similar samples} setting for EleutherAI/gpt-neo-2.7B model.}
		\label{fig:difference_with_original_and_finetuned_for_0_0_facebook_opt_2.7b}
	\end{figure}
	
	\begin{figure}[!ht]
		\centering
		\includegraphics[width=1\linewidth]{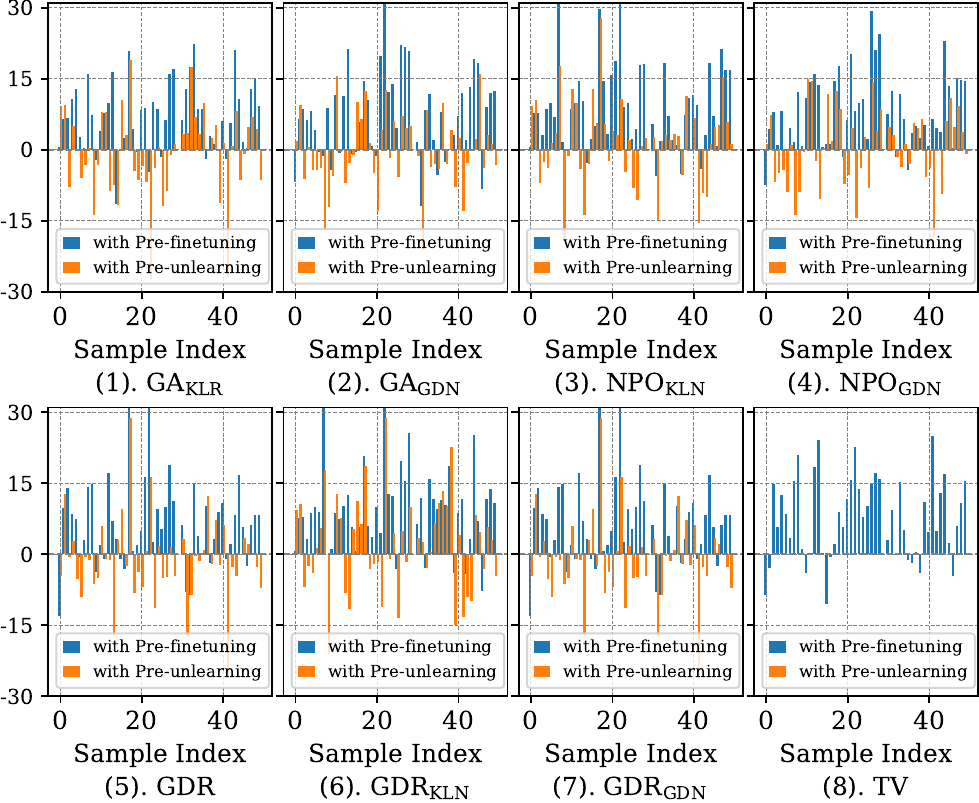}
		\caption{Comparison of verification results from the unlearned model with the results from the pre-unlearning and pre-finetuning models under \texttt{W-similar samples} setting for EleutherAI/gpt-neo-2.7B model.}
		\label{fig:difference_with_original_and_finetuned_for_50_0_facebook_opt_2.7b}
	\end{figure}
	
	\begin{figure}[!ht]
		\centering
		\includegraphics[width=1\linewidth]{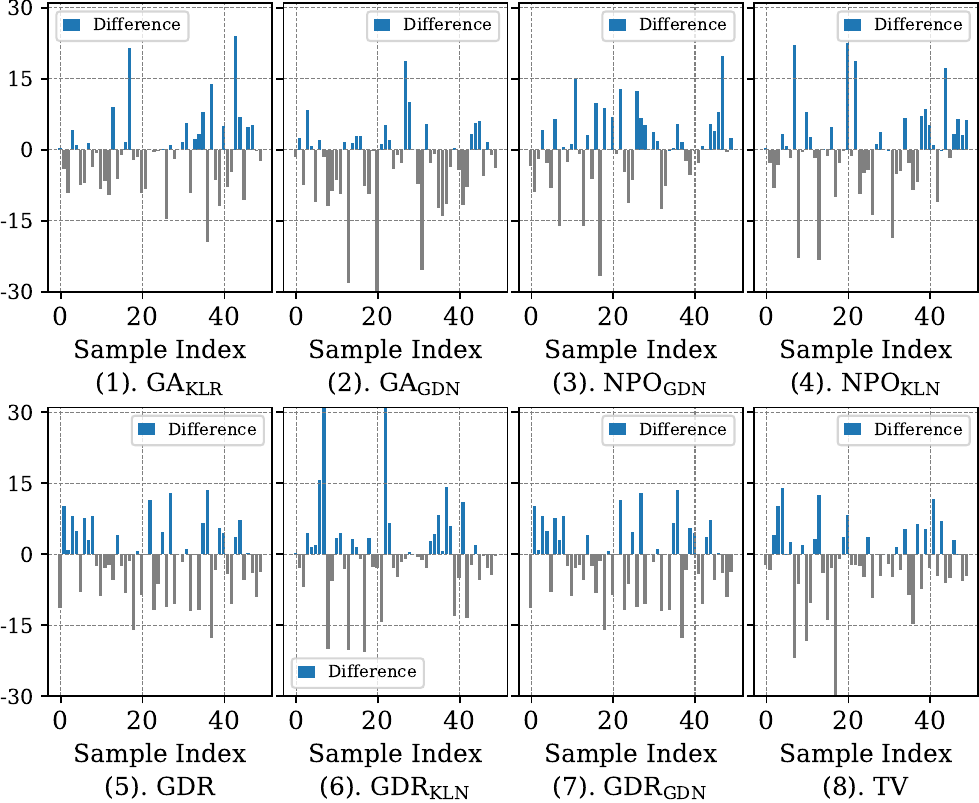}
		\caption{Comparison of verification results toward target samples for EleutherAI/gpt-neo-2.7B model.}
		\label{fig:facebook_opt_2.7b_comparation}
	\end{figure}

        \newpage
        \twocolumn
	\subsection{Other Unlearning Results Toward Similar Samples as Test Samples}
	\label{sec:other_for_influence_toward_similar_samples_as_test_samples}
	
	All unlearning results are greater than those of pre-finetuning but smaller than the results of pre-unlearning. This indicates that unlearning based only on the target sample is unlikely to generalize to other test samples similar to the target sample.
	
	\begin{figure}[!ht]
		\centering
		\includegraphics[width=1\linewidth]{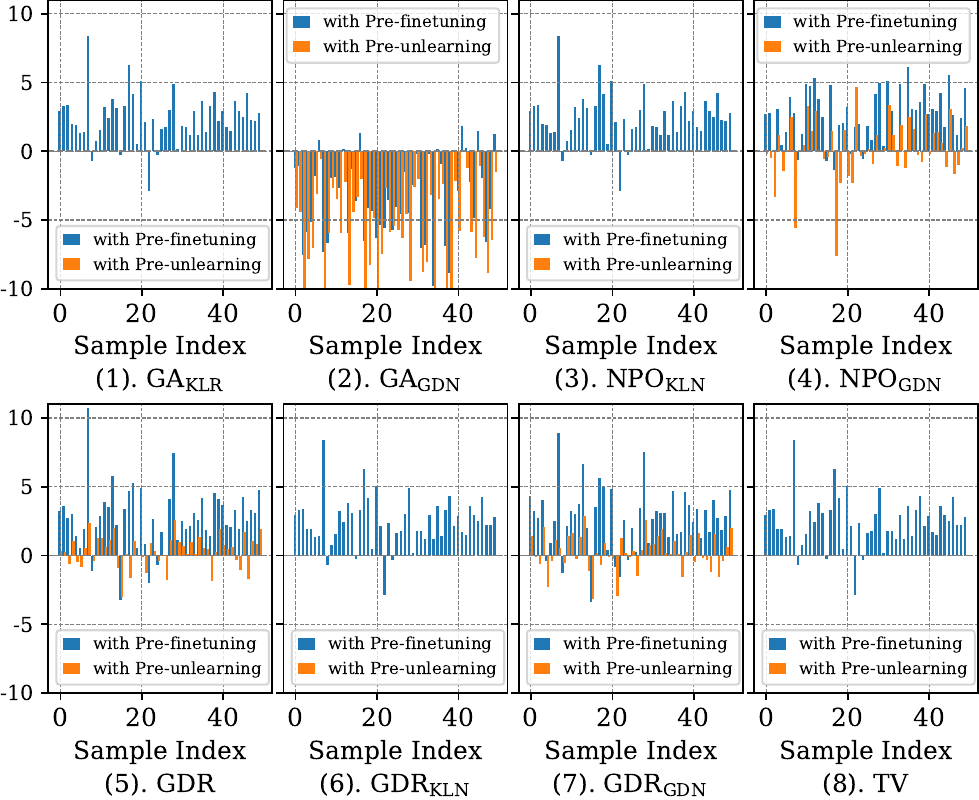}
		\caption{Comparison of verification results toward similar samples as test samples for meta-llama/Llama-3.2-1B-Instruct.}
		\label{fig:meta_llama_Llama_3.2_1B_Instruct_show_0_50}
	\end{figure}

	\begin{figure}[!ht]
		\centering
		\includegraphics[width=1\linewidth]{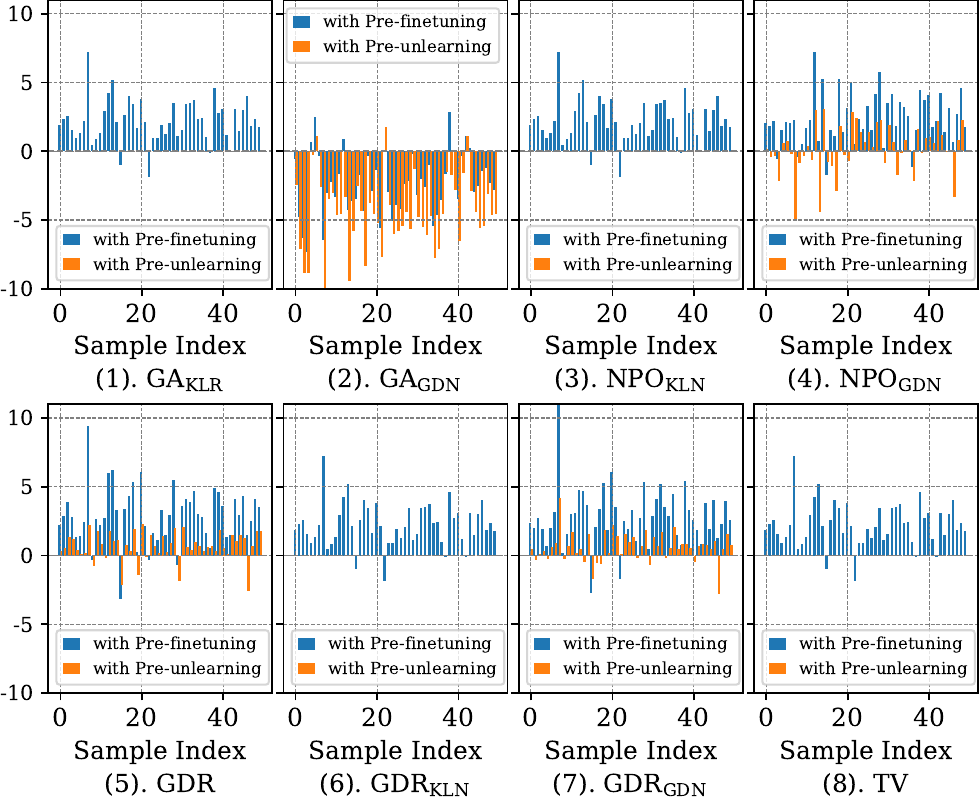}
		\caption{Comparison of verification results toward similar samples as test samples for meta-llama/Llama-3.2-3B-Instruct.}
		\label{fig:meta_llama_Llama_3.2_3B_Instruct_show_0_50}
	\end{figure}

	\begin{figure}[!ht]
		\centering
		\includegraphics[width=1\linewidth]{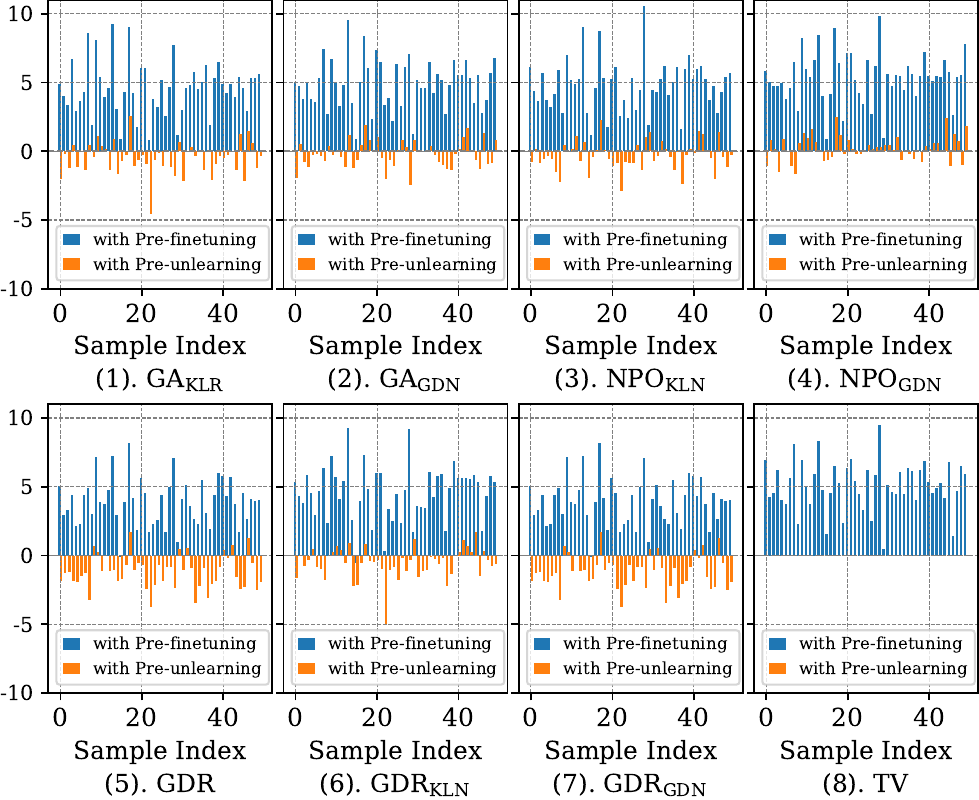}
		\caption{Comparison of verification results toward similar samples as test samples for EleutherAI/gpt-neo-2.7B.}
		\label{fig:EleutherAI_gpt_neo_2.7B_show_0_50}
	\end{figure}

\end{document}